\documentclass[journal]{IEEEtran}
\usepackage{cite}

\ifCLASSINFOpdf
  \usepackage[pdftex]{graphicx}
  \graphicspath{{../pdf/}{../jpeg/}}
  \DeclareGraphicsExtensions{.pdf,.jpeg,.png}
\else
  \usepackage[dvips]{graphicx}
  \DeclareGraphicsExtensions{.eps}
\fi
\usepackage{amsmath}
\usepackage[vlined,english,ruled,linesnumbered]{algorithm2e}
\usepackage{algpseudocode}

\usepackage{array}

\usepackage{stfloats}
\usepackage{url}

\usepackage{multirow}
\usepackage{subfigure}
\usepackage{amsfonts,amssymb}

\usepackage{color}

\DeclareMathOperator*{\argmin}{arg\,min}

\usepackage{changepage}
\usepackage{hyperref}

\usepackage{amsmath,amsfonts,amssymb,amsthm,epsfig,epstopdf,url,array}

\theoremstyle{plain}
\newtheorem{thm}{Theorem}

\theoremstyle{definition}

\newtheorem*{defn*}{Definition}

\theoremstyle{remark}

\begin{document}
\title{Homotopic Convex Transformation: A New Landscape Smoothing Method for the Traveling Salesman Problem}
\date{}

\author{Jialong~Shi,
        Jianyong~Sun,
        Qingfu~Zhang
        ~and~Kai~Ye
        \thanks{The work was partially supported by the National Natural Science Foundation of China (NSFC) under Grants 11991023, 61903294 and 61876163 and the China Postdoctoral Science Foundation under Grant 2018M643656. \emph{Corresponding author: Jianyong Sun.}}
        \thanks{Jialong Shi and Jianyong Sun are with the School of Mathematics and Statistics, Xi'an Jiaotong University, Xi'an, China (e-mail: \{jialong.shi, jy.sun\}@xjtu.edu.cn).}
        \thanks{Qingfu Zhang is with the Department of Computer Science, City University of Hong Kong, Hong Kong (e-mail: qingfu.zhang@cityu.edu.hk).}
        \thanks{Kai Ye is with MOE Key Lab for Intelligent Networks \& Networks Security, School of Electronics and Information Engineering, Xi'an Jiaotong University, Xi'an, 710049, China, and the First Affiliated Hospital of Xi'an Jiaotong University, Xi'an, 710061, China (e-mail: kaiye@xjtu.edu.cn).}
        }

\markboth{}
{Shi \MakeLowercase{\textit{et al.}}: }

\maketitle

\begin{abstract}
  This paper proposes a novel landscape smoothing method for the symmetric Traveling Salesman Problem (TSP). We first define the Homotopic Convex (HC) transformation of a TSP as a convex combination of a well-constructed simple TSP and the original TSP. The simple TSP, called the convex-hull TSP, is constructed by transforming a known local or global optimum. We observe that controlled by the coefficient of the convex combination, with local or global optimum, (i) the landscape of the HC transformed TSP is smoothed in terms that its number of local optima is reduced compared to the original TSP; (ii) the fitness distance correlation of the HC transformed TSP is increased. Further, we observe that the smoothing effect of the HC transformation depends highly on the quality of the used optimum. A high-quality optimum leads to a better smoothing effect than a low-quality optimum. We then propose an iterative algorithmic framework in which the proposed HC transformation is combined within a heuristic TSP solver. It works as an escaping scheme from local optima aiming to improve the global search ability of the combined heuristic. Case studies using the 3-Opt and the Lin-Kernighan local search as the heuristic solver show that the resultant algorithms significantly outperform their counterparts and two other smoothing-based TSP heuristic solvers on most of the test instances with up to \text{20,000} cities.
\end{abstract}


\IEEEpeerreviewmaketitle

\section{Introduction}

As a well-known ${\cal NP}$-hard combinatorial optimization problem, the Traveling Salesman Problem (TSP)~\cite{applegate2006traveling} has been extensively studied for years. Algorithms, including both exact and metaheuristics, have been developed in a number of literatures. It's been well acknowledged that exact algorithms struggle to find the global optima of large size TSPs within a reasonable time. Metaheuristics, while not guaranteeing to find the global optimum, can provide near-optimal solutions within tolerable time for medium and large size TSPs. The hardness of the TSP is due to the ruggedness and irregularity of its fitness landscape. It is extremely hard for a metaheuristic to find the global optimum of a TSP when there are many local optima in its search landscape.

In this paper, we propose a novel method to smooth the TSP landscape. For a TSP with $n$ cities, we first construct a 2D Euclidean TSP in which the $n$ cities lie on the convex hull of the TSP graph and the order of the cities follows the order they appear in a known local optimum of the original TSP. The generated TSP is named as the \emph{convex-hull TSP}. It can be proved that the convex-hull TSP is unimodal, which means that a search process starting from any initial solution will always end in the global optimum of the convex-hull TSP. The original TSP can then be smoothed by a convex combination of the convex-hull TSP and the original TSP with coefficient $\lambda\in [0,1]$. When $\lambda=0$, the transformed TSP is just the original TSP, while $\lambda=1$ means the convex-hull TSP. This actually defines a \emph{homotopic} transformation from the original TSP to the transformed TSP. We thus call the proposed combination as the {\em Homotopic Convex (HC) transformation}.

To investigate the smoothing effects of the HC transformation, we conduct landscape analysis experiments on 12 TSP instances from the well-known TSPLIB~\cite{reinelt1991tsplib}. In our experiment, the ruggedness of the TSP landscape is measured by four metrics, including local optimum density, escaping rate, Fitness Distance Correlation (FDC) and runtime. These metrics are computed from the results of 1000 runs of Iterated Local Search (ILS)~\cite{lourencco2010iterated} with the 3-Opt local search and the double bridge perturbation. In our experiments, different values of $\lambda$ are tested and their influence to the smoothing effect is analyzed.

We observe from the experimental results that the HC transformation exhibits the following three features. First, the landscape of the transformed TSP has less number of local optima than the original TSP, which implies a smoother TSP landscape. Second, in the transformed TSP, the FDC is higher than that of the original TSP.  As stated in~\cite{merz2000fitness,merz2004advanced}, a high FDC implies that a search heuristic can find a path to the global optimum by observing the fitness change of the encountered local optima. Third, the strength of smoothing can be controlled by the coefficient $\lambda$. In addition, we observe that the runtime for ILS to find the global optimum is significantly reduced after the transformation.

We thereby make use of these features to propose a general iterative algorithmic framework. In the framework, the HC transformation is combined within a TSP heuristic. During the search, the heuristic works on the transformed TSP, while the transformed TSP is iteratively updated by the newly-found best solution with respect to the original TSP. As case studies, we embed the 3-Opt local search and the Lin-Kernighan (LK) local search~\cite{lin1973effective} into the proposed framework. The instantiated algorithms are called Landscape Smoothing Iterated Local Search (LSILS) and LSILS-LK, respectively. In the experimental study, LSILS and LSILS-LK are compared against the original ILS and two existing landscape smoothing based algorithms on 17 TSP instances with up to 20,000 cities. Comparison results show that LSILS and LSILS-LK significantly outperform their counterparts on most of the test instances. Note that we do not claim the proposed LSILS is better than the state-of-the-art TSP solvers. Our goal is to show the potential of the HC transformation on improving the existing TSP heuristics.

The rest of this paper is organized as follows. Section~\ref{sec:TSP} presents the related work. Section~\ref{sec:HCT} presents the definition of the HC transformation and the experimental studies on the investigation of the characteristics of the HC transformation. In Section~\ref{sec:LSILS} the framework and the instantiated algorithms are presented. Experimental results against its counterparts are also provided in this section. Section~\ref{sec:conclu} concludes the paper and discusses future work.

\section{Related Work}\label{sec:TSP}

Given a set of cities and the cost between every pair of cities, the TSP is to find the most cost-effective tour that visits every city exactly once and returns to the starting city. It is ${\cal NP}$-hard~\cite{garey1979computers} and one of the most widely used testbeds in combinatorial optimization. The formal definition of the TSP can be stated as follows. Let ${\cal G} =({\cal V}, {\cal E})$ be a fully connected graph with cities as vertexes, where $\cal V$ is the vertex set and $\cal E$ the edge set. Denote $c_{i,j}>0$ the cost of the edge between vertex $i$ and vertex $j$, the objective function of a TSP is defined as
\begin{equation}
\begin{split}
\mbox{minimize} ~&~ f(x)= c_{x_{(n)},x_{(1)}} + \sum_{i=1}^{n-1}c_{x_{(i)},x_{(i+1)}},\\
\mbox{subject to} ~&~ x = (x_{(1)},x_{(2)},\dots,x_{(n)}) \in \mathcal{P}_n,
\end{split}
\end{equation}where $f:\mathcal{P}_n \to \mathbb{R}$ is the objective function and $\mathcal{P}_n$ is the permutation space of $\{1, 2, \cdots, n\}$. In this paper we focus on symmetric TSPs, i.e., $c_{i,j} = c_{j,i}$ for all $i,j \in \{1,2,\dots,n\}$.

The hardness of a TSP is usually characterized by its fitness landscape. Formally, the fitness landscape of a combinatorial optimization problem is defined by $(\mathcal{S},f,d)$, where $\mathcal{S}$ is the solution space, $f:\mathcal{S}\to \mathbb{R}$ is the objective function (fitness) and $d$ is a distance measure on $\mathcal{S}$. For the TSP, the distance between any two solutions $x_1$ and $x_2$ is defined as the number of edges contained in one solution but not in the other:
\begin{equation}\label{eq:tsp_dist}
  d(x_1,x_2) = |\{e\in {\cal E}|e\in x_1 \land e\notin x_2\}|
\end{equation}where $e\in x$ means edge $e$ is in the solution $x$.

A great amount of efforts have been made on the TSP fitness landscape analysis. Among these studies, Stadler and Schnabl~\cite{stadler1992landscape} first investigated the landscape of symmetric and asymmetric TSPs by random walk. Boese~\cite{boese1995cost} first observed that there is a strong correlation between the distance of a solution to the global optimum and its cost when he investigated the landscape of a specific TSP instance att532. He called this phenomenon the \emph{big valley structure}. Hains et al.~\cite{hains2011revisiting} confirmed the existence of the big valley structure on some TSP instances. However, they also argued that the big valley breaks down into several ``funnels'' around some local optima with function values very close to that of the global optimum. Ochoa and Veerapen~\cite{ochoa2016deconstructing} analyzed the big valley structure assumption by their local optima network model. They concluded that in the four studied instances the big valley decomposes into a number of sub-valleys. However, Ochoa and Veerapen's conclusion depends strongly on the solution perturbation method used in their study. Shi et al.~\cite{shi2018eb} gave a more detailed definition of the big valley structure. In their definition, a TSP instance exhibits a big valley structure if (i) all the global optima of a TSP are located in a relatively small region in the solution space (this defines the ``bottom'' of the big valley), and (ii) there is a strong correlation between the fitness of a solution and its distance to the nearest global optimum (this defines the ``slope'' of the big valley). Empirical analysis on the landscape of 10 TSPLIB instances showed that 9 of the 10 test instances meet the detailed definition of the big valley structure.

In the work of Fonlupt et al.~\cite{fonlupt1997fitness}, it is found that the TSP landscape defined by 2-Opt-move has a higher FDC value than the TSP landscape defined by city-swap. Merz and Freisleben~\cite{merz2001memetic} studied the ruggedness and the FDC of the TSP landscape using the 3-Opt local search and the Lin-Kernighan (LK) local search~\cite{lin1973effective}. They observed that local optima are frequently close to each other and also close to the global optimum. Tayarani-N and Pr\"ugel-Bennett~\cite{tayarani2014landscape,tayarani2016analysis} also used the 3-Opt local search as a tool to analyze the characteristics of the fitness landscape on different classes of TSPs. They found that high-quality local optima are more likely to be found by a local search process than low-quality local optima and the probability of finding a global optimum decreases exponentially with increasing problem size. They also found that the Euclidean TSPs have relatively high FDC values and the value is relatively low for random TSP. For more detailed survey, interested readers please see~\cite{reidys2002combinatorial,watson2010introduction,pitzer2012comprehensive}.

Existing TSP metaheuristics fall basically into four categories. One kind is constructive based where a solution to a TSP is constructed gradually from partial to whole. Examples of such including Ant Colony Optimization (ACO)~\cite{dorigo2006ant}, Greedy Randomized Adaptive Search Procedure (GRASP)~\cite{feo1995greedy} and others. The second kind is local search based. Local search first defines a neighborhood structure on the solution space of the TSP. Starting from an initial solution, local search iteratively explores the neighborhood of the current solution for a better solution. Local search can only guarantee to find local optima. The most widely-used local search for the TSP is the $k$-Opt local search ($k\geq 2$), in which the neighborhood structure is defined by edge exchange. In each descent move of a $k$-Opt local search, $k$ edges are replaced by another $k$ edges. The well-known 2-Opt local search, 3-Opt local search~\cite{applegate2006traveling} and the LK local search all belong to this class. The third category is the population based metaheuristics, e.g. Genetic Algorithm (GA)~\cite{davis1991handbook}, Particle Swarm Optimization (PSO)~\cite{shi1999empirical} and others. In this category, solutions are created through recombination operators working in the permutation space. The fourth category is the Memetic Algorithm (MA)~\cite{krasnogor2005tutorial}, in which the local search based heuristics are hybridized with heuristics from other categories. 

To escape from local optima, existing local search based metaheuristics for the TSP are usually embedded with global optimization strategies such as criterion, solution or problem perturbation and landscape smoothing. In criterion perturbation, rather than only accepting better solution in the neighborhood, a worse solution could be accepted in probability (e.g. Simulated Annealing~\cite{aarts2005simulated,charon1996mixing}) or if it prevents from returning to previously visited areas (e.g. Tabu Search~\cite{glover1999tabu}). In the solution perturbation methods, the current local optimum is perturbed to obtain a new solution and the local search process restarts from that new solution. The well-known ILS~\cite{lourencco2010iterated,applegate2003chained} belongs to this category.

The problem perturbation methods modify the problem itself so that the current local optimum is no longer locally optimal to the modified problem. In the Guided Local Search (GLS) proposed by Voudouris and Tsang~\cite{voudouris1999guided}, the original TSP is changed by adding penalties on some selected edges. Whenever the GLS is trapped in a local optimum, edges from the local optimum are evaluated and penalized so that the local optimum is no longer optimal in the next round of search. Shi et al.~\cite{shi2018eb} improved the GLS based on the big valley assumption. Their GLS maintains an elite solution and the edges in the elite solution will be protected from the penalization. Walshaw~\cite{walshaw2004multilevel} introduced a multilevel refinement approach to simplify the TSP. In each level, several selected cities are matched and the edges between them is fixed so that the problem is coarsened to an easier problem. In the backbone-guided local search proposed by Zhang and Looks~\cite{zhang2005novel}, the edge cost is reduced proportionally to the frequency that the edge appears in the historical local optima set.

Very few global search approaches have been proposed based on landscape smoothing. Gu and Huang~\cite{gu1994efficient} proposed to smooth the landscape of a TSP by edge cost manipulation. They normalize and transform the edge cost to
\begin{equation}\label{eq:gu_smooth_cost}
  c'_{i,j}(\alpha) =
  \left\{
    \begin{split}
     \bar c + (c_{i,j}-\bar c)^\alpha ~&~ \mbox{if}~ c_{i,j}\geq\bar c,\\
     \bar c - (\bar c-c_{i,j})^\alpha ~&~ \mbox{if}~ c_{i,j}<\bar c,\\
  \end{split}
  \right.
\end{equation}where $\bar c $ is the average cost of all edges and $\alpha\geq1$ is the smoothing factor. When $\alpha=1$, it is just the original TSP. Note that in Eq.~(\ref{eq:gu_smooth_cost}) the edge costs are normalized, hence $\lim_{\alpha \rightarrow +\infty} c'_{i,j} = \bar c$. That is, with increasing $\alpha$, all edge costs approach to a fixed value $\bar c$. This implies that all solutions will have the same fitness, i.e., the TSP landscape is smoothed to a plane, as sketched in Fig.~\ref{fig:smooth_gu1994}.
\begin{figure}
  \centering
  \includegraphics[width=0.75\linewidth]{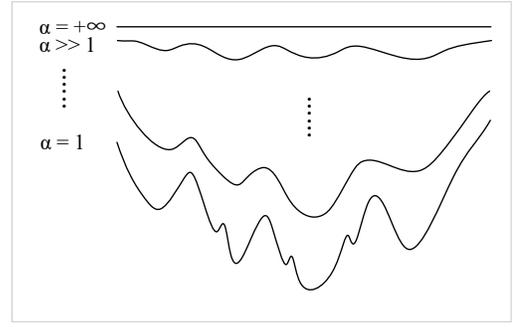}\\
  \caption{A sketch of the effect of the search space smoothing approach proposed by Gu and Huang (reproduced from~\cite{gu1994efficient}).}\label{fig:smooth_gu1994}
\end{figure}
Schneider et al.~\cite{schneider1997search} proposed several new smoothing functions (exponential, hyperbolic, sigmoidal, and logarithmic) for the TSP. Based on the work of Gu and Huang and Schneider et al., Coy et al.~\cite{coy1998see} carried out systematic experiments to investigate the smoothing effect of the smoothing functions to the fitness landscape. They claimed that the smoothing method can help local search escape from poor local optima by moving uphill occasionally and by not taking a downhill move occasionally. Further study by Coy et al.~\cite{coy2000computational} shows that a sequential smoothing algorithm with alternated convex and concave smoothing function can achieve satisfactory performance. Dong et al.~\cite{dong2006stochastic} proposed a stochastic local search metaheuristic based on the exponential smoothing method proposed by Schneider et al.. Hasegawa and Hiramatsu~\cite{hasegawa2013mutually} combined Gu and Huang's smoothing method within a MA and found that it is beneficial to apply the smoothing method in the MA. 

\section{Homotopic Convex Transformation}\label{sec:HCT}

In this section, we present a new landscape smoothing method called the Homotopic Convex (HC) Transformation. It transforms the original TSP by combining it with a well-designed unimodal TSP, which is named as the \emph{convex-hull TSP}. In the following, we first define the convex-hull TSP and the HC transformation. The effects of the HC transformation to landscape smoothing are illustrated afterwards.

\subsection{Convex-Hull TSP}

Although the TSP is ${\cal NP}$-hard, some TSP instances are easy to solve. For example, a 2D Euclidean TSP with all the cities lying on the convex hull of the TSP graph, which we name it as the convex-hull TSP, is very easy to solve~\cite{macgregor1996human}. Fig.~\ref{fig:convex_hull} shows an example of the convex-hull TSP and its global optimum. The only global optimum of a convex-hull TSP is the convex hull itself.

\begin{figure}
  \centering
  \subfigure[a convex-hull TSP]{
    \label{fig:convex-hull_TSP} 
    \includegraphics[width=0.25\linewidth]{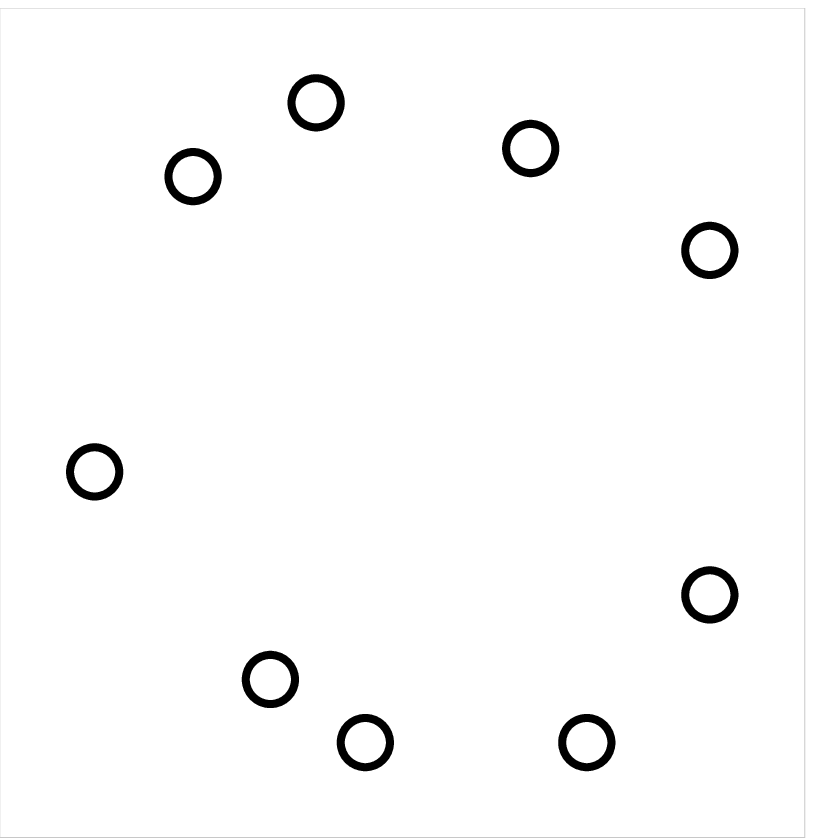}}
    \hspace{0.1\linewidth}
  \subfigure[the corresponding global optimum]{
    \label{fig:convex-hull_TSP_sol} 
    \includegraphics[width=0.25\linewidth]{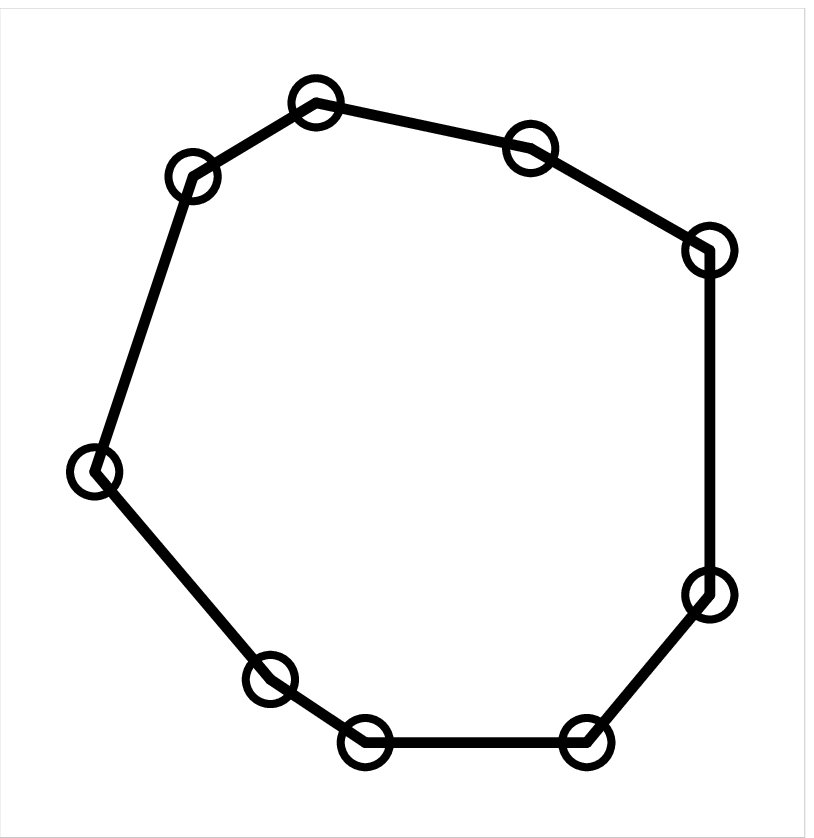}}
 \caption{An example of the convex-hull TSP and the corresponding global optimum}\label{fig:convex_hull}
\end{figure}

Here we prove that the convex-hull TSP is unimodal for any $k$-Opt local search, i.e., the TSP tour formed by the edges on the convex hull is the only $k$\emph{-optimal} tour~\cite{lin1965computer} of such TSP.
\begin{defn*}
A tour is said to be $k$\emph{-optimal} if it is impossible to obtain a tour with smaller cost by replacing any $k$ of its edges by any other set of $k$ edges.
\end{defn*}
The property of the $k$\emph{-optimal} tour is summarized in Theorem~\ref{thm:C_k} reproduced from~\cite{lin1965computer}.
\begin{thm}\label{thm:C_k}
Denote $C_k$ the set of all $k$-optimal tours, we have $C_1 \supset C_2 \supset \dots \supset C_n$. In other words, a $k$-optimal tour is also a $k'$-optimal tour for $k'<k$.
\end{thm}

Based on Theorem~\ref{thm:C_k}, we prove that the convex-hull TSP is unimodal to any $k$-Opt local search.  Theorem~\ref{thm:unimodal} summarizes the result.
\begin{thm}\label{thm:unimodal}
For a convex-hull TSP, let $x_c$ denote the convex hull tour, then $x_c$ is the only $k$-optimal tour in the solution space of the convex-hull TSP for any $k\in\{2,3,\dots,n\}$. 
\end{thm}

\begin{proof} It is sure that $x_c$ is $k$-optimal for any $k\in\{2,3,\dots,n\}$. We here prove the uniqueness of $x_c$ by contradiction. Assume there is another $k$-optimal tour $x_\ast$ for $f_c$. By Theorem~\ref{thm:C_k}, we known that $x_\ast$ is also 2-optimal. Since $x_\ast \neq x_c$, there are at least two edges in $x_\ast$ that cross each other. According to the triangle inequality, it is always possible to replace the crossed edges with non-interacting edges which are less cost. This indicates that $x_\ast$ is not 2-optimal, which contradict the assumption.
\end{proof}

\subsection{Definition of the HC Transformation}

Given a TSP $f_o$~\footnote{In the sequel, without causing confusion, we do not differentiate the name of a TSP and its objective function.} with known local or global optimum $x^*$, a convex-hull TSP, denoted by $f_c$, is firstly constructed based on $x^*$. The construction makes sure that $f_c$ has the same optimum $x^*$ as $f_o$. Assume that the original TSP $f_o$ has $n$ cities, to construct the corresponding convex-hull TSP $f_c$, the HC transformation simply let $n$ cities uniformly distribute on a circle following the order they appear in $x^*$ of $f_o$. To level the scale of $f_c$ and $f_o$, the city interval on the circle of $f_c$ is set to be
\begin{equation}\label{eq:neighbor_dist}
\mbox{city interval}=\frac{1}{n}\sum_{i=1}^{n}c_{i,m(i)},
\end{equation}
where $c_{i,m(i)}$ denotes the edge cost in the original TSP $f_o$ and $m(i)$ denotes the index of the nearest city to city $i$ in all the cities of $f_o$, i.e.,
\begin{equation}\label{eq:mi}
m(i)=\argmin\limits_{j\in \{1,\dots,n\},j\neq i}c_{i,j}.
\end{equation}The city interval is used to determine the distance between adjoint pair of cities in the circle.
Once all the cities are arranged on the circle, the distance, denoted as $\hat{c}_{i,j}$, between each city pair $(i,j)$ is computed and considered as the edge cost in the convex-hull TSP. Fig.~\ref{fig:circle51} shows an example convex-hull TSP (and its global optimum) constructed for the TSP instance eil51 (Fig.~\ref{fig:eil51}).
\begin{figure}
    \centering
    \subfigure[the global optimum of eil51]{
      \label{fig:eil51}
      \includegraphics[width=0.5\linewidth]{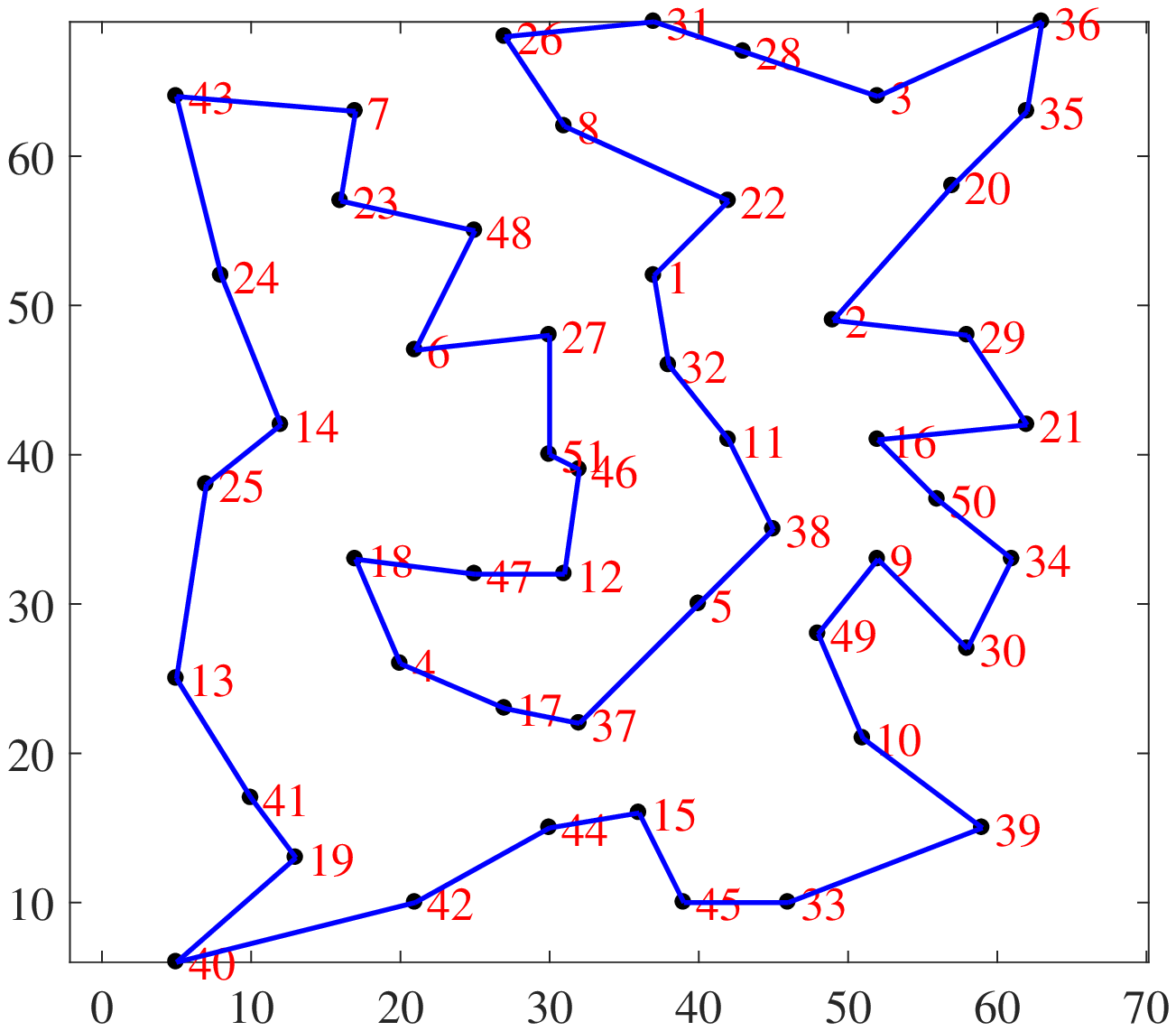}}
      \hspace{-0.2in}
    \subfigure[the convex-hull TSP for eil51]{
      \label{fig:circle51}
      \includegraphics[width=0.5\linewidth]{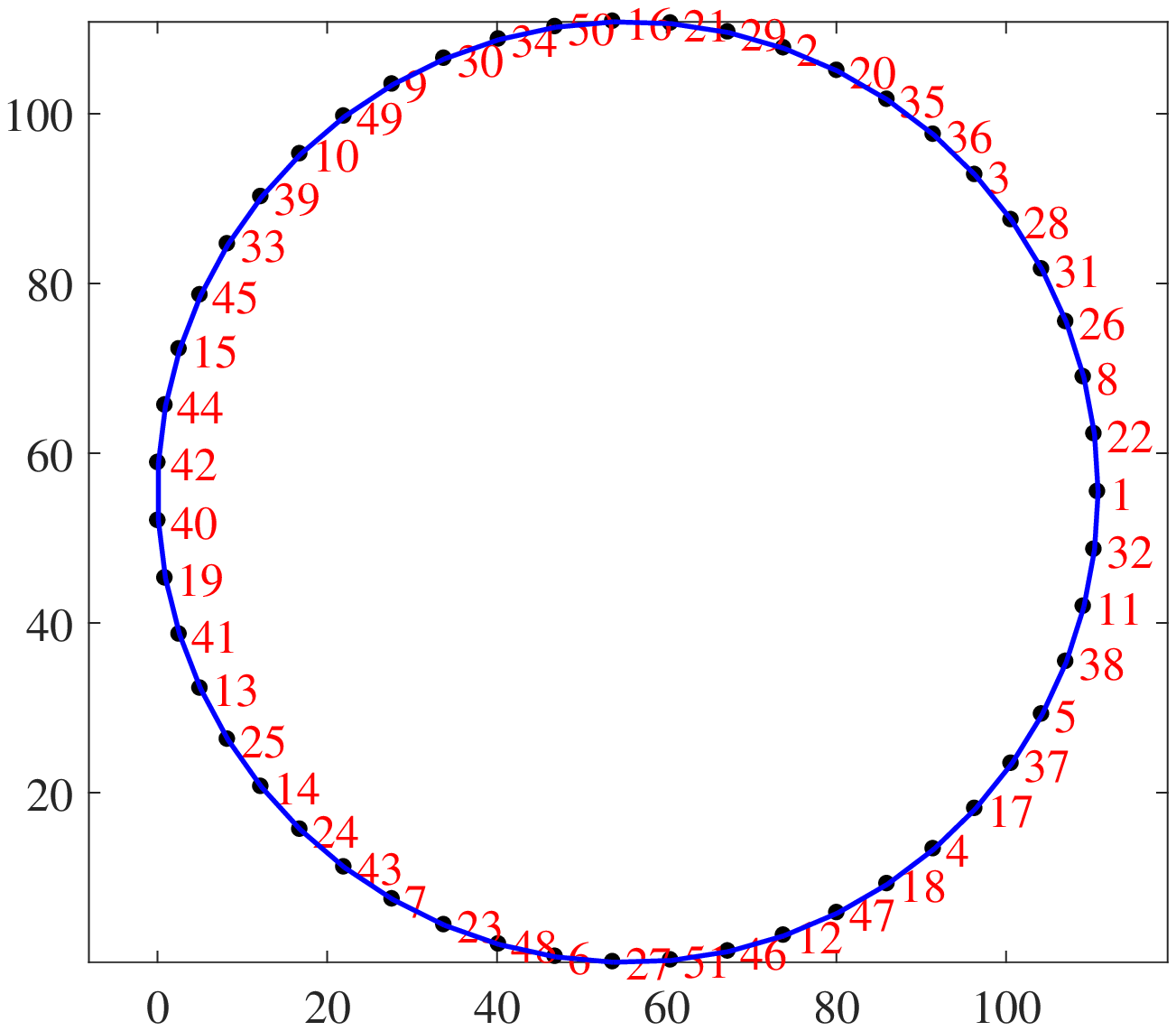}}
      \caption{The global optimum of the TSP instance eil51 and the corresponding convex-hull TSP.}\label{fig:eil51_circle51}
\end{figure}

Once the convex-hull TSP is constructed, the HC transformation is defined by combining the convex-hull TSP with the original TSP by a coefficient $\lambda\in[0,1]$. In the transformed TSP, the edge cost $c'_{i,j}$ between city $i$ and city $j$ is calculated by
\begin{equation}\label{eq:trans_tsp_edge}
  c'_{i,j}(\lambda) = (1-\lambda) c_{i,j} + \lambda \hat c_{i,j}.
\end{equation}
The objective function of the transformed TSP, denoted as $g$, can then be expressed as
\begin{equation}\label{eq:trans_tsp}
  g(x|\lambda) = (1-\lambda) f_o(x) + \lambda f_c(x).
\end{equation}

It is seen that when $\lambda=0$ the transformed TSP $g$ degenerates to the original TSP $f_o$ and when $\lambda=1$ it is smoothed to the convex-hull TSP $f_c$. With increasing $\lambda$, the original TSP $f_o$ can continuously deform into the convex-hull TSP $f_c$. This actually defines a \emph{homotopy} transformation from $f_o$ to $f_c$ (cf. homotopy in topology~\cite{hatcher2002algebraic}). We thus call it the homotopic convex transformation. The sketch of the HC transformation is illustrated in Fig.~\ref{fig:HCT}.
\begin{figure}
  \centering
  \includegraphics[width=0.6\linewidth]{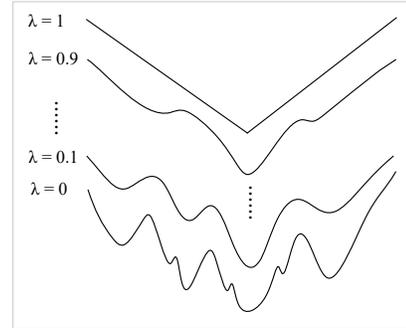}\\
  \caption{A sketch of the effect of the proposed HC transformation.}\label{fig:HCT}
\end{figure}Comparing Fig.~\ref{fig:smooth_gu1994} with Fig.~\ref{fig:HCT}, it is seen that in the TSP smoothing approach proposed in~\cite{gu1994efficient}, with $\alpha \rightarrow \infty$, the landscape becomes flat. While in the HC transformation, with $\lambda \rightarrow 1$, the landscape becomes unimodal. The key difference is that the optimum can be preserved after smoothing by the HC transformation. Note that the transformed TSP is still $\cal {NP}$-hard. For a detailed description, an example of the HC transformation can be found in the supplementary material of this paper.

\subsection{Effects of the HC transformation}

To illustrate the effects of the HC transformation on landscape smoothing, in the following we conduct landscape analysis experiments to study the changes of the TSP landscape after the HC transformation. In the experiments, we execute ILS on the transformed TSP and use the following four metrics to characterize the landscape:
\begin{itemize}
  \item \emph{Local Optimum Density}: It is the number of local optima encountered by an ILS process per 100 moves. Here one move means the local search algorithm moves from the current solution to a new solution.
       \item \emph{Escaping Rate}: It is the success rate that a new local optimum is reached by ILS starting from a new solution obtained by perturbing the current local optimum.
  \item \emph{Fitness Distance Correlation (FDC)}: To calculate the FDC, 1000 local optima (denoted as $x_{LO}$) obtained by ILS are randomly selected. Their function values $f(x_{LO})$ and their distances to the nearest global optimum $d_{opt}$ are counted. Then the FDC is define as
      \begin{equation}\label{eq:FDC}
        \mbox{FDC}(f(x_{LO}),d_{opt}) = \frac{\mbox{cov}(f(x_{LO}),d_{opt})}{\sigma(f(x_{LO}))\sigma(d_{opt})},
      \end{equation}where $\mbox{cov}(\cdot)$ denotes the covariance and $\sigma(\cdot)$ denotes the standard deviation.
  \item \emph{Runtime}: It is the average runtime (CPU time) for ILS to find the global optimum.
\end{itemize}

The first two metrics, local optimum density and escaping rate, measure the ruggedness of the landscape of a TSP. A lower local optimum density means a smoother TSP landscape. It is intuitive that the escaping rate is negatively correlated to the average diameter of the attractive basins of local optima. A large size of attractive basin implies a few local optima on the TSP landscape. Hence a lower escaping rate means a smoother TSP landscape. The third metric, FDC, measures the overall trend of the TSP landscape. A higher FDC means a more regular TSP landscape which implies that a search algorithm can find a path to the global optimum by observing the fitness change of the encountered local optima~\cite{merz2000fitness,merz2004advanced}. The last metric, runtime, directly reflects the hardness of a TSP for ILS. A larger runtime means a harder TSP.

In the experiments, ILS is used to sample local optima on the TSP landscape. As shown in Alg.~\ref{alg:ILS}, ILS iteratively executes a local search procedure and a perturbation procedure till the stopping criterion is met. In our experiments, the 3-Opt local search and double bridge perturbation (please see Fig.~\ref{fig:double_bridge} for its illustration) are selected as the local search method and the perturbation method of ILS, respectively. In Alg.~\ref{alg:ILS},  LocalSearch($x_j^\prime$, $x_{best}$) takes $x_j^\prime$ and the current best $x_{best}$ as input and returns a local optimum $x_{j+1}$ and a new $x_{best} = \arg\min \{f(x_{j+1}), f(x_{best})\}$.
\begin{figure}
  \centering
  \includegraphics[width=0.4\linewidth]{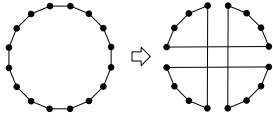}\\
  \caption{An example of the double bridge perturbation on the TSP.}\label{fig:double_bridge}
\end{figure}

\begin{algorithm}
\small
    $x_0 \gets $ random or heuristically generated solution.\;
    $x_0 \gets $ LocalSearch($x_0$)\;
    $j \gets 0$\;
    \While{stopping criterion is not met}{
        $x_j^\prime \gets$ Perturbation($x_j$)\;
        $\{x_{j+1},x_{best}\} \gets$ LocalSearch($x_j^\prime$, $x_{best}$)\;
        $j\gets j+1$\;
    }
    \KwRet{$x_{best}$}
\caption{Iterated Local Search}
\label{alg:ILS}
\end{algorithm}

A total of 12 instances are selected from the TSPLIB as the test instances. The characteristics of the test instances are shown in Table~\ref{tbl:ins}. The third column of Table~\ref{tbl:ins} shows the edge cost type of each test instance, in which \emph{EUC\_2D} means the edge costs are Euclidean distances in a 2D space, \emph{GEO} means the edge costs are geographical distances and \emph{EXPLICIT} means the edge costs are listed explicitly in the TSP description file. The reasons we select these 12 instances are: (1) their sizes are not too small to lose the complexity and ruggedness of the landscape; (2) their sizes are not too large to find their global optima; (3) they belong to different TSP types.
To conduct the analysis, $\lambda$ is set to be $0, 0.01,\dots, 0.09, 0.1$ respectively, for each TSP instance. Here we set $\lambda\leq 0.1$ since our preliminary results show that when $\lambda >0.1$ the property of the transformed TSPs are nearly the same to the property of the convex-hull TSP, which means ILS can find the global optimum immediately.

In the following, we investigate the effects of the HC transformation by using either global optimum or local optimum. The code is implemented in GNU C++ with O2 optimizing compilation. The computing platform is two 6-core 2.00GHz Intel Xeon E5-2620 CPUs (24 Logical Processors) under Ubuntu OS.

\subsection{Effects of the HC transformation with global optimum}\label{sec:known_go}

Since the test instances are from the TSPLIB, their global optima are known. First, we use the global optima of the test instances to construct the convex-hull TSPs and conduct HC transformation with different $\lambda$ values. On each transformed TSP, 1000 runs of ILS are executed from random solutions and terminate only when the global optima are reached. We record all the local optima ever encountered by ILS and other relating information in the entire search history. Then we compute the four metrics for each transformed TSP, respectively.


\begin{table*}
\caption{Test Instances} \label{tbl:ins}\centering
\resizebox{\textwidth}{!}{
\begin{tabular}{l | l l l l l l l l l l l l}
\hline
Instance & eil51 & berlin52 & st70 & pr76 & rat99 & rd100 & ch130 & kroA150 & gr96 & brazil58 & gr120 & si175 \\
\hline
Size & 51 & 52 & 70 & 76 & 99 & 100 & 130 & 150 & 96 & 58 & 120 & 175 \\
\hline
Edge type & EUC\_2D & EUC\_2D & EUC\_2D & EUC\_2D & EUC\_2D & EUC\_2D & EUC\_2D & EUC\_2D & GEO & EXPLICIT & EXPLICIT & EXPLICIT \\
\hline
\end{tabular}
}
\end{table*}


Fig.~\ref{fig:FLA_LO_den} shows how the local optimum density of the transformed TSP landscape changes against $\lambda$. From Fig.~\ref{fig:FLA_LO_den} we can see that, on the instances with Euclidean edge cost type (denoted by ``EUC''), in general the local optimum density decrease as $\lambda$ increases, which means that the HC transformation can smooth the TSP landscape and its smoothing effect is controlled by $\lambda$. However, on the instances with non-Euclidean edge cost type (``GEO'' and ``EXP''), the local optimum density is not always negatively related to $\lambda$. For example, it is seen that  the local optimum density of the TSP landscape for brazil58 increases when $0.01\leq\lambda\leq0.08$.

The escaping rates of ILS on different transformed TSPs are shown in Fig.~\ref{fig:FLA_escapeR}. As stated before, a lower escaping rate means a smoother TSP landscape. From Fig.~\ref{fig:FLA_escapeR} we can see that, on all of the Euclidean TSP instances, in general the escaping rate decrease as $\lambda$ increases, which exhibits similar phenomena as the local optimum density. On the non-Euclidean TSP instances, it seems that the escaping rate also decrease as $\lambda$ increases. However, the decline of the escaping rate is not very obvious on brazil58 and si175 when $\lambda\geq0.01$. Based on the above observations, we may conclude that the HC transformation can indeed smooth the landscapes of all the Euclidean TSP and some non-Euclidean TSP instances.

Fig.~\ref{fig:FLA_FDC} shows how the FDC of the transformed TSP changes against $\lambda$. From Fig.~\ref{fig:FLA_FDC} we can see that, although the FDC curve has fluctuations, in general the FDC increases as $\lambda$ increases on most test instances. An exception is found on the non-Euclidean instance si175, on which the lowest FDC value appears when $\lambda=0.02$ and when $\lambda\geq0.04$ the FDC is always 1. Based on the above observations, we may conclude that the HC transformation can increase the FDC value on most of the test instances.

On different transformed TSPs, the runtime for ILS to find the global optimum is shown in Fig.~\ref{fig:FLA_runtime}. From Fig.~\ref{fig:FLA_runtime} we can see that, in general, the runtime of ILS decreases as $\lambda$ increases on most of the test instances. An exception is found on the non-Euclidean instance brazil58, on which the runtime first increases then decreases. Based on the above observation, we conclude that the HC transformation can reduce the running time of ILS to find the global optimum simply because the fitness landscape of the transformed TSP is smoothed with fewer number of local optima and higher FDC than the original TSP.  Of course, this claim is based on the premise that the instance size is relatively small.

\begin{figure*}
  \subfigure[LO density]{
    \label{fig:FLA_LO_den} 
    \includegraphics[height=0.202\linewidth]{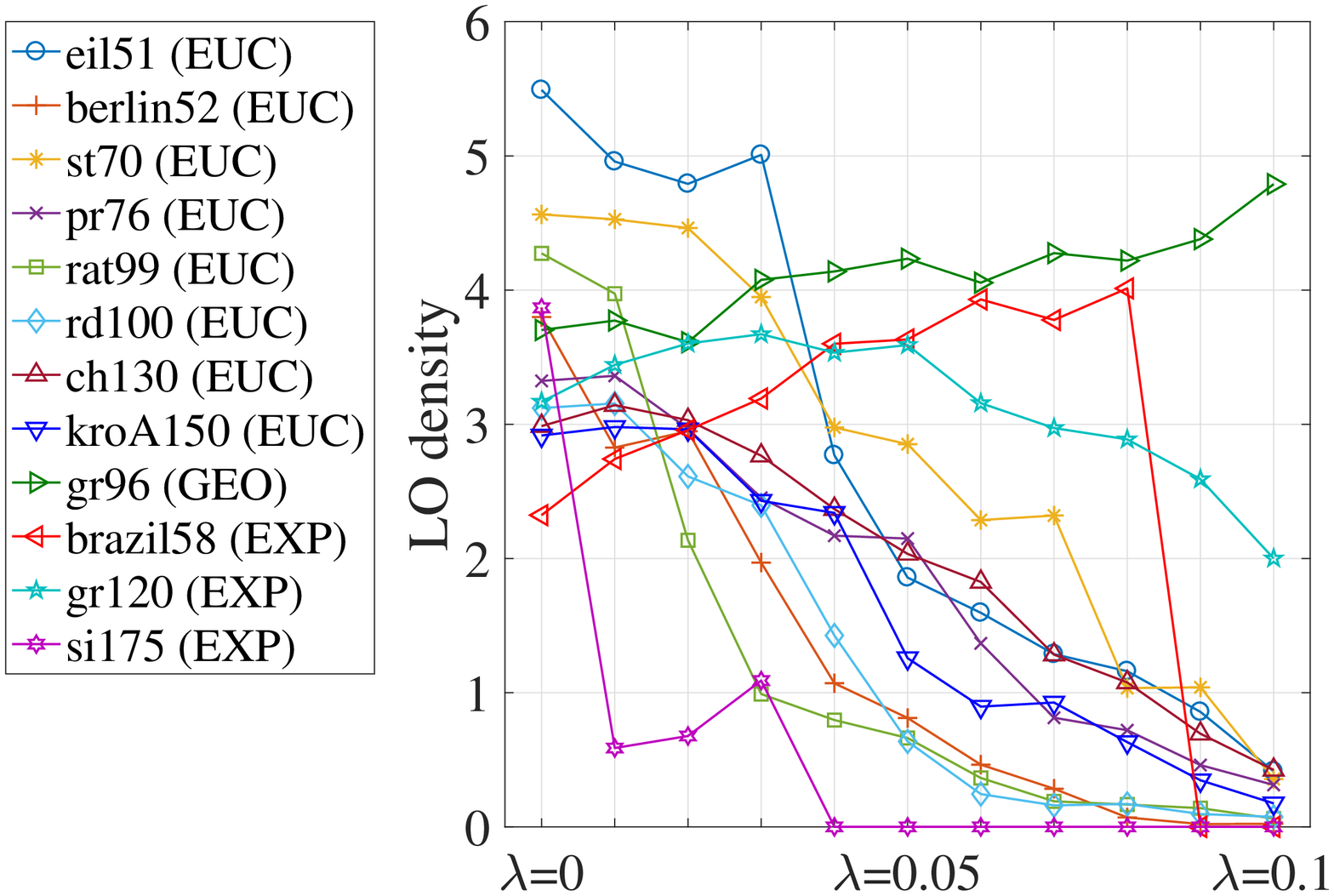}}
    \hspace{-0.05in}
  \subfigure[Escaping rate]{
    \label{fig:FLA_escapeR} 
    \includegraphics[height=0.202\linewidth]{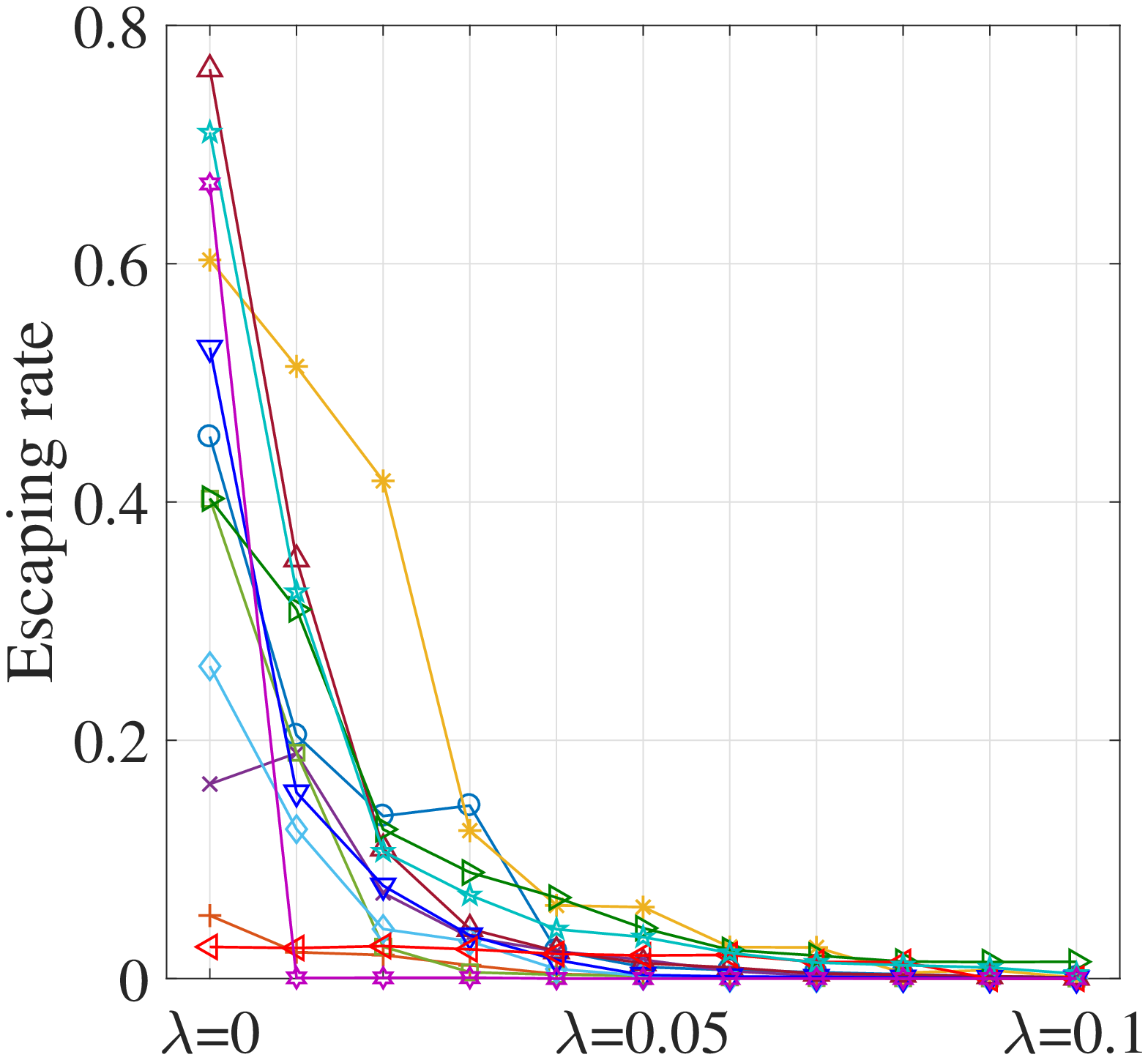}}
    \hspace{-0.05in}
  \subfigure[FDC]{
    \label{fig:FLA_FDC} 
    \includegraphics[height=0.202\linewidth]{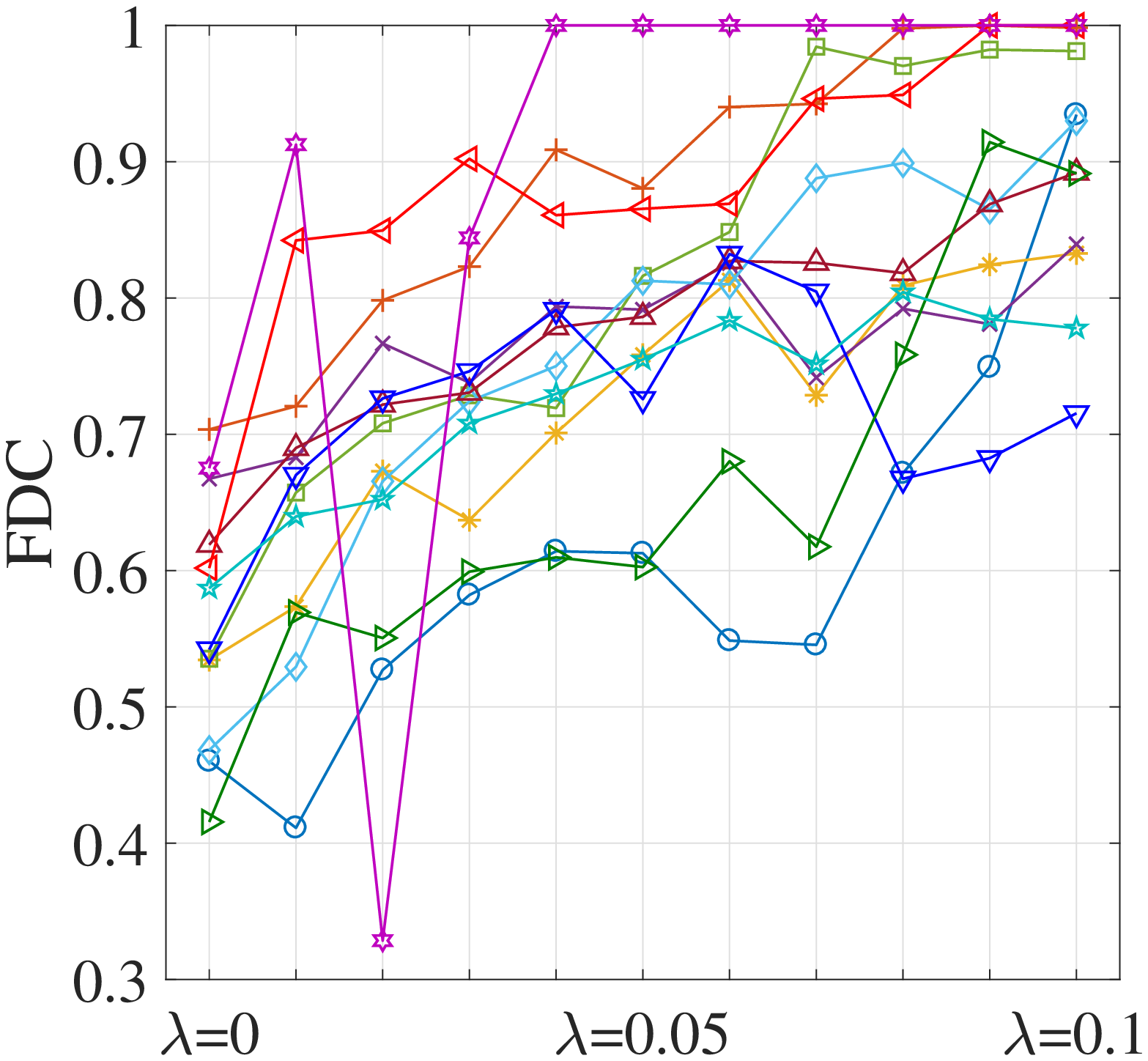}}
    \hspace{-0.05in}
  \subfigure[ILS Runtime]{
    \label{fig:FLA_runtime} 
    \includegraphics[height=0.202\linewidth]{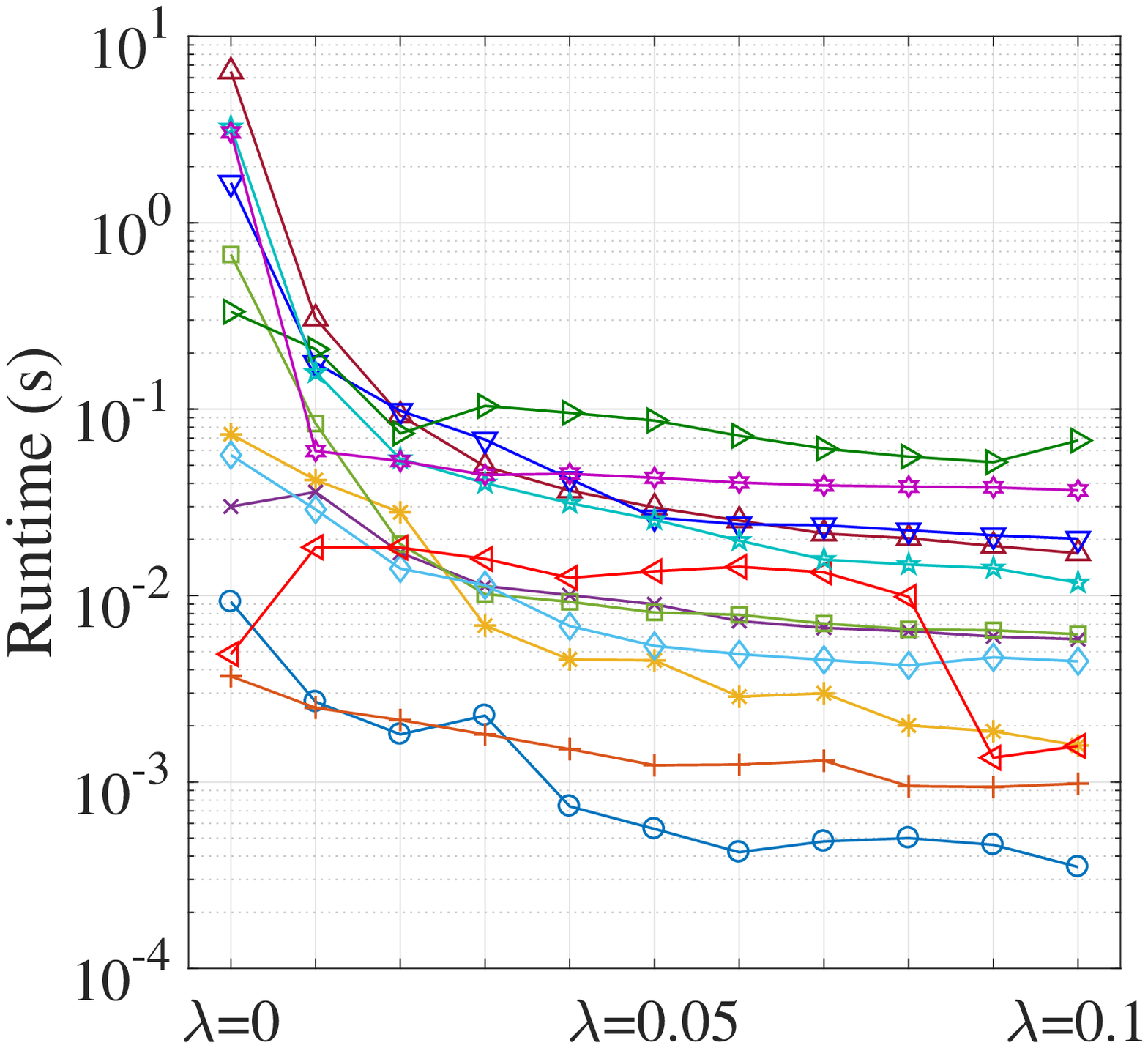}}
  \caption{Landscape analysis results of the transformed TSPs with different $\lambda$ values based on the global optimum of the original TSP.}\label{fig:FLA}
\end{figure*}
\begin{figure*}
  \subfigure[LO density]{
    \label{fig:FLA_LO_den_LOc} 
    \includegraphics[height=0.202\linewidth]{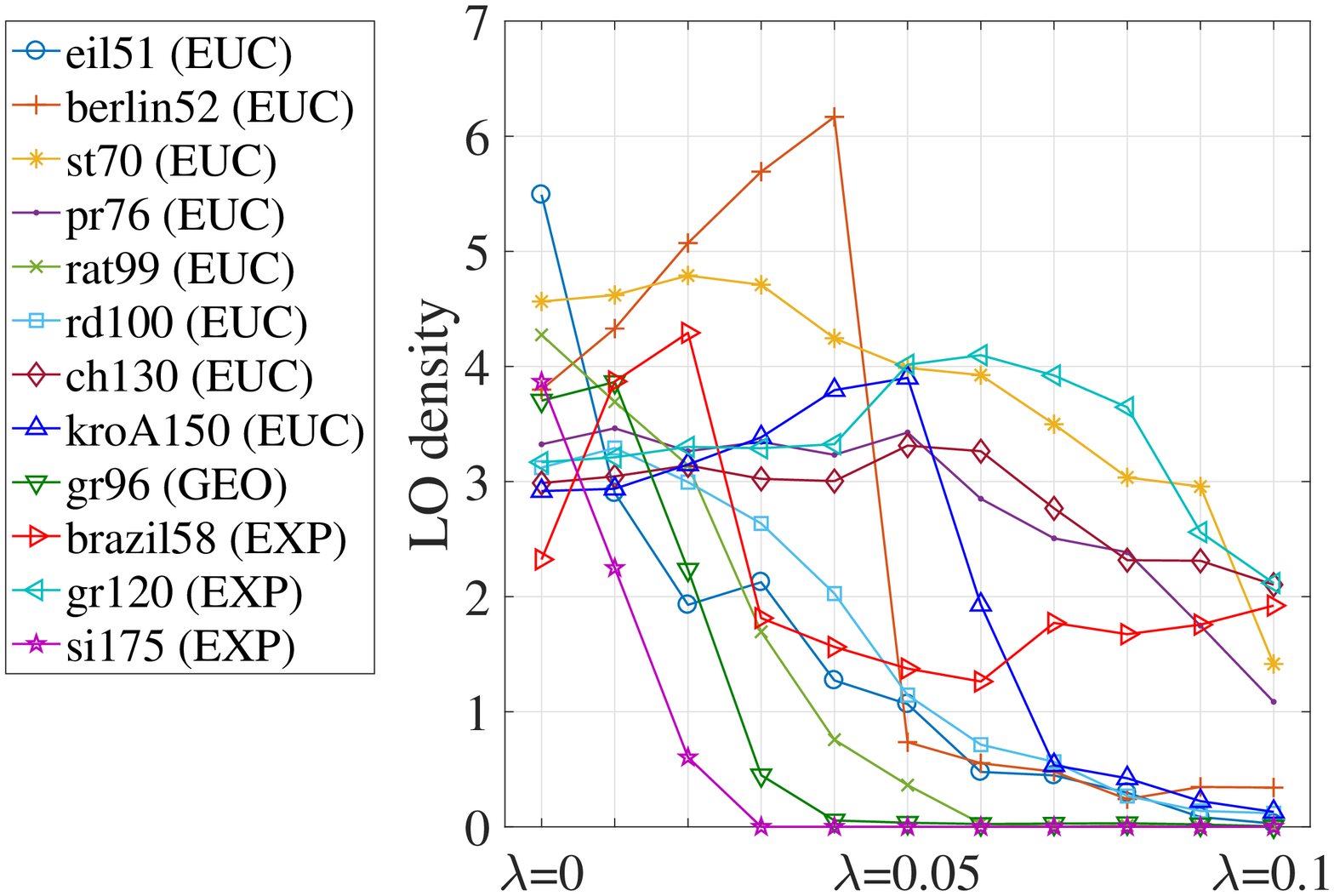}}
    \hspace{-0.05in}
  \subfigure[Escaping rate]{
    \label{fig:FLA_escapeR_LOc} 
    \includegraphics[height=0.202\linewidth]{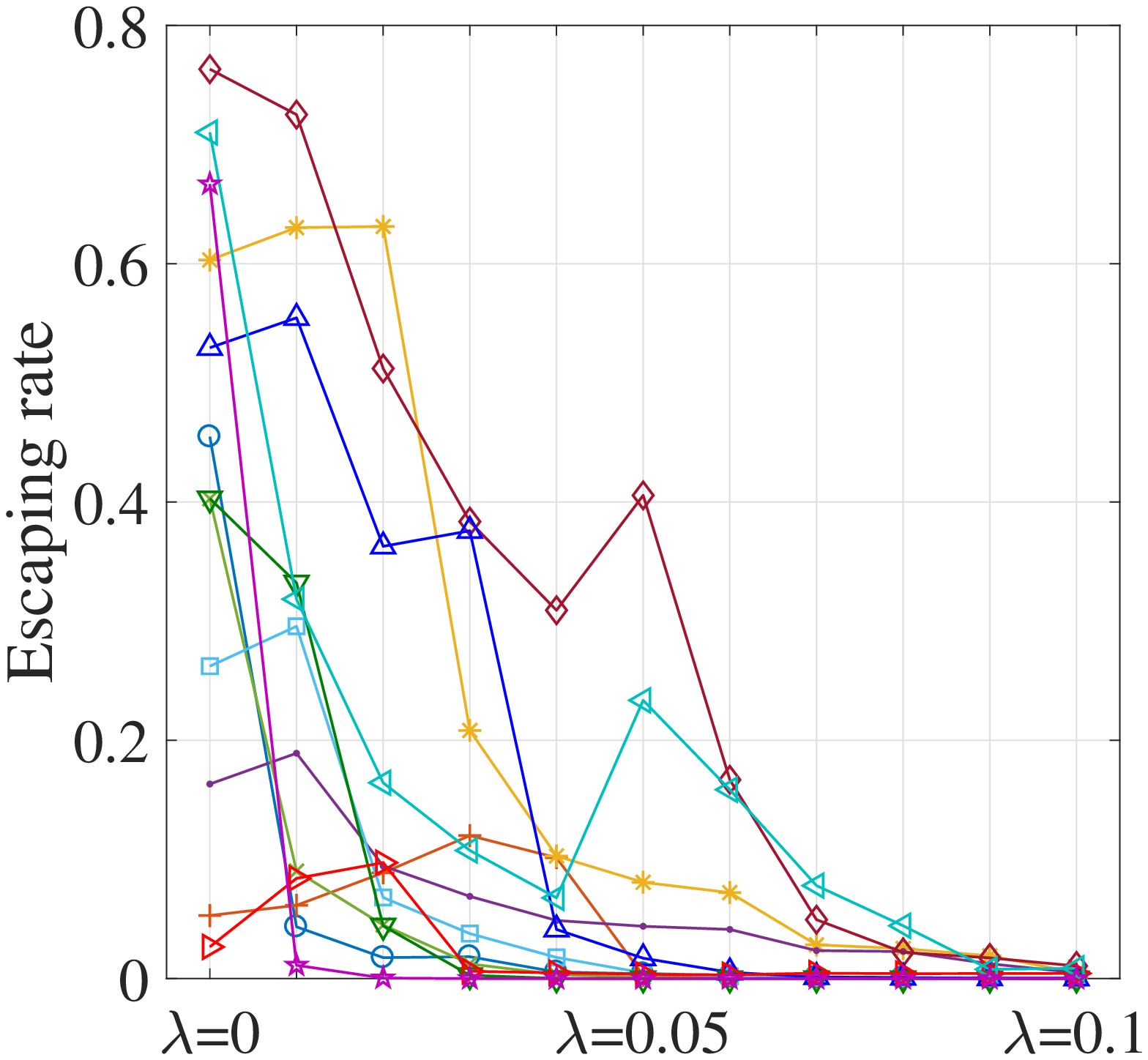}}
    \hspace{-0.05in}
  \subfigure[FDC]{
    \label{fig:FLA_FDC_LOc} 
    \includegraphics[height=0.202\linewidth]{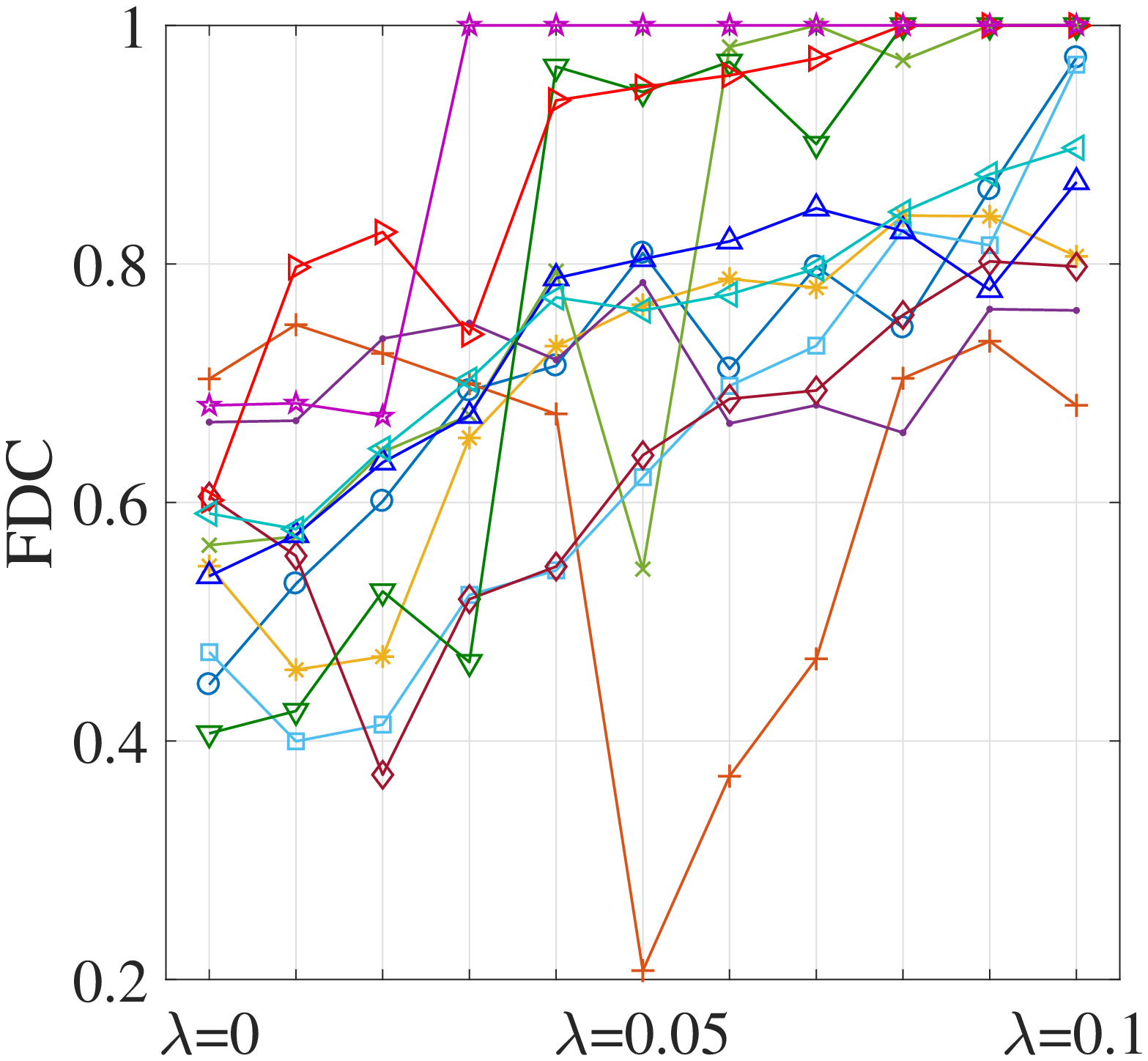}}
    \hspace{-0.05in}
  \subfigure[Runtime]{
    \label{fig:FLA_runtime_LOc} 
    \includegraphics[height=0.202\linewidth]{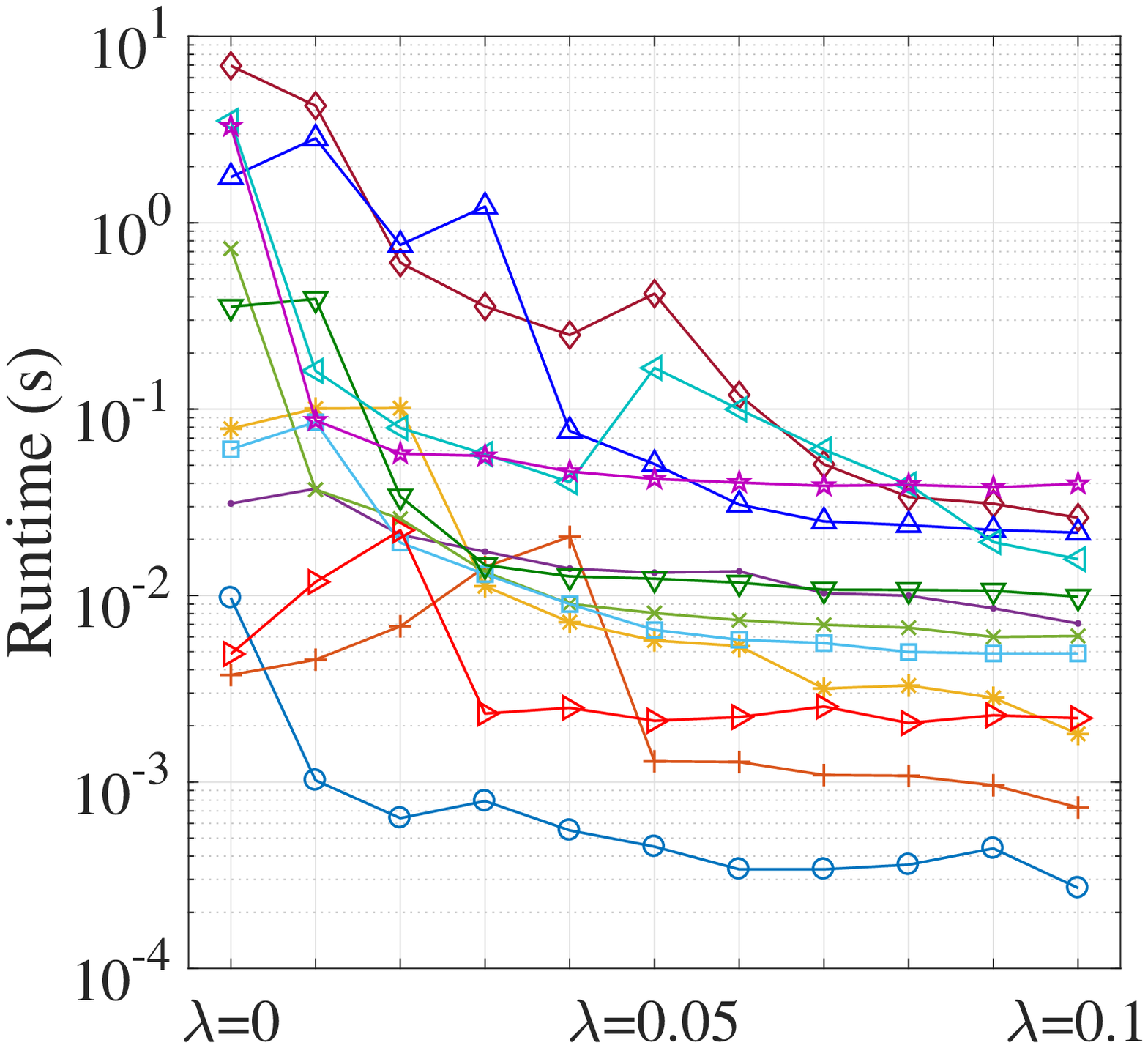}}
  \caption{The landscape analysis results of the transformed TSPs with different $\lambda$ values based on a local optimum of the original TSP.}\label{fig:FLA_LOc}
\end{figure*}

In summary, we may conclude that in most cases the proposed HC transformation can indeed smooth the TSP landscape and the strength on the smoothing can be controlled by $\lambda$. This effect is more significant in the instances with the Euclidean edge cost type, which is the most common edge cost type in the TSP, than with other TSP types.

\subsection{Effects of the HC Transformation with local optimum}\label{sec:unknown_go}

In practice, the global optimum is not known. We therefore investigate the effects of the HC transformation when a local optimum is used to construct the convex-hull TSP.  The same test instances in Table~\ref{tbl:ins} are used. For each instance, we run a round of 3-Opt local search to obtain a local optimum $x_{LO}$. Then we use this local optimum to construct the convex-hull TSP with different $\lambda$ values. The function value of the used local optima are listed in Table~\ref{tbl:ins_LOc}. Table~\ref{tbl:ins_LOc} also lists the excess of the used local optima, which is defined in Eq.~(\ref{eq:excess}): \begin{equation}\label{eq:excess}
    \mbox{excess} = \frac{f(x_{LO}) - f(x_{opt})}{f(x_{opt})} \times 100\%,
\end{equation}where $x_{opt}$ is the global optimum. The other experimental settings are the same as those in Section~\ref{sec:known_go}.

\begin{table}
\caption{Costs of the selected optima, where $x_{opt}$ and $x_{LO}$ denotes the global and local optimum, respectively.} 
\centering
\label{tbl:ins_LOc}
\begin{tabular}{l l l l}
\hline
Instance & $f(x_{opt})$ & $f(x_{LO})$ & excess (\%) \\
\hline
eil51 & 426 & 428 & 0.4695\\
berlin52 & 7542 & 7902 & 4.7733\\
st70 & 675 & 684 & 1.3333\\
pr76 & 108159 & 109190 & 0.9532\\
rat99 & 1211 & 1217 & 0.4955\\
rd100 & 7910 & 8054 & 1.8205\\
ch130 & 6110 & 6363 & 4.1408\\
kroA150 & 26524 & 27040 & 1.9454\\
gr96 & 55209 & 55507 & 0.5398\\
brazil58 & 25395 & 25643 & 0.9766\\
gr120 & 6942 & 7066 & 1.7862\\
si175 & 21407 & 21485 & 0.3644\\
\hline
\end{tabular}
\end{table}

Since the HC transformation is based on local optima, the global optima of the transformed TSPs may be different from the original global optima. To conduct landscape analysis experiments, we first use the LKH software~\cite{helsgaun2000effective} to obtain the global optima of the transformed TSPs. On each transformed TSP, we run the LKH 100 times and record the best solution as the global optimum. Since the LKH is one of the state-of-the-art algorithms for the TSP and the size of the transformed TSPs is relatively small, we believe the TSP solutions found by the LKH are indeed the global optima of the transformed TSPs.

The local optimum density values of different transformed TSPs are shown in Fig.~\ref{fig:FLA_LO_den_LOc}. From Fig.~\ref{fig:FLA_LO_den_LOc} we can see that in 11 of the 12 test instances, the local optimum density is approximately negatively related to $\lambda$. Specifically, on seven instances (eil51, st70, pr76, rat99, rd100, gr96 and si175) the local optimum density decreases as $\lambda$ increases. On four instances (berlin52, ch130, kroA150 and gr120), the local optimum density first increases a little then starts to decrease.

Experimental results for the escaping rate are shown in Fig.~\ref{fig:FLA_escapeR_LOc}. From Fig.~\ref{fig:FLA_escapeR_LOc} we can see that in 9 of the 12 test instances, the escaping rate is approximately negatively related to $\lambda$. Specifically, on five instances (eil51, rat99, ch130, gr96 and gr120), in general the escaping rate decreases as $\lambda$ increases in spite of the fact that on some instances the curve of the escaping rate is not very smooth. On four instances (st170, pr76, rd100 and kroA150), the escaping rate first increases a little then starts to decrease. Through comparing Fig.~\ref{fig:FLA_LO_den_LOc} and Fig.~\ref{fig:FLA_escapeR_LOc} with Fig.~\ref{fig:FLA_LO_den} and Fig.~\ref{fig:FLA_escapeR}, we conclude that, on most of the test instances, the HC transformation based on local optima can achieve approximately the same smoothing effect as the HC transformation based on global optima.

Fig.~\ref{fig:FLA_FDC_LOc} shows the FDC versus $\lambda$ when the HC transformation is based on local optima. From Fig.~\ref{fig:FLA_FDC_LOc} we can see that in 9 of the 12 test instances, the FDC is approximately positively related to $\lambda$. Specifically, on four instances (eil51, kroA150, gr96 and brazil58), there is a general trend that the FDC increases as $\lambda$ increases, in spite of the fact that on some instances the FDC curve is not very smooth. On four instances (st70, rd100, ch130 and gr120), the FDC first decreases a little then starts to increase. On rat99, although the lowest FDC appears when $\lambda=0.05$, but in general the FDC is increase as $\lambda$ increases. Through comparing Fig.~\ref{fig:FLA_FDC_LOc} to Fig.~\ref{fig:FLA_FDC} we can conclude that, on most of the test instances, the HC transformation based on local optima can achieve approximately the same FDC-increasing effect as the HC transformation based on global optima.

Fig.~\ref{fig:FLA_runtime_LOc} shows the runtime for ILS to find the global optimum when the HC transformation is based on local optima. From Fig.~\ref{fig:FLA_runtime_LOc} we can see that in 10 of the 12 test instances, the runtime of ILS is approximately positively related to $\lambda$. Specifically, on five instances (eil51, rat99, ch130, gr120 and si175), the runtime decreases as $\lambda$ increases, while on another five instances (st70, pr76, rd100, kroA150 and gr96), the runtime first increases a little then starts to decrease. Through comparing Fig.~\ref{fig:FLA_runtime_LOc} to Fig.~\ref{fig:FLA_runtime} we can conclude that, in most of the test instances, the HC transformation based on local optima can achieve approximately the same runtime reducing effect as the HC transformation based on global optima.

Based on the above observations, we may conclude that, in most test instances, the HC transformation based on local optima can achieve approximately the same effects as the the HC transformation based on global optima.

\subsection{Influences of different local optima}\label{sec:diff_lo}

In previous experiments, on some test instances, especially on berlin52, the HC transformation based on local optima does not achieve as good smoothing and FDC-increasing effect as the HC transformation based on global optima. Recall that in Table~\ref{tbl:ins_LOc}, the local optimum used to construct the convex-hull TSP for berlin52 has a relatively low quality compared to the local optima of other instances. We thus conjecture that the effects depend highly on the quality of the local optimum used in the HC transformation. To confirm, we conduct the following experiments.


For berlin52, using the 3-Opt local search, we find totally 89 local optima with different function values ranging from 7658 to 8385. Based on these 89 local optima and the global optimum (with function value 7542), the experimental configurations used in previous sections are applied on berlin52. The obtained results on the four metrics against different local optima (and the global optimum) and $\lambda$ values are shown in Fig.~\ref{fig:berlin52_MLO} while Fig.~\ref{fig:LO_density_berlin52_MLO_csi}-\ref{fig:runtime_berlin52_MLO_csi} show the cubic spline interpolation of the results. 


From the results we can see that, when the local optima have relatively high quality (the blue curves in Fig.~\ref{fig:LO_density_berlin52_MLO}-\ref{fig:runtime_berlin52_MLO}) the general tendency of the curves achieved by the local optima are similar to that achieved by the global optimum (the  curves with the lowest function value in Fig.~\ref{fig:LO_density_berlin52_MLO}-\ref{fig:runtime_berlin52_MLO}), while when the local optima have relatively low quality (the red curves in Fig.~\ref{fig:LO_density_berlin52_MLO}-\ref{fig:runtime_berlin52_MLO}), the difference between the curves achieved by local optima and global optimum are quite clear. The cubic spline interpolation results (Fig.~\ref{fig:LO_density_berlin52_MLO_csi}-\ref{fig:runtime_berlin52_MLO_csi}) also show the same tendency. This means that the HC transformation based on a high-quality local optimum can achieve approximately the same transformation performance as the HC transformation based on the global optimum. Hence, we conclude that the HC transformation using a high-quality local optimum can smooth the TSP landscape similar to as using the global optimum.

\section{Algorithmic Framework}\label{sec:LSILS}

We claim that the proposed HC transformation can be used to improve the global search ability of existing TSP heuristics. To justify, in this section, we first propose a general framework in which the HC transformation is combined with a local search for the TSP. Then two instantiated algorithms of the framework with the 3-Opt local search and the LK local search, respectively, are proposed and compared against some known TSP smoothing algorithms. Note here that our objective is not to propose a state-of-the-art algorithm for the TSP, but to show that the proposed HC transformation can be used to improve existing TSP heuristics.

\begin{figure*}
  \subfigure[Local optimum density]{
    \label{fig:LO_density_berlin52_MLO} 
    \includegraphics[height=0.310\linewidth]{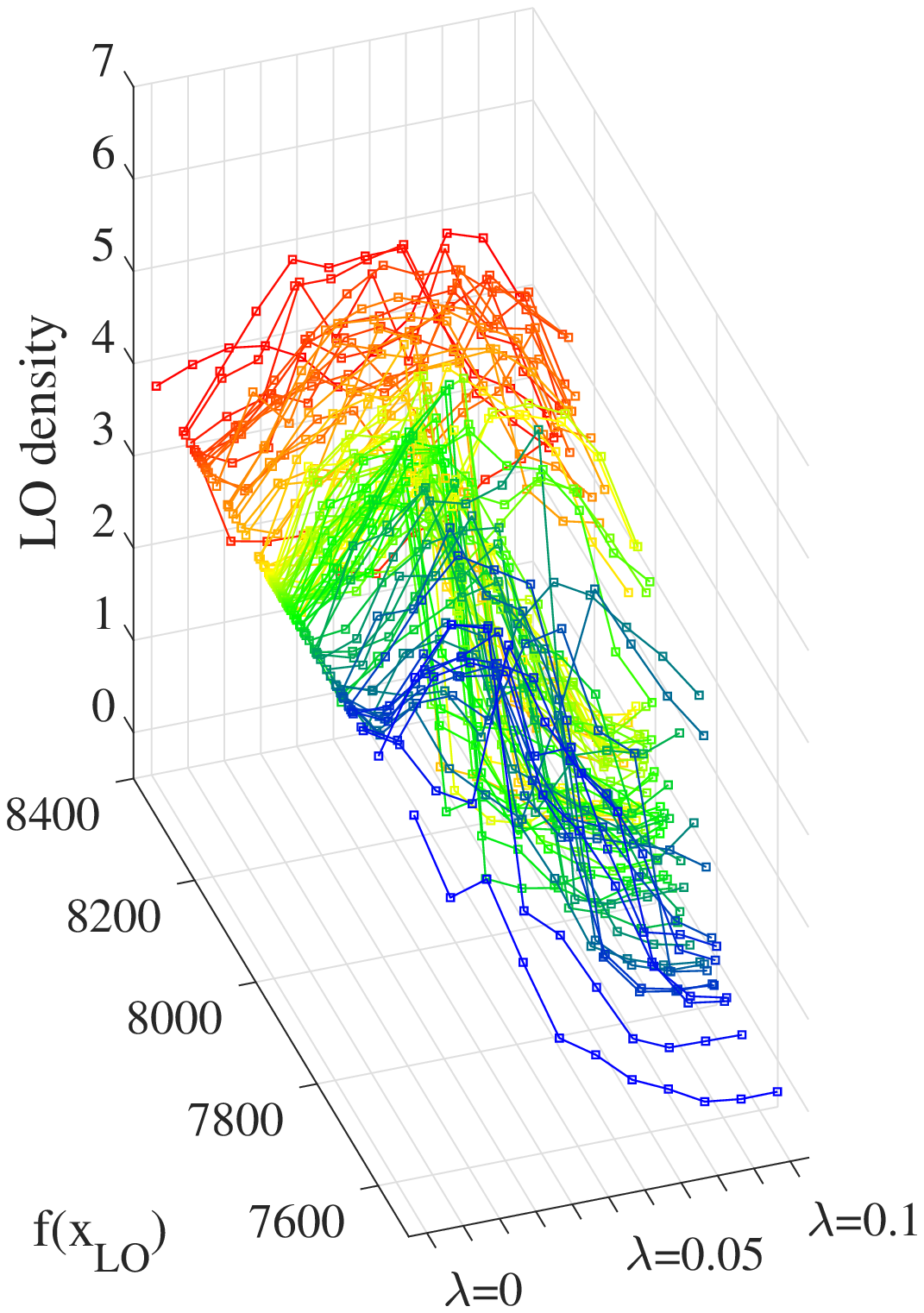}}
    \hspace{0in}
  \subfigure[Escaping rate]{
    \label{fig:escaping_rate_berlin52_MLO} 
    \includegraphics[height=0.30\linewidth]{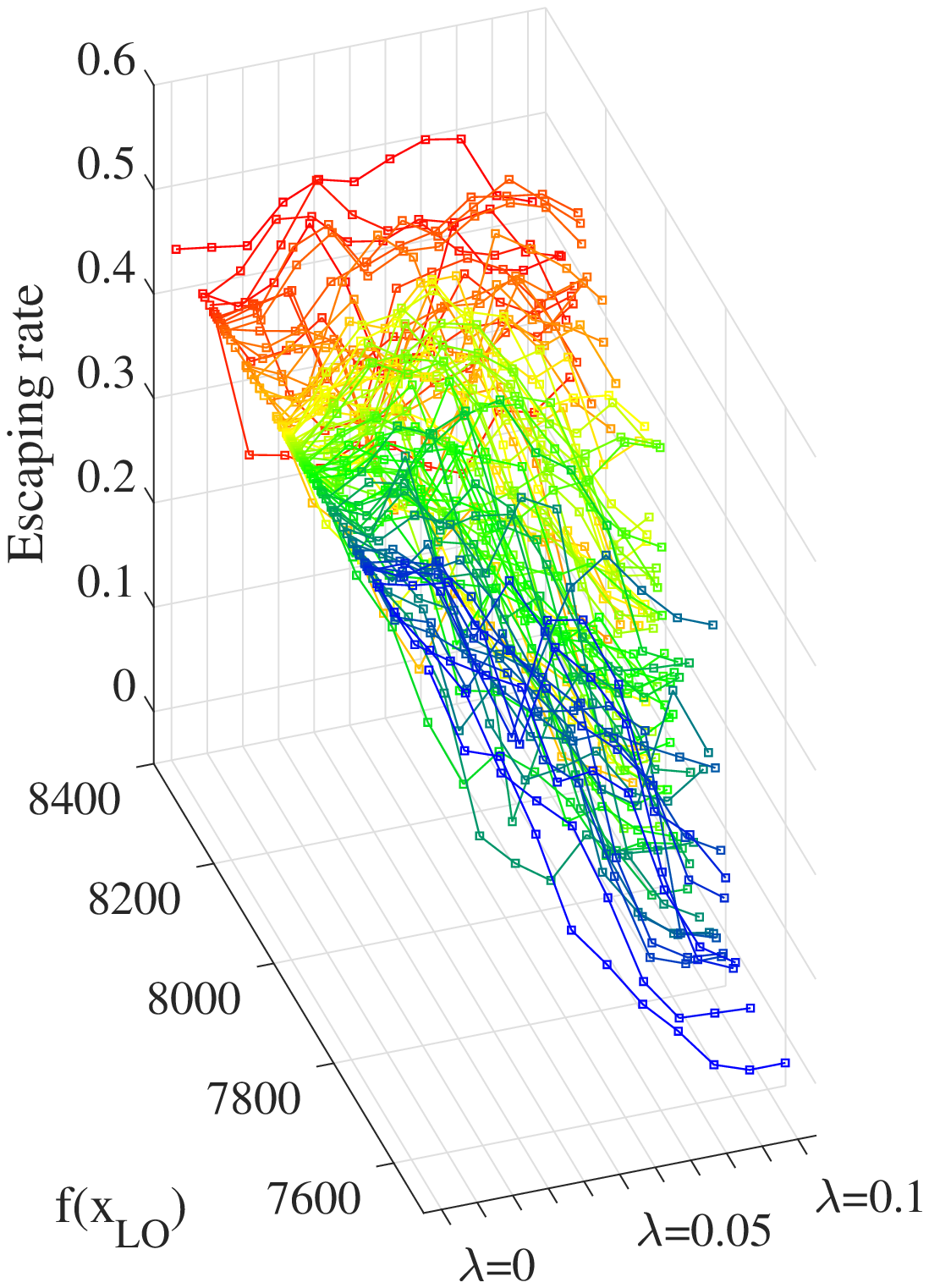}}
    \hspace{0in}
  \subfigure[FDC]{
    \label{fig:FDC_berlin52_MLO} 
    \includegraphics[height=0.30\linewidth]{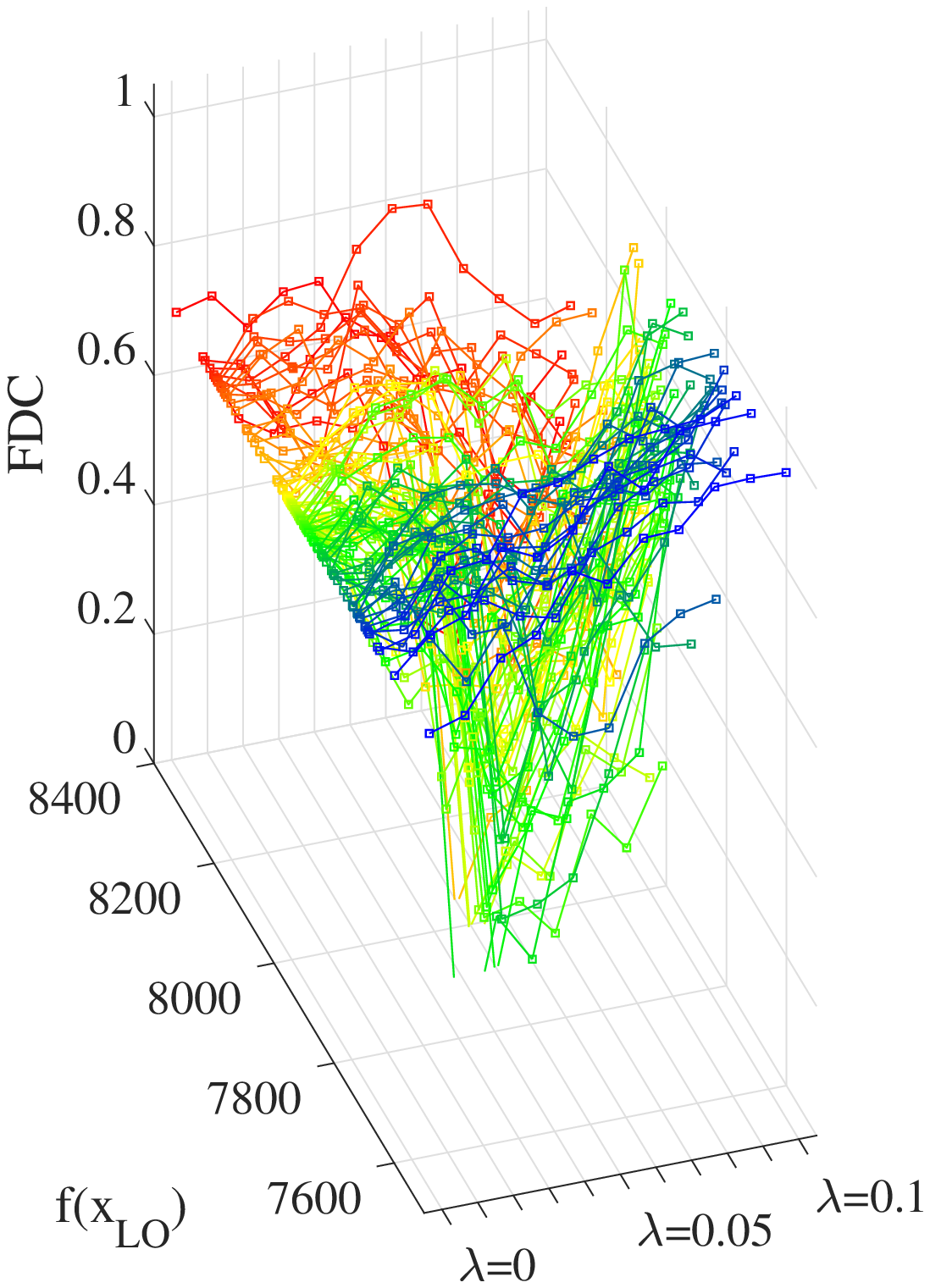}}
    \hspace{0in}
  \subfigure[Runtime]{
    \label{fig:runtime_berlin52_MLO} 
    \includegraphics[height=0.30\linewidth]{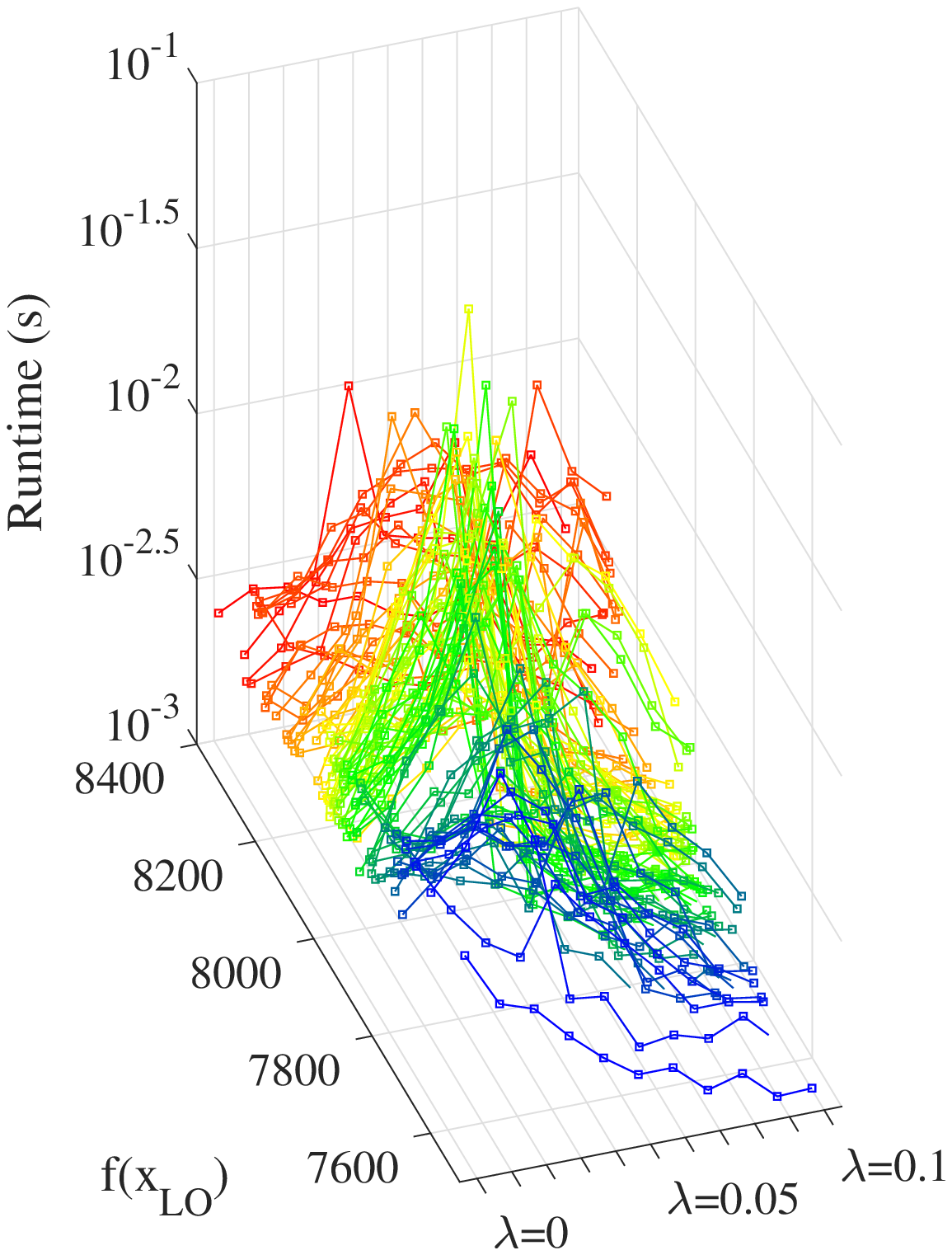}}\\
  \subfigure[Local optimum density (cubic spline interpolation)]{
    \label{fig:LO_density_berlin52_MLO_csi} 
    \includegraphics[height=0.30\linewidth]{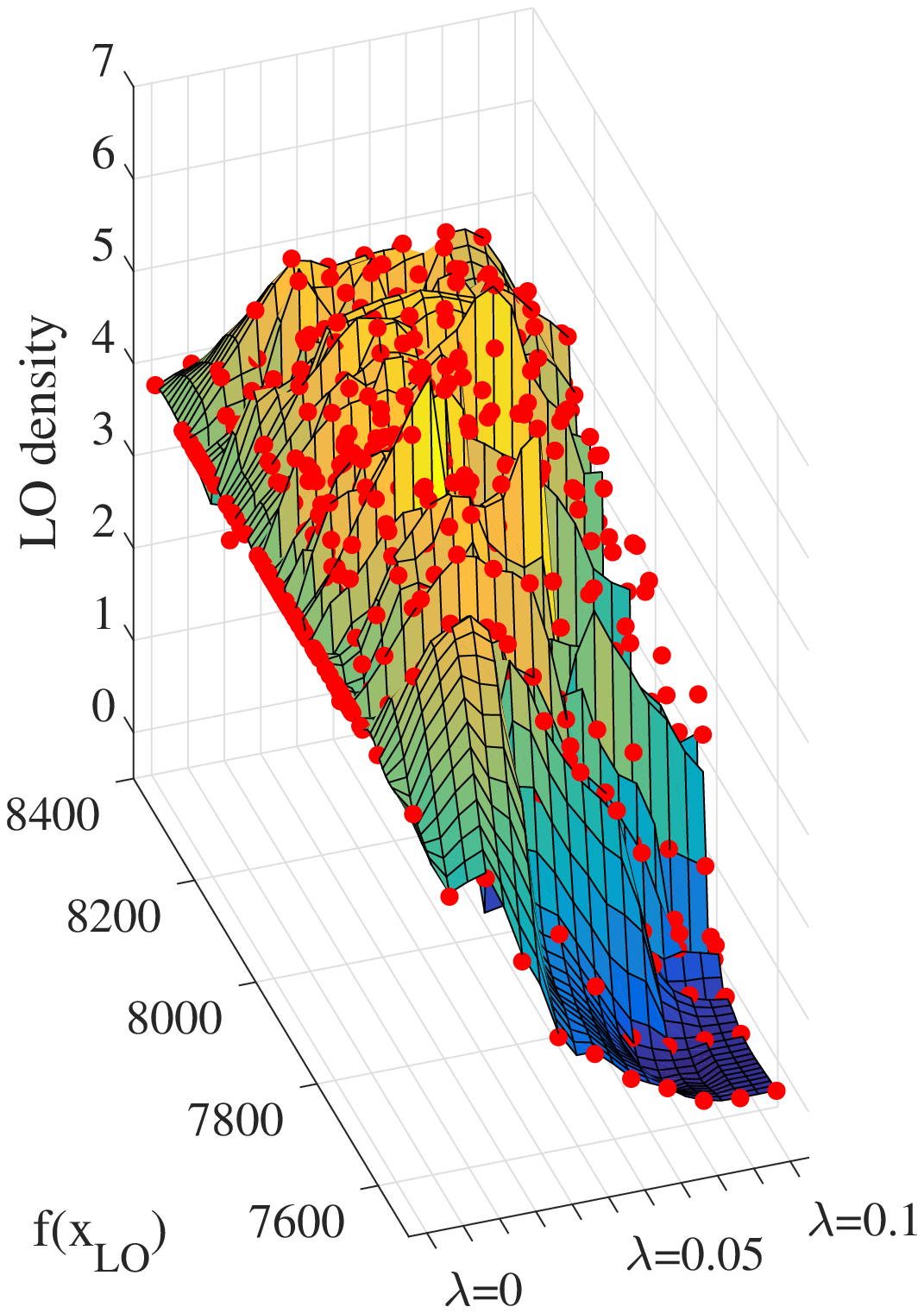}}
    \hspace{0in}
  \subfigure[Escaping rate (cubic spline interpolation)]{
    \label{fig:escaping_rate_berlin52_MLO_csi} 
    \includegraphics[height=0.30\linewidth]{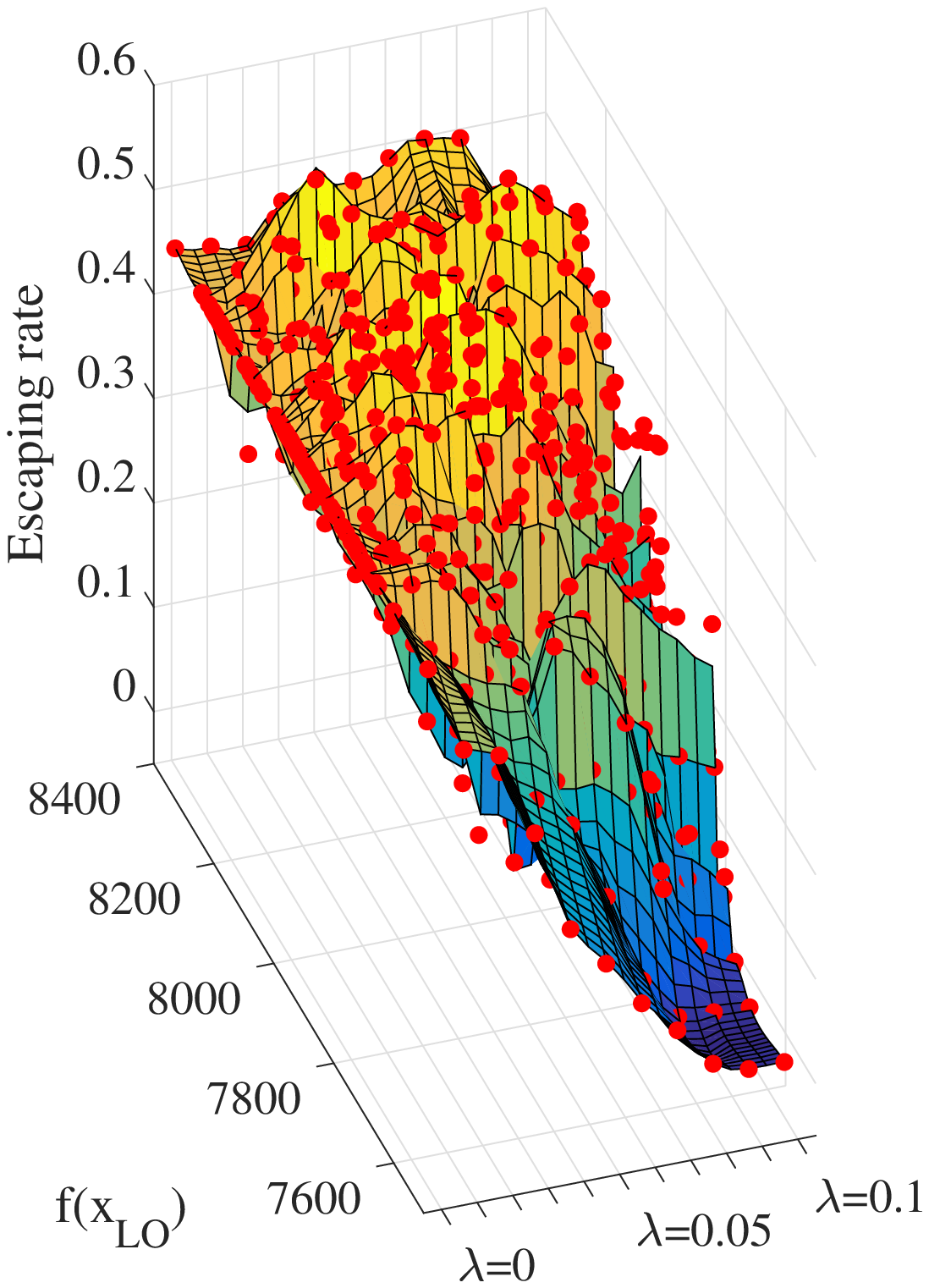}}
    \hspace{0in}
  \subfigure[FDC (cubic spline interpolation)]{
    \label{fig:FDC_berlin52_MLO_csi} 
    \includegraphics[height=0.30\linewidth]{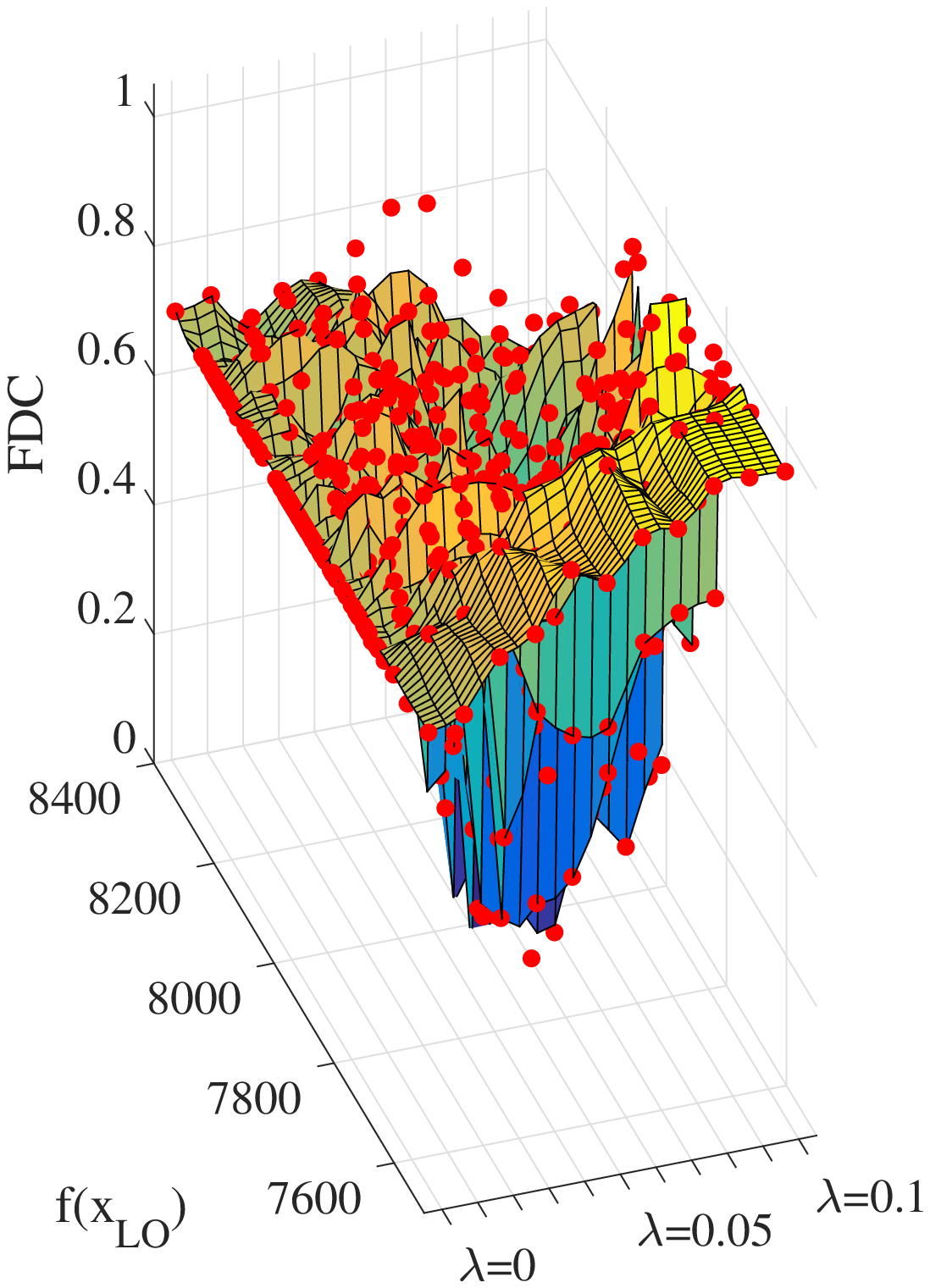}}
    \hspace{0in}
  \subfigure[Runtime (cubic spline interpolation)]{
    \label{fig:runtime_berlin52_MLO_csi} 
    \includegraphics[height=0.30\linewidth]{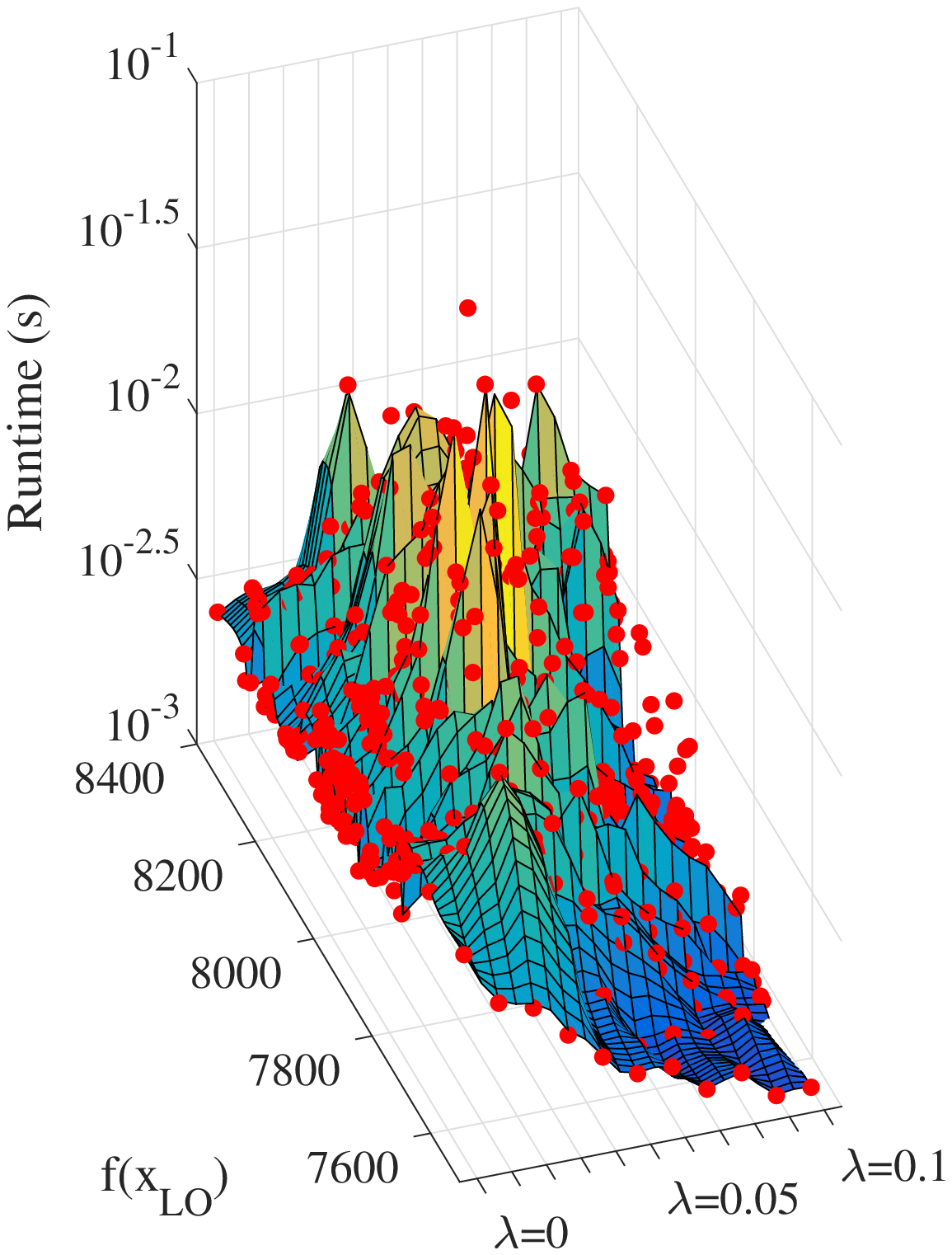}}
  \caption{Landscape analysis results of the transformed berlin52 based on 89 local optima and the global optimum with different $\lambda$ values. (a)-(d) show the metric curves of different local optima; (e)-(h) show the cubic spline interpolation of the results. Here red and blue color indicates high and low function value, respectively. }\label{fig:berlin52_MLO}
\end{figure*}

\subsection{The general framework}

The general framework is summarized in Alg.~\ref{alg:LSILS}. Similar to solution perturbation based algorithms, the framework iteratively executes a local search procedure and a perturbation procedure. The key is that the local search is performed on the transformed TSP $g = (1-\lambda)f_o + \lambda f_c$. If we set $\lambda=0$, then $g=f_o$ and Alg.~\ref{alg:LSILS} degenerates to ILS in Alg.~\ref{alg:ILS}. A local optimum, denoted as $x_0$, is first obtained by any local search algorithm (line~\ref{fmk01}) from a random initial solution $x_{ini}$ (line~\ref{fmk01}). The algorithm repeats the following procedure until stop. At each iteration, the current best solution w.r.t. $f_o$, denoted as $x_{f_o,best}$, is used to construct a convex-hull TSP (line~\ref{fmk02}) for a transformed TSP $g$ (line~\ref{fmk03}). The local search is then applied on the transformed TSP from a perturbed solution (line~\ref{fmk05}) for a new best $x_{f_o,best}$. The coefficient $\lambda$ is updated (line~\ref{fmk06}). Note that LocalSearch($x'_j,x_{f_o,best}|g$) means executing a local search from $x'_j$ on the transformed TSP $g$, meanwhile keeping updating $x_{f_o,best}$ by tracking the value change of $f_o$. It returns a local optimum $x_{j+1}$ w.r.t. the transformed TSP $g$ and the updated $x_{f_o,best}$ w.r.t. the original TSP $f_o$. In the next iteration, $x_{j+1}$ will be perturbed to generate an initial solution for the next local search. $\lambda$ is updated after local search (line~\ref{fmk06}).

\begin{algorithm}[th!]
\SetKwInput{KwInput}{Input}
\small
    $x_{ini} \gets $ random or heuristically generated solution.\;
    $x_0 \gets $ LocalSearch($x_{ini}|f_o$)\; \label{fmk01}
    $x_{f_o,best} \gets x_0$\;
    $j \gets 0$\;
    \While{\emph{stopping-criterion} is not met}{
        Construct convex-hull TSP $f_c$ based on $x_{f_o,best}$\; \label{fmk02}
        $g \gets (1-\lambda)f_o + \lambda f_c$\; \label{fmk03}
        $x'_j \gets$ Perturbation($x_j$)\; \label{fmk04}
        $\{x_{j+1},x_{f_o,best}\} \gets$ LocalSearch($x'_j,x_{f_o,best}|g$)\; \label{fmk05}
        $j\gets j+1$\;
        $\lambda = $Update($\lambda$)\; \label{fmk06}
    }
    \KwRet{$x_{f_o,best}$}
\caption{The landscape smoothing global search framework.}
\label{alg:LSILS}
\end{algorithm}

\subsection{Performance Comparison based on 3-Opt Local Search}

We first use the 3-Opt local search as the local search operator and the double bridge perturbation as the perturbation operator to realize the proposed framework. The resultant algorithm is named as \emph{Landscape Smoothing Iterated Local Search} (LSILS). The developed LSILS is compared against ILS, the smoothing algorithm proposed by Gu and Huang~\cite{gu1994efficient} (denoted as GH) and the sequential smoothing algorithm proposed by Coy et al.~\cite{coy2000computational} (denoted as SSA).

Seven instances are chosen from the TSPLIB as the test instances, including rd400, p654, u724, pcb1173, rl1304, vm1748 and u1817. Each algorithm is executed 100 runs on each test instance from different random initial solutions. The stopping criterion of each run is $10^8$ function evaluations. For LSILS, we test five different settings of $\lambda$. In the first three settings constant values of $\lambda$ are applied during the search and in the last two settings, $\lambda=0$ at the beginning and increases during the search, as shown in Table~\ref{tbl:LSILS_lambda}. The parameter settings are the same as in~\cite{coy2000computational}. That is, for GH the smoothing factor $\alpha = 6\to5\to4\to3\to2\to1$ in the first six local search rounds and for SSA $\alpha = 7\to5\to3\to1$ in the first four local search rounds. In both GH and SSA, if $\alpha$ has dropped to $1$ and the function evaluation budget has not been run out, the algorithm will execute ILS on the original TSP landscape (i.e. $\alpha=1$) until the function evaluation budget is run out. The 3-Opt local search and the double bridge perturbation are also used in the implementations of ILS, LSILS, GH and SSA. In addition, the 3-Opt local search implementation is speeded up by the techniques of \emph{near neighbor search} and \emph{don't look bits}~\cite{bentley1992fast}. The other experimental settings are the same to the settings in the previous section.

\begin{table*}
\caption{Settings of $\lambda$ in LSILS Implementations}
\centering
\label{tbl:LSILS_lambda}
\begin{tabular}{c >{\centering}p{22pt} >{\centering}p{22pt}  >{\centering}p{22pt} >{\centering}p{22pt}  >{\centering}p{22pt} >{\centering}p{22pt}  >{\centering}p{22pt} >{\centering}p{22pt}  >{\centering}p{22pt} c }
\hline
\begin{minipage}{120pt}Function Evaluation Period ($\times10^7$):\end{minipage} & $0\to1$ & $1\to2$ & $2\to3$ & $3\to4$ & $4\to5$ & $5\to6$ & $6\to7$ & $7\to8$ & $8\to9$ & $9\to10$ \\
\hline
Setting1 & 0.02 & 0.02 & 0.02 & 0.02 & 0.02 & 0.02 & 0.02 & 0.02 & 0.02 & 0.02 \\
\hline
Setting2 & 0.04 & 0.04 & 0.04 & 0.04 & 0.04 & 0.04 & 0.04 & 0.04 & 0.04 & 0.04 \\
\hline
Setting3 & 0.06 & 0.06 & 0.06 & 0.06 & 0.06 & 0.06 & 0.06 & 0.06 & 0.06 & 0.06 \\
\hline
Setting4 & 0.00 & 0.00 & 0.01 & 0.01 & 0.02 & 0.02 & 0.03 & 0.03 & 0.04 & 0.04 \\
\hline
Setting5 & 0.00 & 0.01 & 0.02 & 0.03 & 0.04 & 0.05 & 0.06 & 0.07 & 0.08 & 0.09 \\
\hline
\end{tabular}
\end{table*}

Fig.~\ref{fig:excess} shows the mean excess values to the best solutions found so far by different smoothing algorithms against time. To properly compare among TSP solvers, a reasonable method is to compare the areas under curves (AUC)~\cite{weise2014benchmarking,weise2018automatically}. The AUC values of the excess curves are listed in Table~\ref{tbl:AUC}. From Fig.~\ref{fig:excess} and Table~\ref{tbl:AUC} we can see that, in six out of seven test instances, the minimum AUC is obtained by LSILS, which means that it performs better than ILS, GH and SSA on these six instances. From Fig.~\ref{fig:excess} and Table~\ref{tbl:AUC} we can also see that, in four instances, LSILS with increasing $\lambda$ values (Setting5 in Table~\ref{tbl:LSILS_lambda}) achieves the best performance. On p654 (Fig.~\ref{fig:excess_p654}), SSA achieves the best performance.

\begin{figure*}
  \raggedleft
  \subfigure[rd400]{
    \label{fig:excess_rd400} 
    \includegraphics[height=0.210\linewidth]{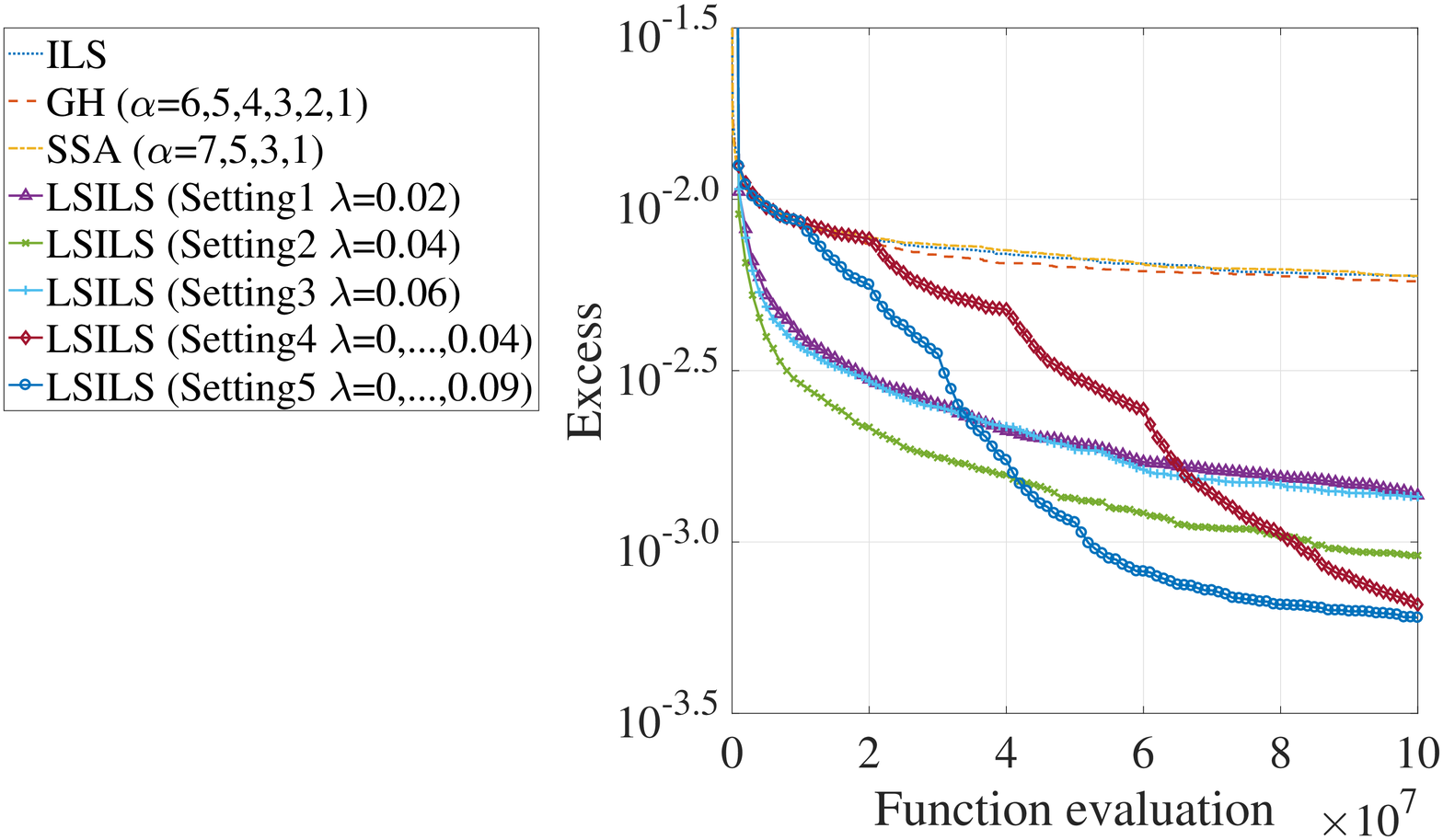}}
  \subfigure[p654]{
    \label{fig:excess_p654} 
    \includegraphics[height=0.210\linewidth]{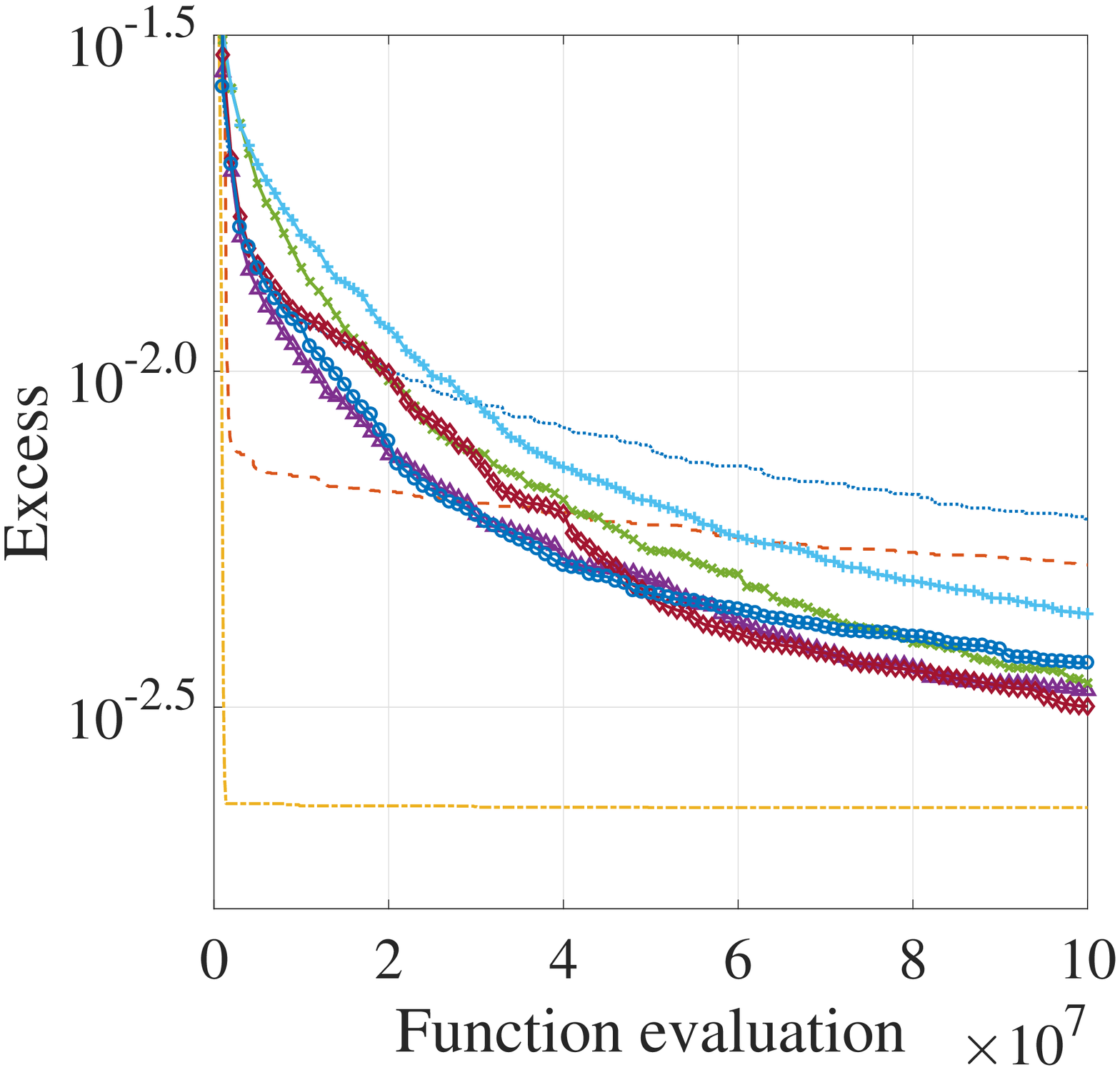}}
  \subfigure[u724]{
    \label{fig:excess_u724} 
    \includegraphics[height=0.210\linewidth]{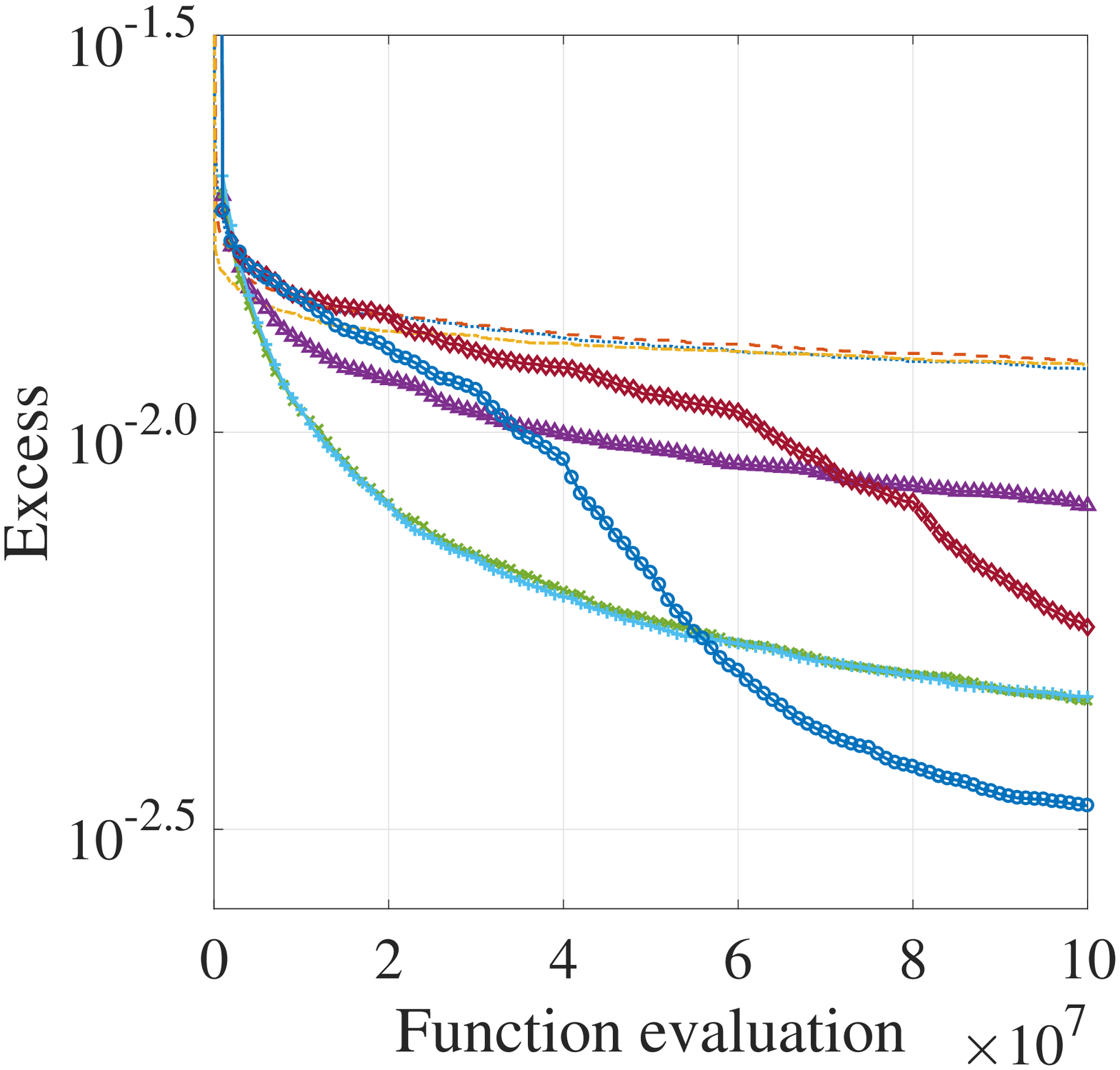}}\\
  \subfigure[pcb1173]{
    \label{fig:excess_pcb1173} 
    \includegraphics[height=0.210\linewidth]{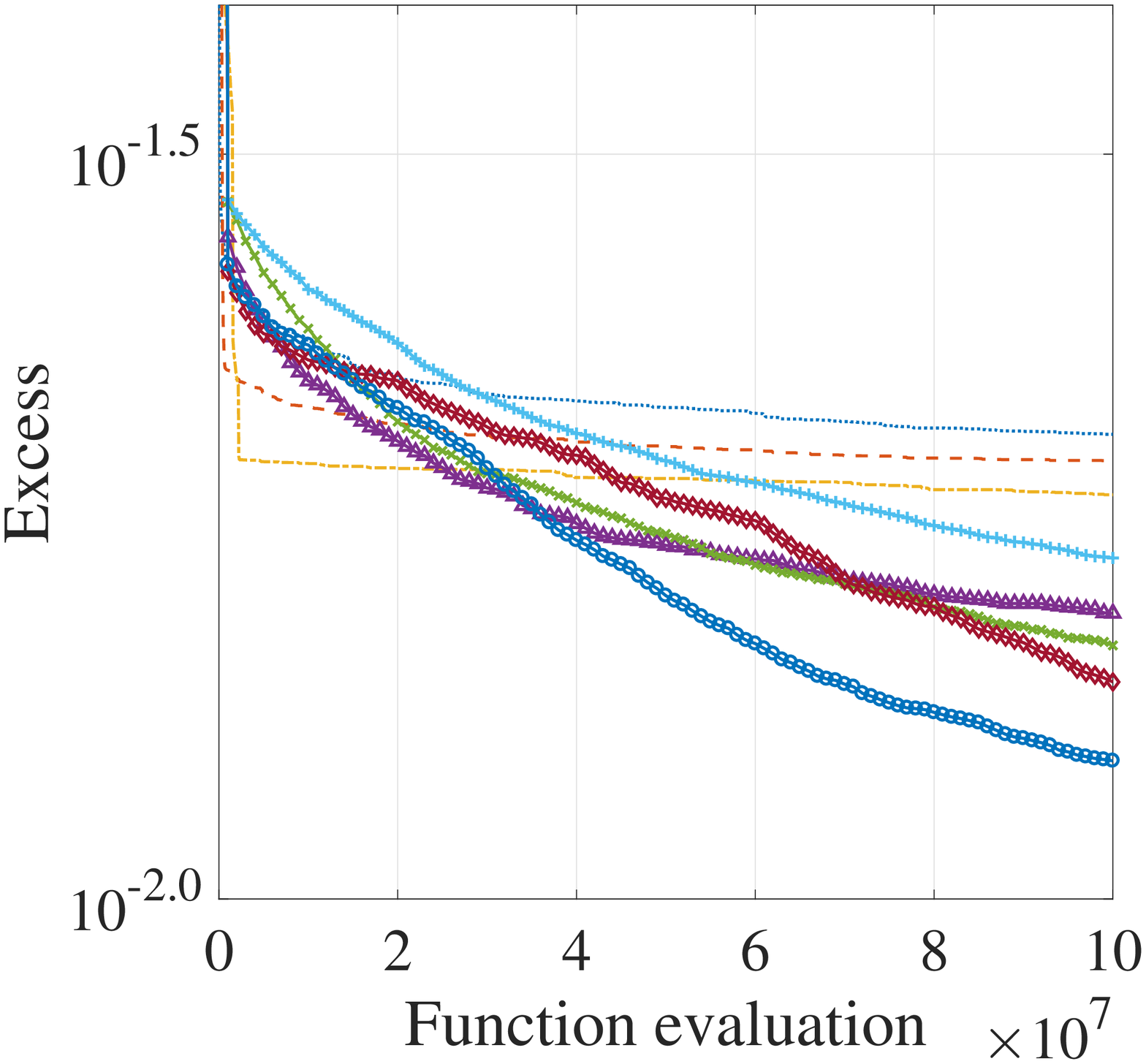}}
  \subfigure[rl1304]{
    \label{fig:excess_rl1304} 
    \includegraphics[height=0.210\linewidth]{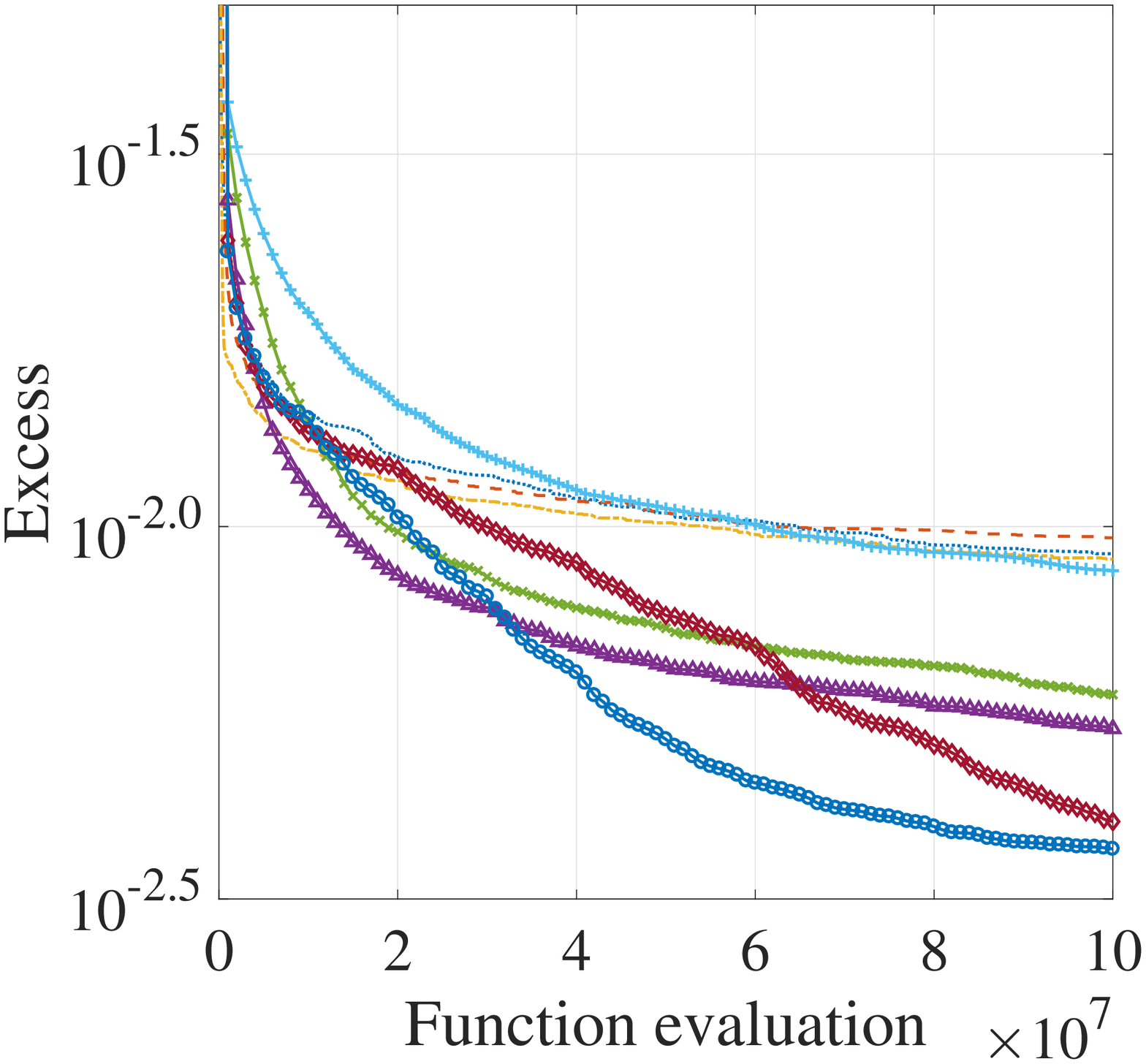}}
  \subfigure[vm1748]{
    \label{fig:excess_vm1748} 
    \includegraphics[height=0.210\linewidth]{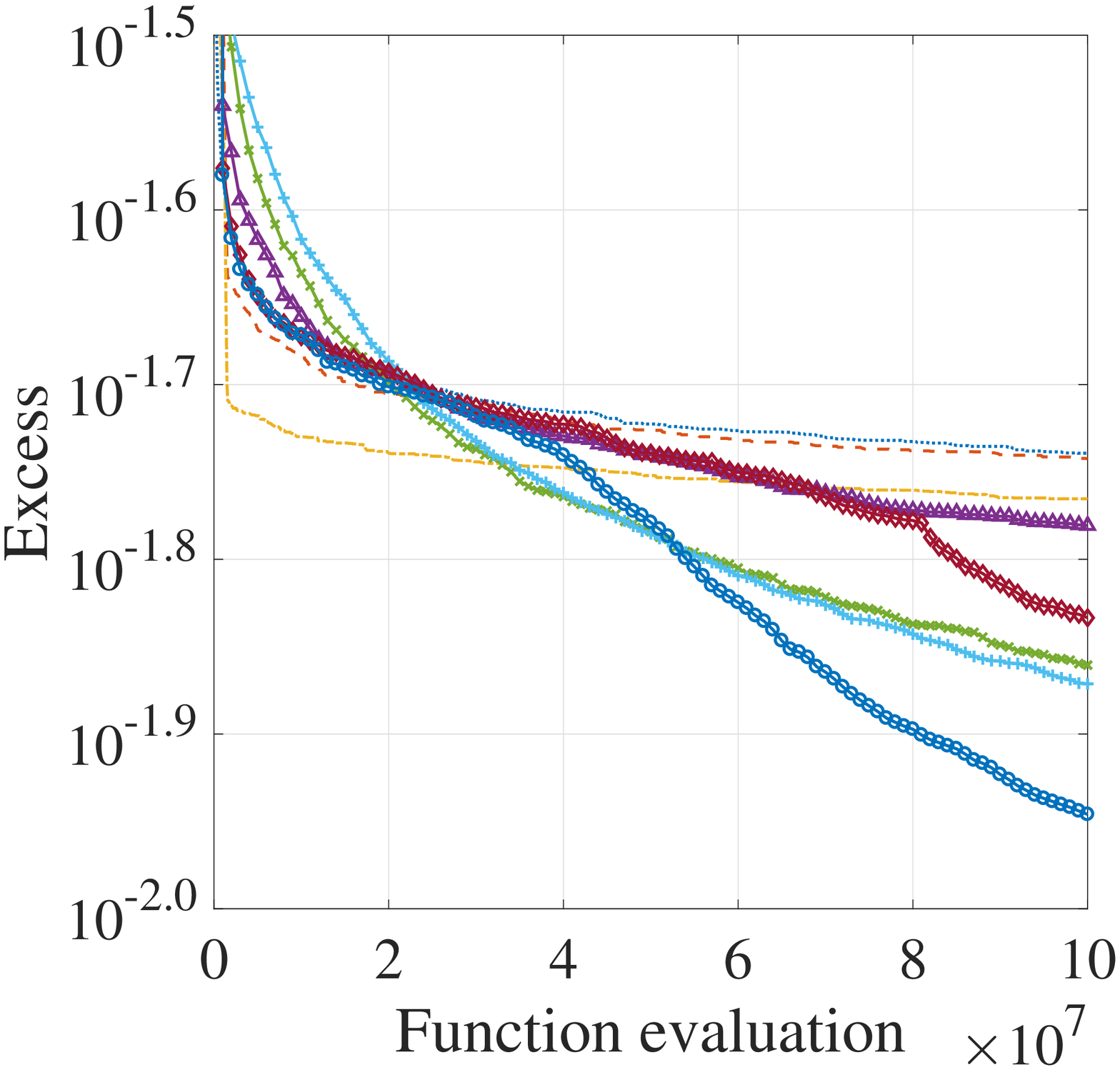}}
  \subfigure[u1817]{
    \label{fig:excess_u1817} 
    \includegraphics[height=0.210\linewidth]{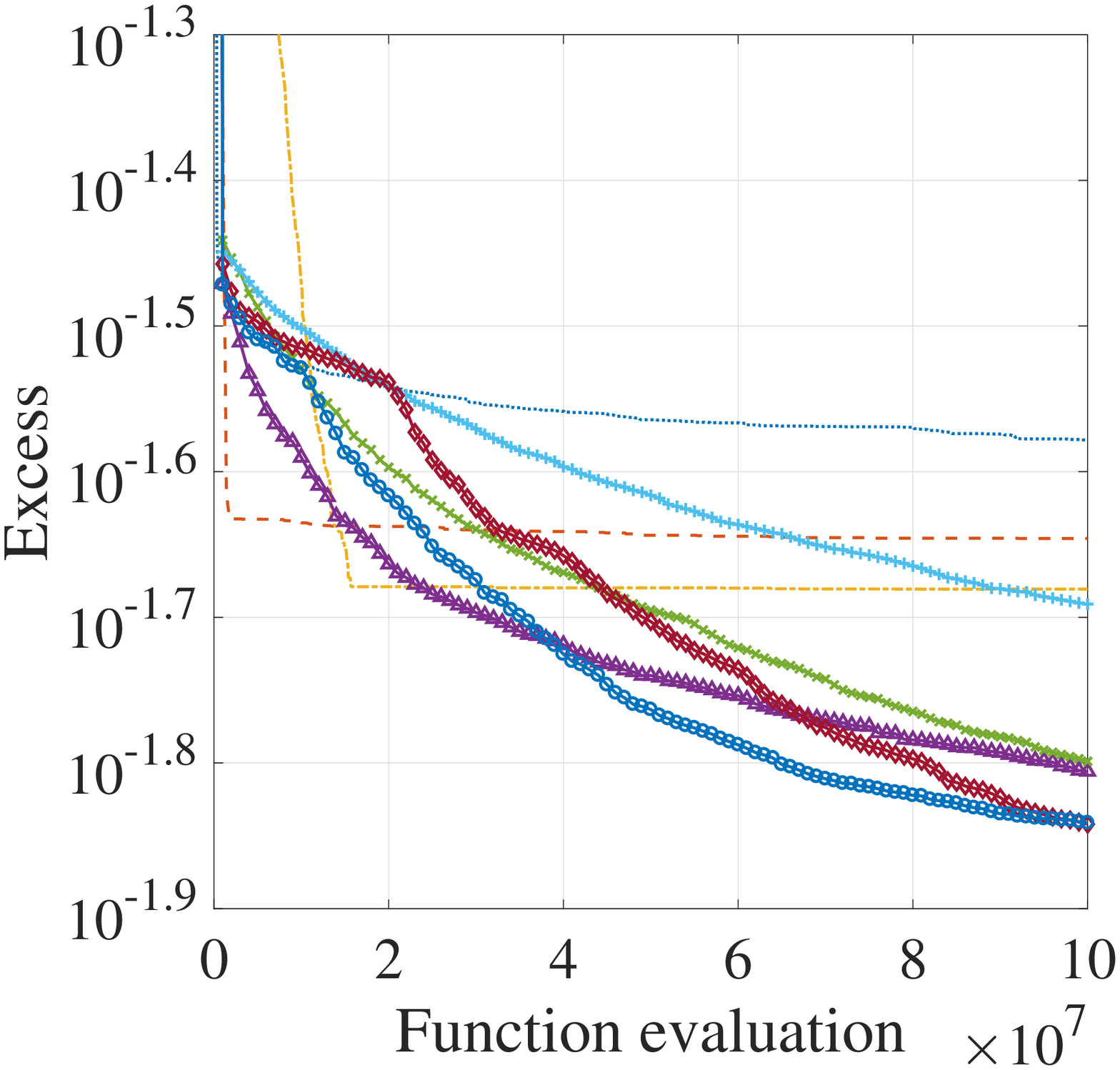}}
  \caption{Comparison result between LSILS, ILS, GH and SSA based on 3-Opt local search and double bridge perturbation.}\label{fig:excess}
\end{figure*}

However, we notice that on p654 the performance of LSILS is close to that of SSA at termination. We conjecture LSILS can perform better if there are more computational budget. To justify, we run LSILS on p654 with an increased function evaluation budget ($5\times 10^8$). The results are shown in Fig.~\ref{fig:p654_new}, from which we can see that at termination LSILS outperforms SSA in all settings. This confirms our conjecture.

\begin{figure}
  \centering
  \includegraphics[width=0.8\linewidth]{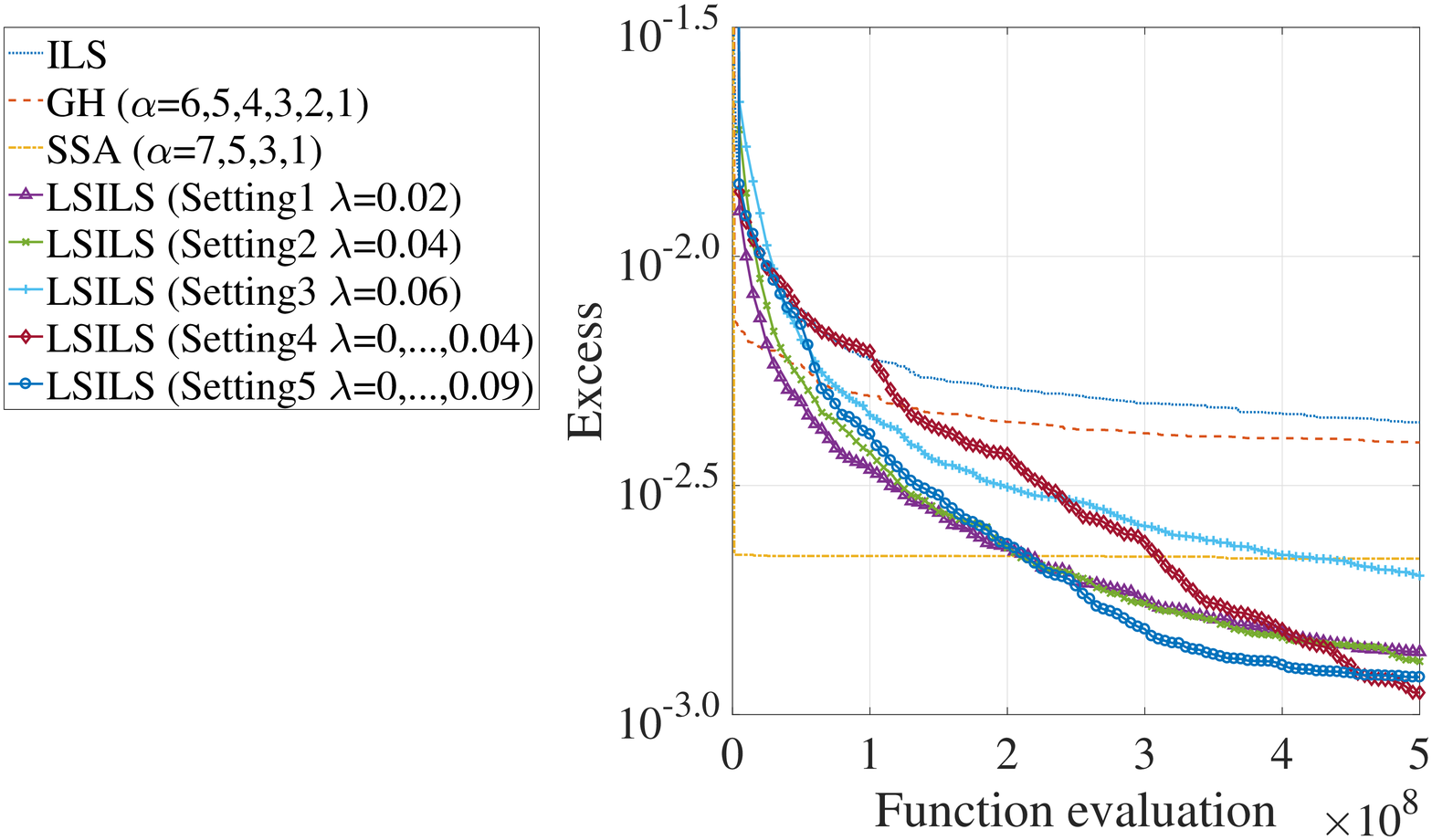}\\
  \caption{Comparison results on p654 with more function evaluations.}\label{fig:p654_new} 
\end{figure}

On average, we see that LSILS performs significantly better than ILS, GH and SSA in terms of excess on most of the test instances. This clearly indicates that the proposed HC transformation can truly improve the performance of ILS.

However, we can also see that, the performance of LSILS is relatively poor at the beginning of the search. In fact, in most cases, the descent rate of GH and SSA is much higher than that of LSILS at the beginning of the search. This is because that the smoothing effect of the HC transformation depends highly on the quality of the local optimum used to construct the convex-hull TSP. As stated in Section~\ref{sec:diff_lo}, with a low-quality local optimum, the smoothing effect of the HC transformation is not as good as a high-quality local optimum. As search goes by, better solutions are found, the performance of LSILS becomes better and exceeds the performance of GH and SSA eventually in most cases.

Table~\ref{tbl:cpu_time} lists the mean CPU time used by each algorithm in the experiments. Table~\ref{tbl:cpu_time} shows that, LSILS with Setting3 (i.e. $\lambda$ is constant and equals to $0.06$) has the shortest CPU time in six out of seven test instance and on most instances, the runtime of LSILS with increasing $\lambda$ value (i.e. Setting4 and Setting5) is lower than that of GH and SSA. This means that LSILS with increasing $\lambda$ value can find better solutions and requires less time than those of GH and SSA in most cases.

\begin{table*}
\caption{Areas Under the Excess Curves} 
\centering
\label{tbl:AUC}
\begin{tabular}{l c c c c c c c c }
\hline
\multirow{2}{*}{Instance} & \multirow{2}{*}{ILS} & \multirow{2}{*}{GH} & \multirow{2}{*}{SSA} & \multicolumn{5}{c}{LSILS} \\
\cline{5-9}
 & & & & Setting1 & Setting2 & Setting3 & Setting4 & Setting5 \\
\hline
rd400 & 6.930e+05 & 6.692e+05 & 7.005e+05 & 2.368e+05 & \textbf{1.684e+05} & 2.266e+05 & 3.900e+05 & 2.706e+05\\
p654 & 8.422e+05 & 6.072e+05 & \textbf{2.227e+05} & 5.856e+05 & 7.223e+05 & 8.285e+05 & 6.388e+05 & 6.084e+05\\
u724 & 1.307e+06 & 1.321e+06 & 1.283e+06 & 1.004e+06 & 6.724e+05 & \textbf{6.701e+05} & 1.085e+06 & 7.906e+05\\
pcb1173 & 2.149e+06 & 2.012e+06 & 1.910e+06 & 1.801e+06 & 1.832e+06 & 2.038e+06 & 1.862e+06 & \textbf{1.690e+06}\\
rl1304 & 1.109e+06 & 1.093e+06 & 1.039e+06 & 7.521e+05 & 8.690e+05 & 1.229e+06 & 8.329e+05 & \textbf{6.904e+05}\\
vm1748 & 1.922e+06 & 1.894e+06 & 1.767e+06 & 1.856e+06 & 1.734e+06 & 1.751e+06 & 1.820e+06 & \textbf{1.634e+06}\\
u1817 & 2.753e+06 & 2.263e+06 & 2.801e+06 & 1.922e+06 & 2.121e+06 & 2.487e+06 & 2.106e+06 & \textbf{1.920e+06}\\
\hline
\end{tabular}
\end{table*}

\begin{table*}
\caption{Mean CPU time} 
\centering
\label{tbl:cpu_time}
\begin{tabular}{l c c c c c c c c }
\hline
\multirow{2}{*}{Instance} & \multirow{2}{*}{ILS} & \multirow{2}{*}{GH} & \multirow{2}{*}{SSA} & \multicolumn{5}{c}{LSILS} \\
\cline{5-9}
 & & & & Setting1 & Setting2 & Setting3 & Setting4 & Setting5 \\
\hline
rd400 & 48.04s & 81.37s & 68.88s & 42.75s & 34.58s & \textbf{28.55s} & 79.98s & 65.87s\\
p654 & \textbf{7.96s} & 13.57s & 10.85s & 8.11s & 8.14s & 8.36s & 16.38s & 16.32s\\
u724 & 76.40s & 138.50s & 119.41s & 60.86s & 41.52s & \textbf{28.71s} & 115.37s & 85.00s\\
pcb1173 & 119.94s & 229.56s & 201.21s & 56.18s & 26.49s & \textbf{18.59s} & 117.19s & 80.67s\\
rl1304 & 73.89s & 165.36s & 127.96s & 61.47s & 43.28s & \textbf{31.10s} & 109.97s & 85.77s\\
vm1748 & 116.19s & 229.42s & 196.85s & 78.50s & 49.04s & \textbf{31.64s} & 151.44s & 103.60s\\
u1817 & 123.23s & 230.96s & 212.64s & 53.15s & 24.99s & \textbf{17.79s} & 122.85s & 79.61s\\
\hline
\end{tabular}
\end{table*}

\subsection{Performance Comparison based on the LK Local Search}

In previous sections, we only test the performance of the proposed framework based on the 3-Opt local search and the test instances is relatively small with city number $n<2000$. We conduct another experiment to test the performance of the proposed framework based on the most powerful TSP local search method, the Lin-Kernighan (LK) local search~\cite{lin1973effective}, on ten middle-size and large-size TSP instances. The resultant algorithm is named as LSILS-LK. 

In our experiment, the implementation of the LK local search is from the Concorde software package~\footnote{\url{http://www.math.uwaterloo.ca/tsp/concorde/}}. In the LK local search, the edge exchange is restricted in a sub-graph of the original TSP graph. In the sub-graph, each vertex (city) only connects with its 20 nearest vertexes (cities). Based on the LK local search and the double bridge perturbation, we compare LSILS-LK against ILS with LK local search (ILS-LK), GH with LK local search (GH-LK) and SSA with LK local search (SSA-LK). In ILS-LK, GH-LK and SSA-LK, the LK local search is used as the local search and double bridge perturbation is used as the perturbation method.

Further, to test the perturbation strategy to the performance of LSILS-LK, we include another two algorithm instantiations, i.e. LSILS-LK with three double bridge perturbation (LSILS-LK-3DBP) and ILS-LK with three double bridge perturbation (ILS-LK-3DBP). The difference between LSILS-LK-3DBP and LSILS-LK is that in LSILS-LK-3DBP the double bridge perturbation is replaced by three successive double bridge perturbations. Here we do not apply the 3DBP strategy to GH-LK and SSA-LK, because except the first few iterations, these two algorithms are exactly the same to ILS-LK.

Six middle-size and large-size instances from the TSPLIB: \{pcb3038, fnl4461, pla7397, rl11849, usa13509, d18512\} and four randomly generate  instances: \{rand5000, rand10000, rand15000, rand20000\} are used as benchmarks. Among the selected TSPLIB instances, pcb3038, pla7397 and rl11849 come from printed circuit board, while fnl4461, usa13509 and d18512 come from real world maps. Since it is very hard to count the number of function evaluations in the LK local search, we use the CPU runtime as the stopping criterion for all the algorithms. The maximal runtimes in seconds allowed for optimizing the test instances are \{pcb3038: 600s, fnl4461: 900s, pla7397: 1500s, rl11849: 2400s, usa13509: 2700s, d18512: 3700s, rand5000:1000s, rand10000:2000s, rand15000:3000s, rand20000:4000s\}, respectively. For LSILS-LK and LSILS-LK-3DBP, we use Setting4 and Setting5 since with these settings the proposed framework has shown its superiority in previous study.

The time point to increase $\lambda$ is based on the CPU time. For example, for LSILS-LK running on pcb3038 with Setting5, $\lambda$ will change from 0 to 0.01 at the 60 second, from 0.01 to 0.02 at the 120 second and so on. Further, to save computational effort, in the implementations of LSILS-LK and LSILS-LK-3DBP the transformed TSP $g$ (which guides the LK local search) is only updated at the time points when $\lambda$ increases (i.e., each time $\lambda$ increases, $g$ will be updated based on the new $\lambda$ value and the current best solution $x_{f_o,best}$). Each algorithm is executed 50 runs on each test instance from different random initial solutions. The other experimental settings are the same to the previous experiments.

Fig.~\ref{fig:excess_LK} shows the mean excess values of the best-so-far solutions found by the compared algorithms against time. Since the globally optimal function values of the randomly generated instances are unknown, we use the best found function value $\times 0.99$ as the globally optimal function value to calculate the excess on the random instances. From Fig.~\ref{fig:excess_LK} we can see that, on all the test instances, the best performance is achieved by LSILS-LK-3DBP with Setting5.

Fig.~\ref{fig:excess_LK} also shows that the curves of LSILS-LK and LSILS-LK-3DBP are regularly jagged. We observed that a jag happens when the $\lambda$ increases. This phenomenon confirms again that the HC transformation with better local optima can result in better smoothing effect and hence a higher $\lambda$ value is preferred. We can also observe that, although three double bridge perturbation reduces the performance of ILS-LK, in LSILS-LK, when the perturbation strength becomes larger, the algorithm performance significantly improves. It means that LSILS-LK can work better with three double bridge perturbation strategy. A possible reason is that after smoothing, the number of local optima in the TSP landscape reduces and the size of the attractive basin increases. Hence a larger jump perturbation strategy can be more helpful for LSILS-LK to escape from current local optimum.

We further conduct some experiments for LSILS-LK-3DBP on some medium- and large size TSPs with increased runtime budget. Please see supplementary material for details. The results show that LSILS-LK-3DBP can further improve its performance in case extra budget is available.

In summary, we may conclude that the proposed HC transformation based on local optima can indeed improve the global search ability of existing TSP algorithms. The good performance of LSILS-LK and LSILS-LK-3DBP is due to the HC transformation, which can reduce the number of local optima in the transformed landscape meanwhile preserve the useful information in the optimum of the original TSP.

\begin{figure*}
  \subfigure[pcb3038]{
    \label{fig:excess_pcb3038} 
    \includegraphics[height=0.24\linewidth]{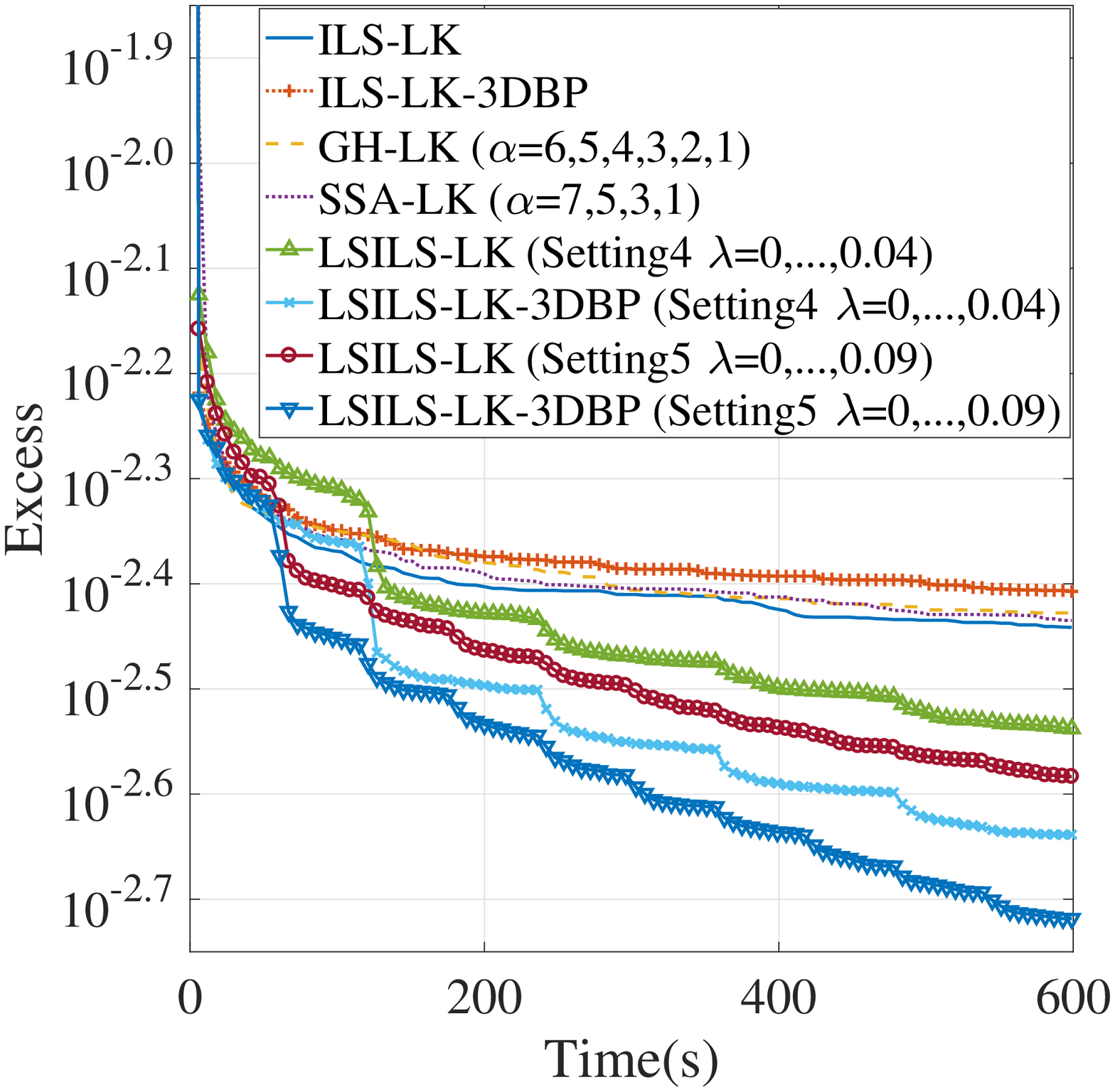}}
    \hspace{-0.12in}
  \subfigure[fnl4461]{
    \label{fig:excess_fnl4461} 
    \includegraphics[height=0.24\linewidth]{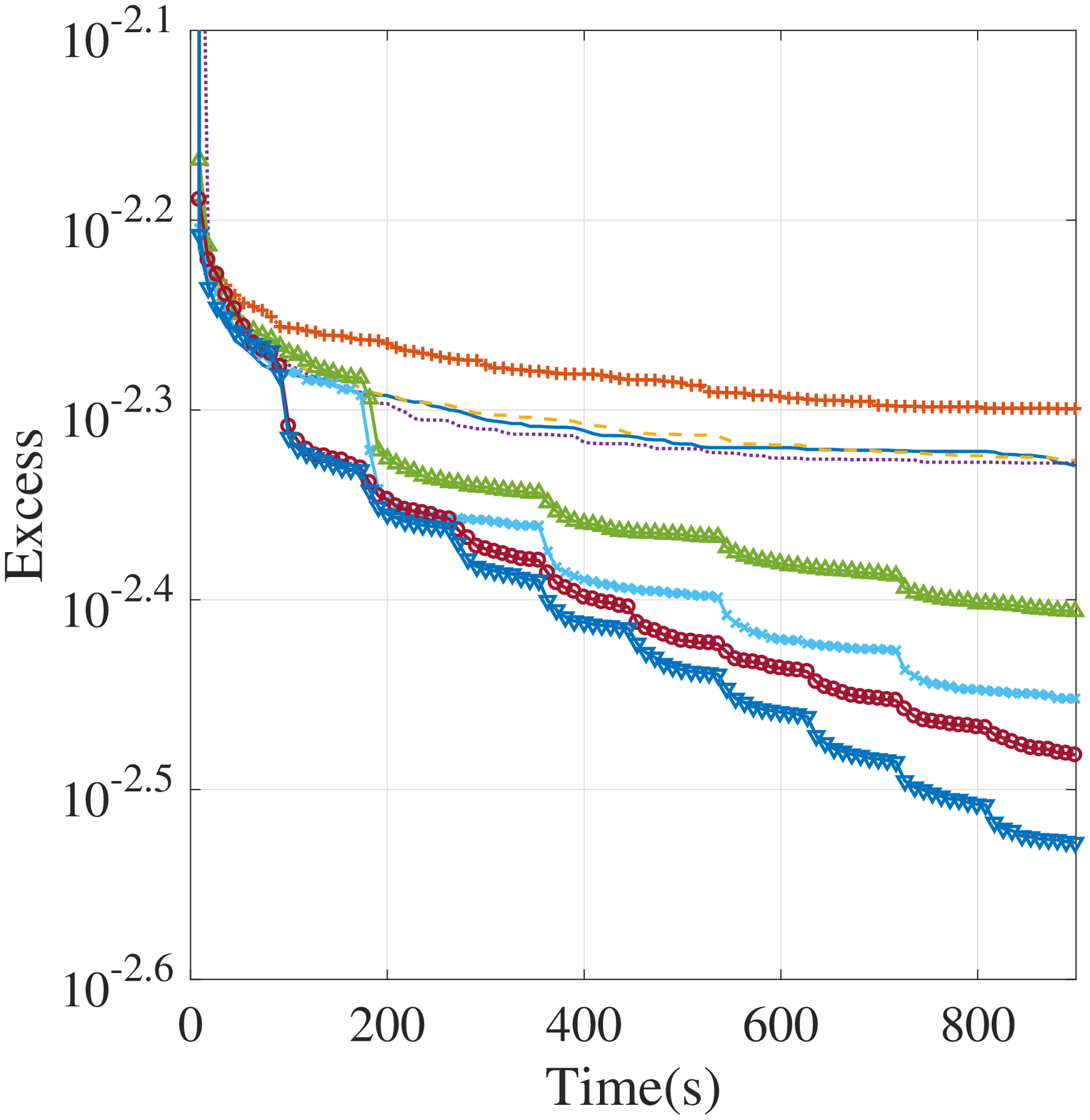}}
    \hspace{-0.12in}
  \subfigure[pla7397]{
    \label{fig:excess_pla7397} 
    \includegraphics[height=0.24\linewidth]{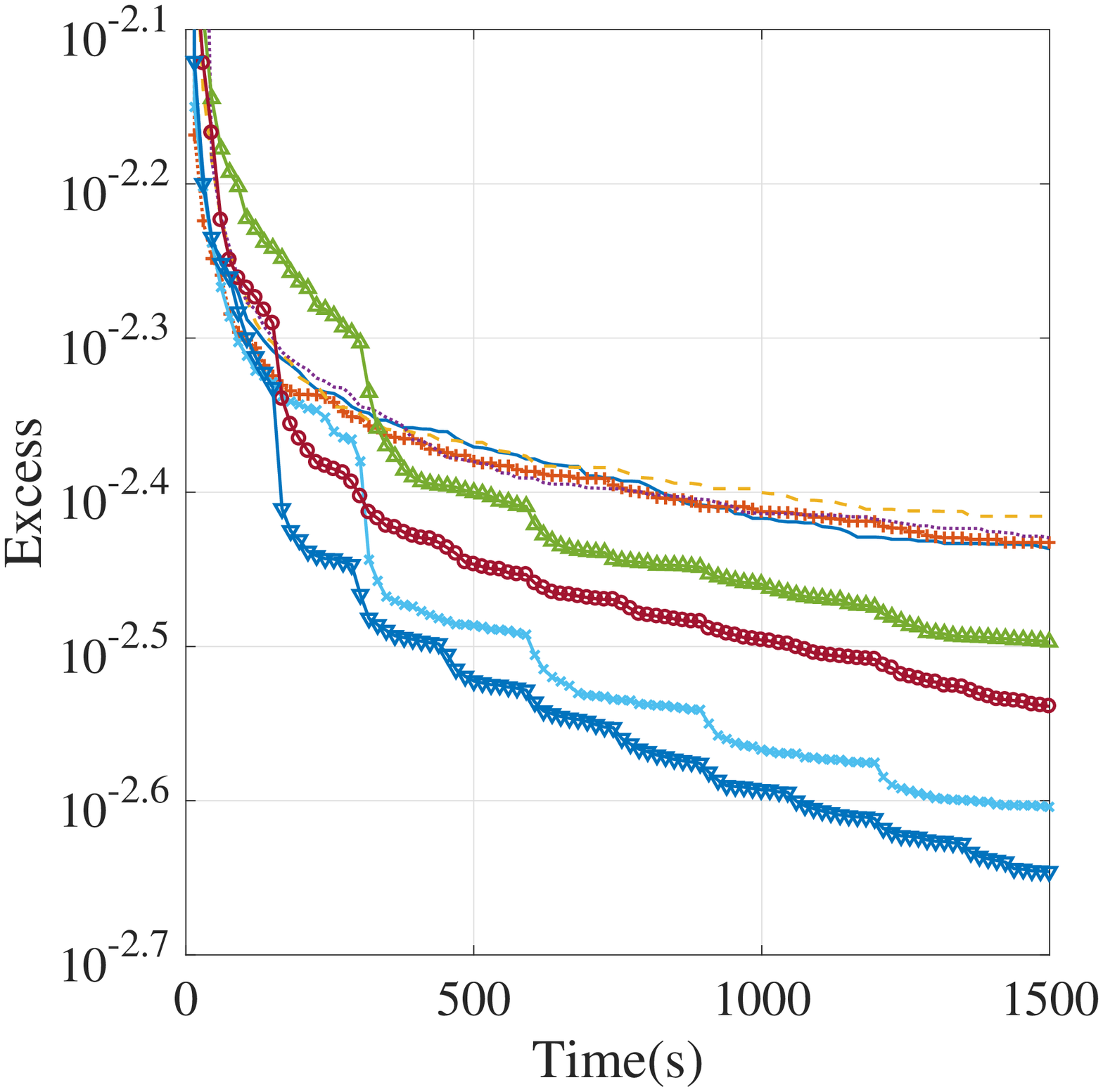}}
    \hspace{-0.12in}
  \subfigure[rl11849]{
    \label{fig:excess_rl11849} 
    \includegraphics[height=0.24\linewidth]{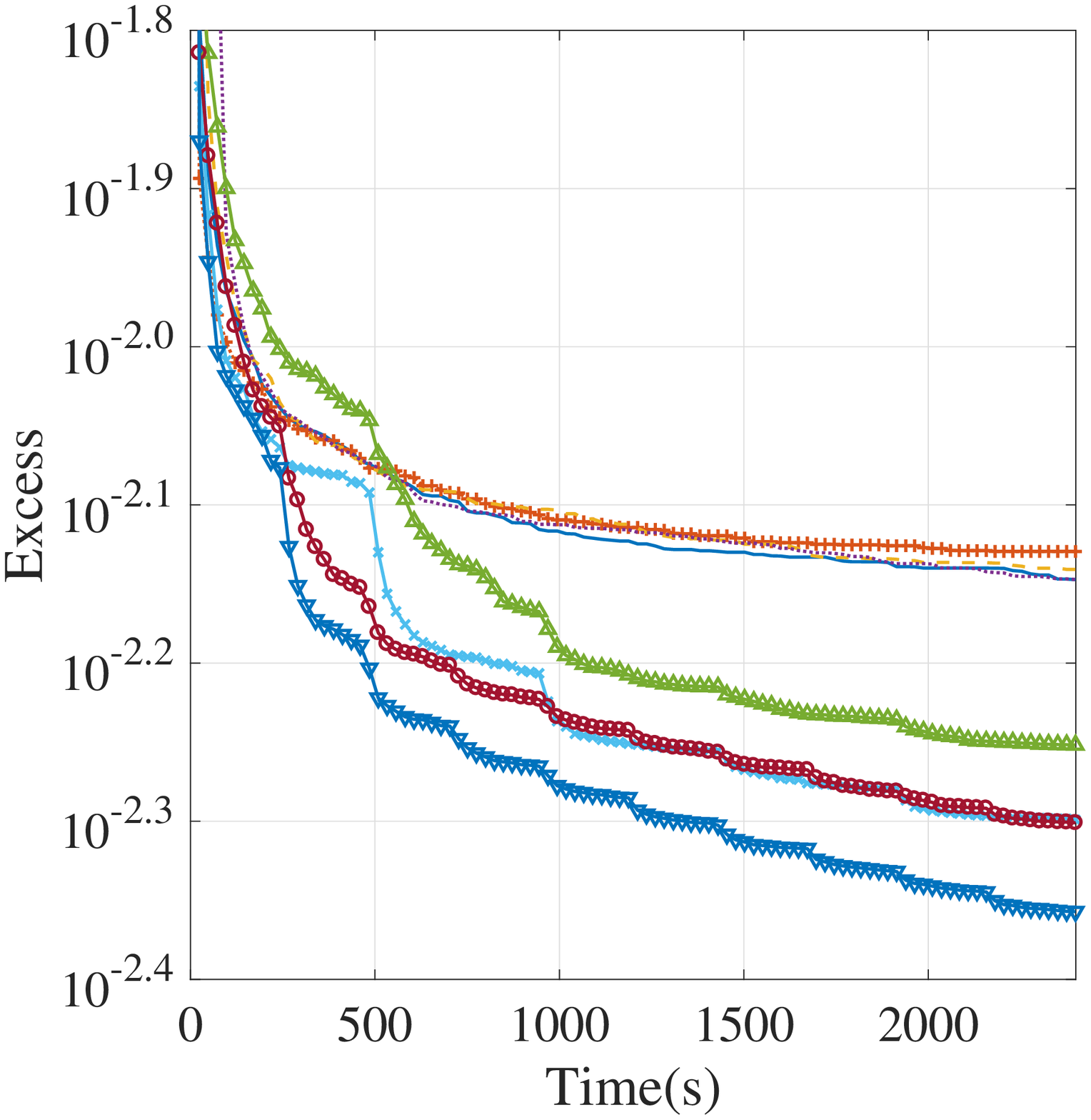}}\\
    \hspace{-0.12in}
  \subfigure[usa13509]{
    \label{fig:excess_usa13509} 
    \includegraphics[height=0.24\linewidth]{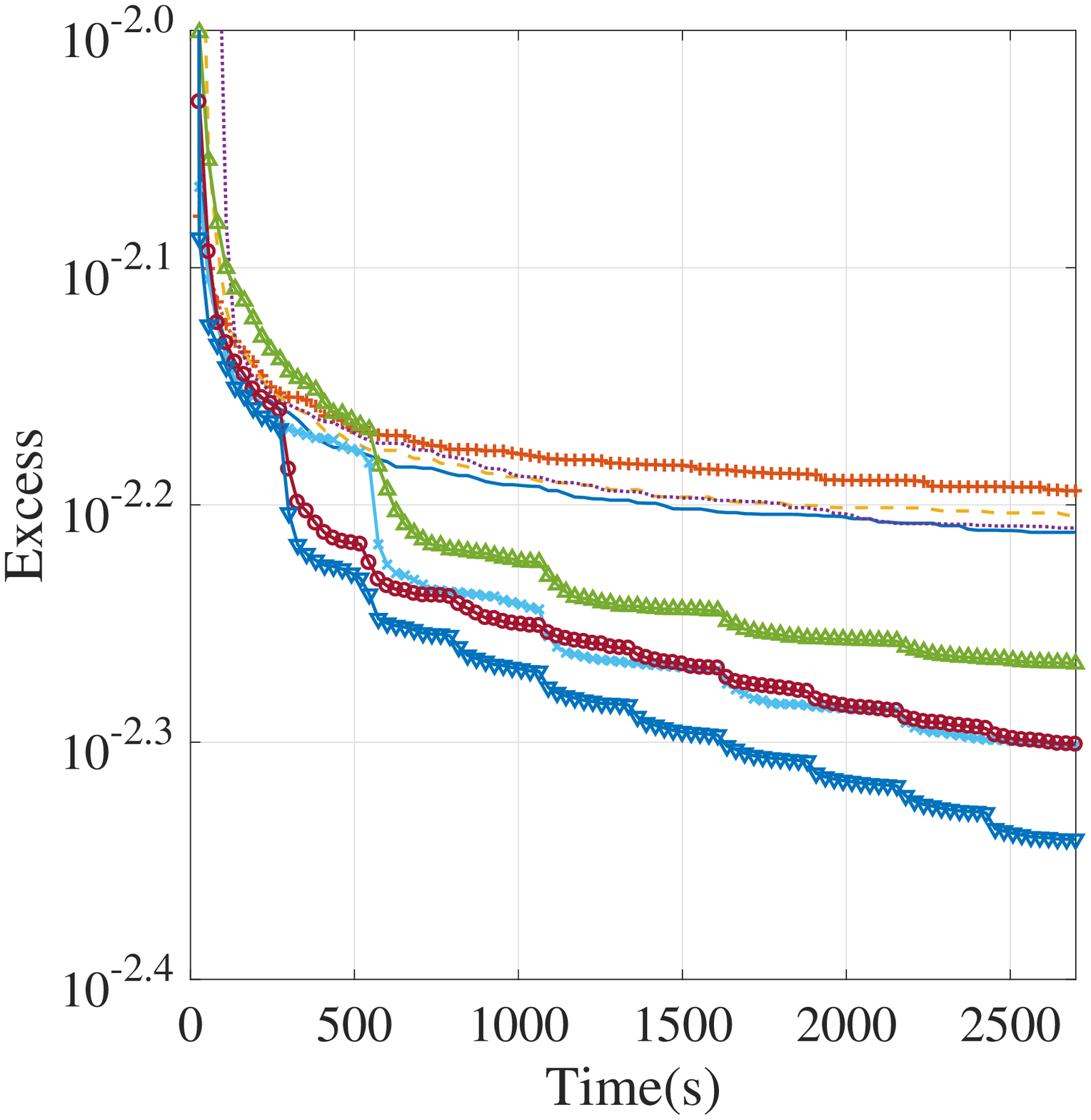}}
  \subfigure[d18512]{
    \label{fig:excess_d18512} 
    \includegraphics[height=0.24\linewidth]{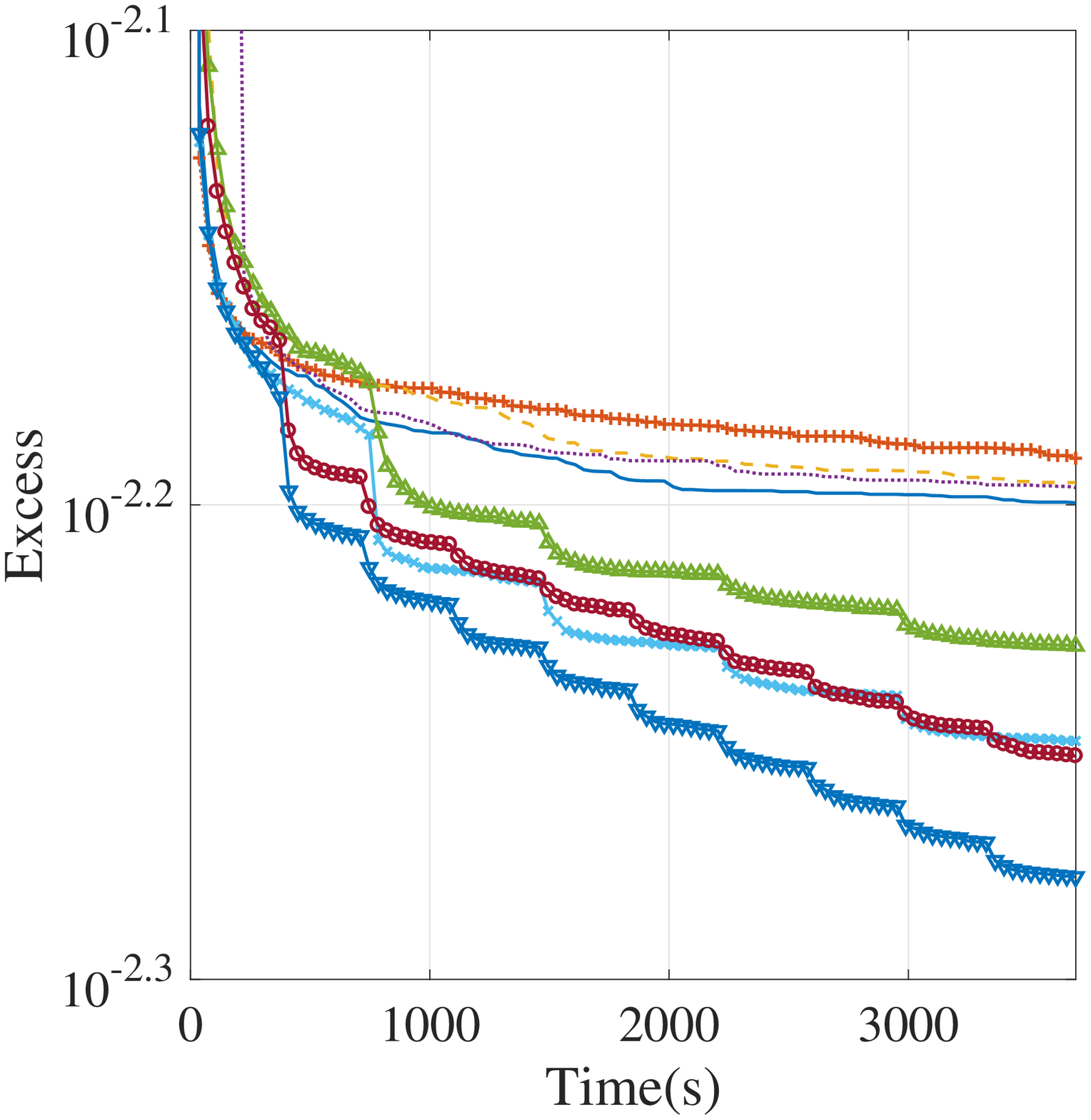}}
    \hspace{-0.12in}
  \subfigure[rand5000]{
    \label{fig:excess_rand5000} 
    \includegraphics[height=0.24\linewidth]{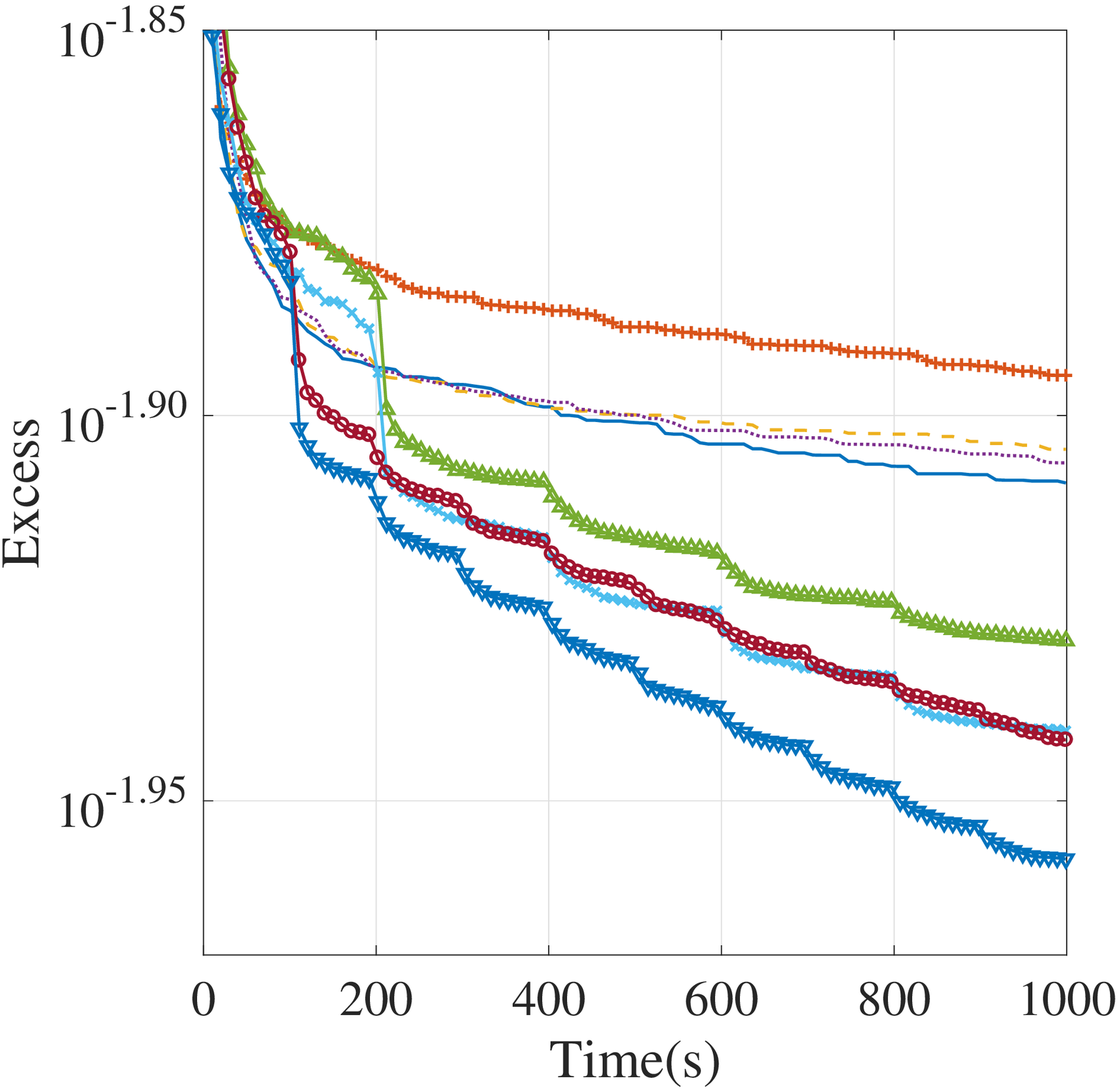}}
    \hspace{-0.12in}
  \subfigure[rand10000]{
    \label{fig:excess_rand10000} 
    \includegraphics[height=0.24\linewidth]{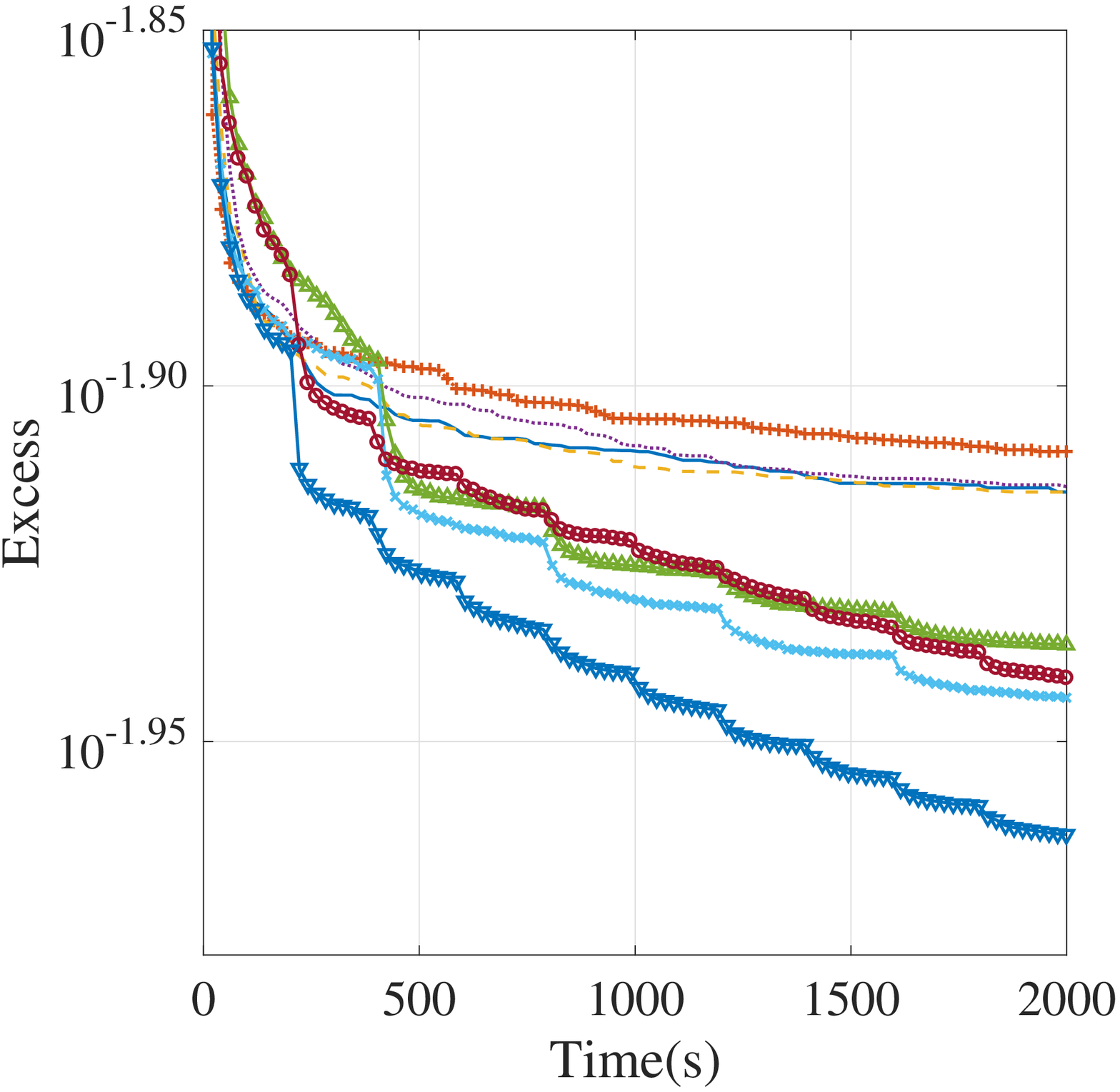}}\\
    \hspace{-0.12in}
  \subfigure[rand15000]{
    \label{fig:excess_rand15000} 
    \includegraphics[height=0.24\linewidth]{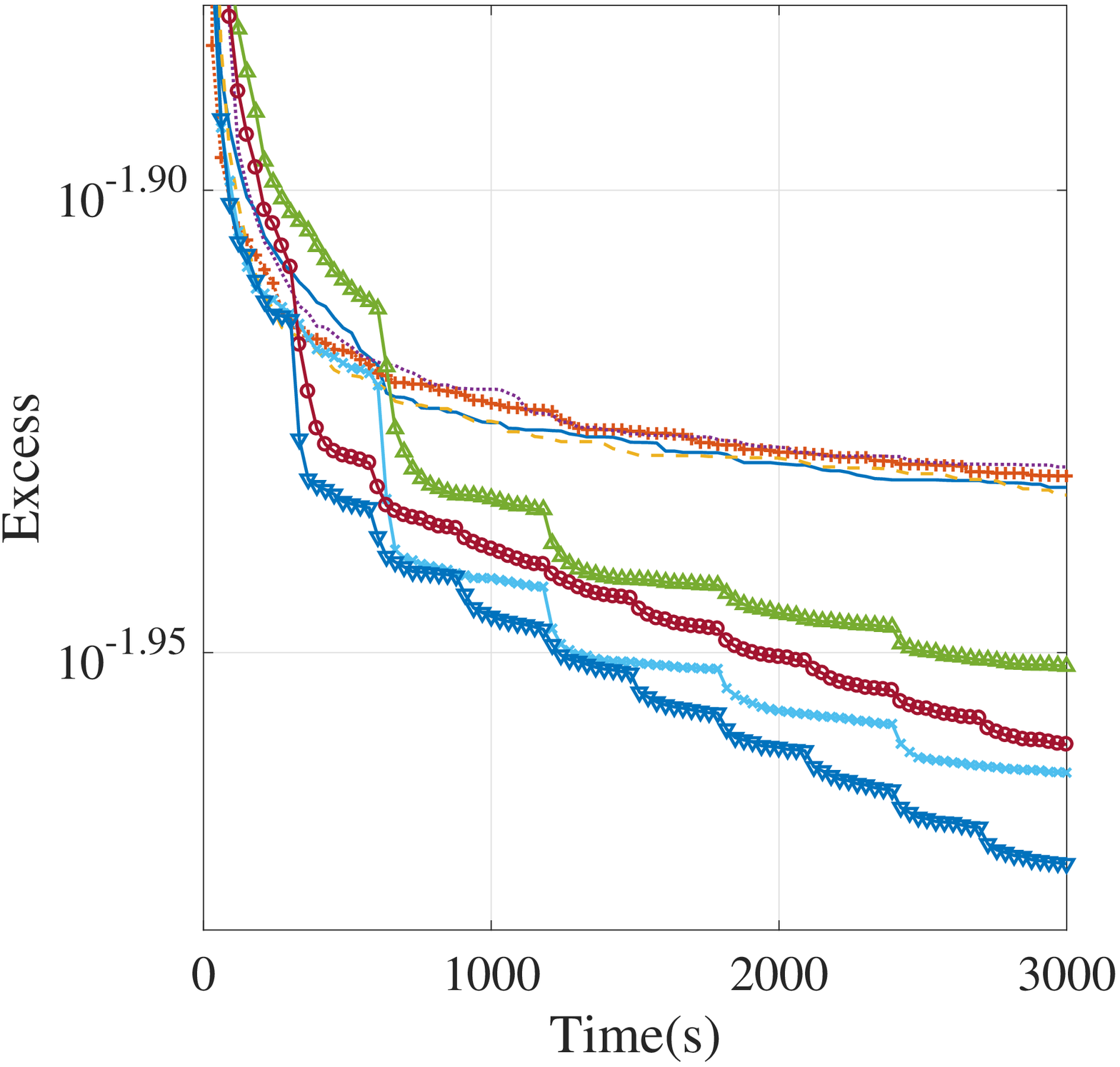}}
    \hspace{-0.12in}
  \subfigure[rand20000]{
    \label{fig:excess_rand20000} 
    \includegraphics[height=0.24\linewidth]{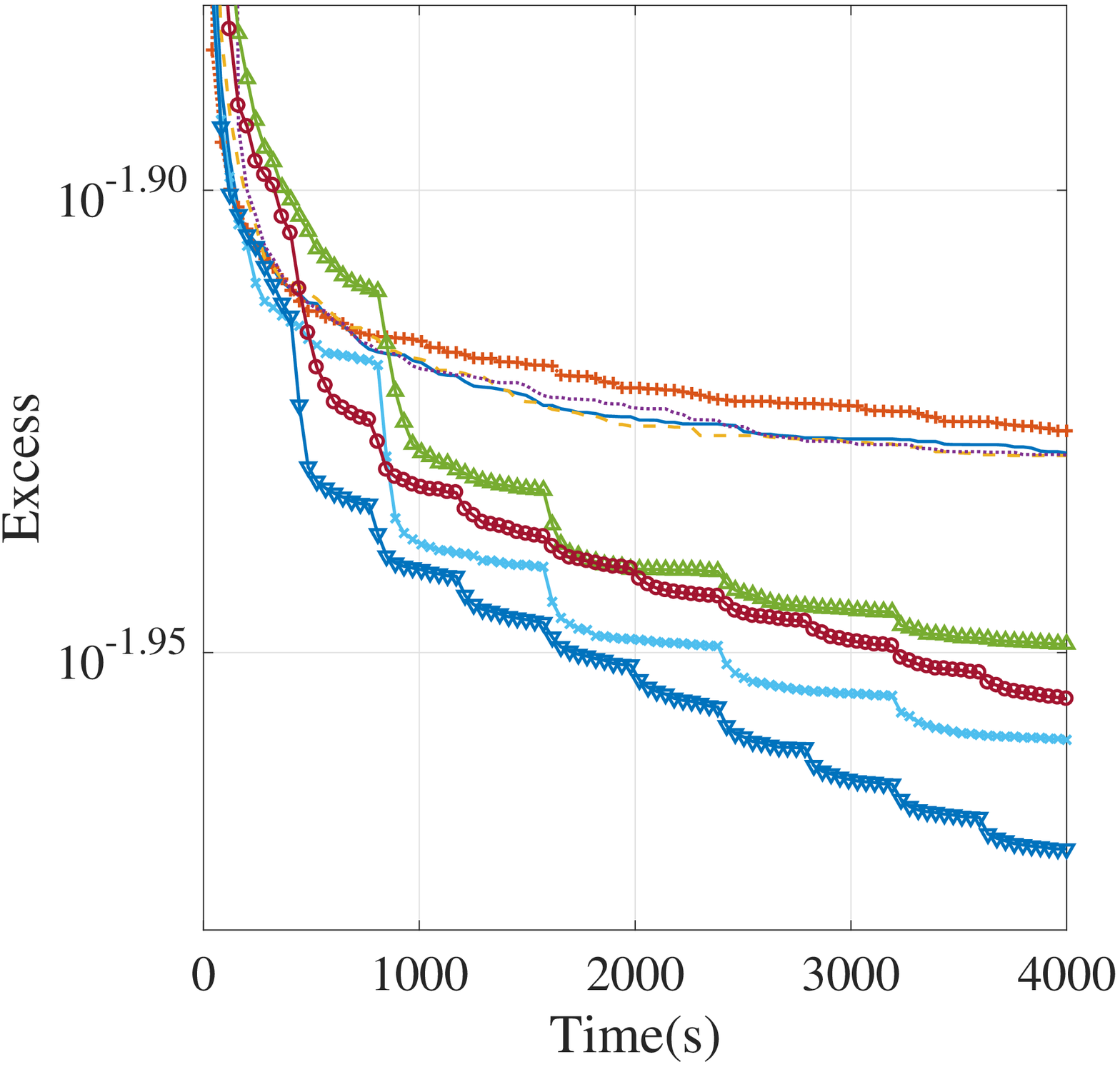}}
  \caption{Comparison results between LSILS-LK, ILS-LK, GH-LK and SSA-LK based on the LK local search on middle-size and large-size TSP instances in terms of the excess metric.}\label{fig:excess_LK}
\end{figure*}

\section{Conclusions} \label{sec:conclu}

In this paper, a new TSP landscape smoothing method called the Homotopic Convex (HC) transformation was proposed. The HC transformation is defined as the convex combination of the original TSP and a convex-hull TSP. For a given TSP, the convex-hull TSP is constructed based on a known optimum. We proved that the convex-hull TSP is unimodal for any $k$-Opt local search ($k \geq 2$). Our empirical study showed that controlled by the convex coefficient, the proposed HC transformation can smooth the landscape in terms that in the transformed TSP, the number of local optima is reduced and the fitness distance correlation of the TSP landscape is increased. Systematic experiments were conducted to verify our statements, in which the HC transformation was used to transform 12 TSP instances with different $\lambda$ values and the characteristics of the landscape of the transformed TSPs were analyzed. In addition, to investigate the relationship between the smoothing effect of the HC transformation and the fitness of the local optimum it based on, we empirically analyzed the effect of the HC transformation on the TSP instance berlin52 based on 89 local optima and the global optimum. The experimental results showed that the HC transformation do have the claimed advantages and using a high-quality local optimum can help the HC transformation achieve a smoothing effect similar to the smoothing effect using the global optimum

Based on the observations, we proposed a landscape smoothing based iterative algorithmic framework, in which the HC transformation is combined with a local search algorithm. An instantiation of the algorithmic framework with 3-Opt local search and double bridge perturbation, named as Landscape Smoothing Iterated Local Search (LSILS), was studied on some widely-used TSP test instances. The experimental results showed that LSILS significantly outperforms ILS and two existing TSP smoothing algorithms on most test instances. In addition, we embedded the LK local search and three double bridge perturbation strategy within the framework, named as LSILS-LK and LSILS-LK-3DBP respectively and tested the resultant algorithm on some middle-size and large-size TSP instances (with up to 20,000 cities). The experimental results showed that LSILS-LK performs the best on all the middle-size and large-size TSP instances against the ILS-LK and other algorithms and the three double bridge perturbation strategy can further improve the performance of LSILS-LK.  In addition, we found that LSILS-LK-3DBP performs significantly better than LSILS-LK, which implies that a larger jump strategy is helpful to further improve LSILS. In conclusion, we may conclude that the  the proposed HC transformation is very promising to improve the global search ability of existing TSP algorithms.

In the future, we will apply the proposed HC transformation to the asymmetric TSP and other combinatorial optimization problems. It is not easy since the HC transformation is based on a unimodal toy problem with known optimum. For other combinatorial optimization problems, finding a way to build this unimodal toy problem is very challenging. We will try some problems that are similar to the TSP, for example, the Vehicle Routing Problem (VRP). Another possible research avenue is to use the HC transformation to improve other kinds of TSP heuristics, for example, ACO and GA.

\appendices



\bibliographystyle{IEEEtran}
%



%
\begin{IEEEbiography}[{\includegraphics[width=1in,height=1.25in,clip,keepaspectratio]{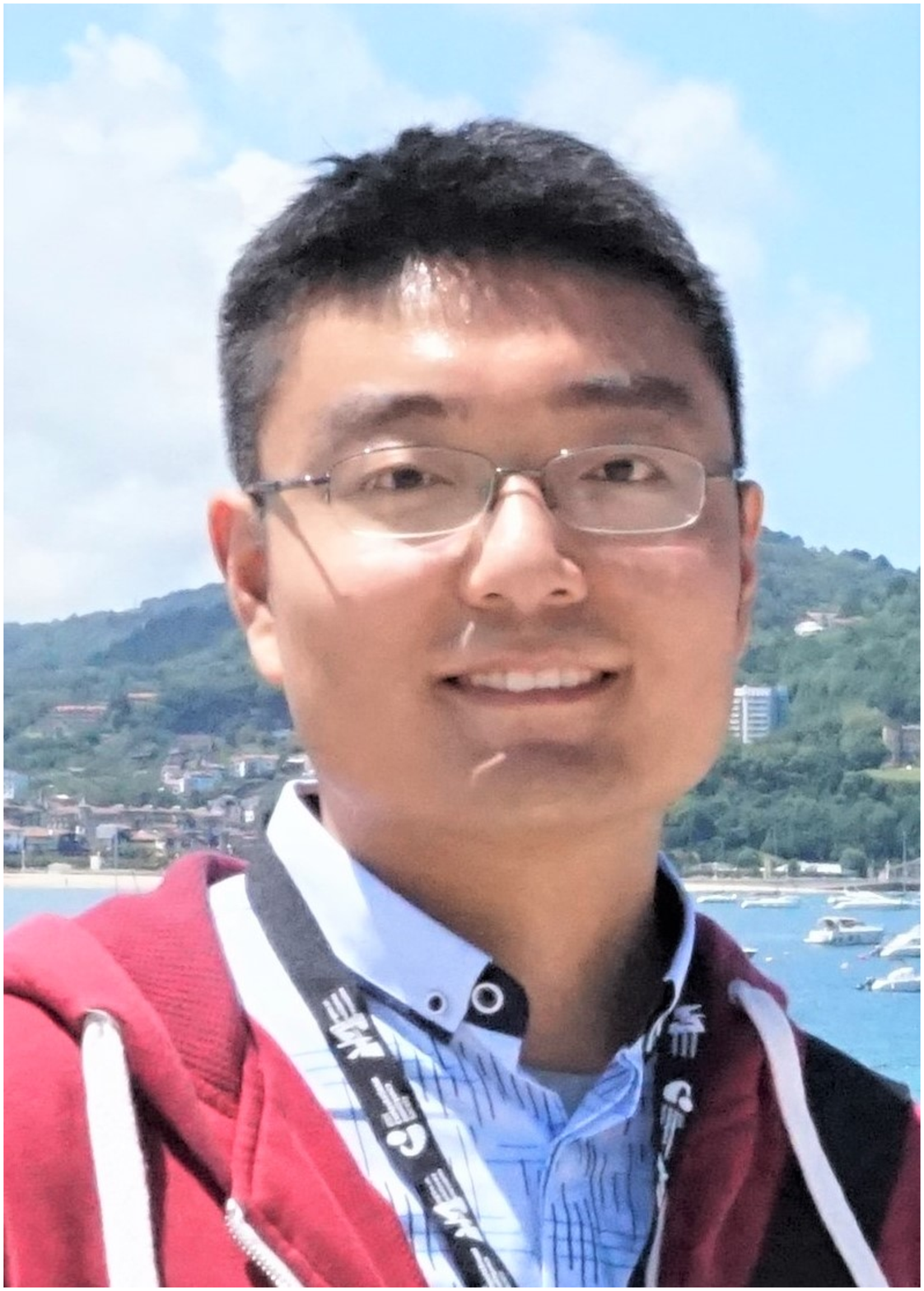}}]{Jialong Shi}
received the B.Sc. and M.Sc. degrees in information engineering from Xi'an Jiaotong University, Xi'an, China, in 2012 and 2014, respectively, and the Ph.D. degree in computer science from City University of Hong Kong, Hong Kong SAR, China, in 2018. He is currently a postdoctoral fellow at the School of Mathematics and Statistics, Xi'an Jiaotong University. His main research interests include
metaheuristics, parallel computation, multiobjective optimization, machine learning and their applications.
\end{IEEEbiography}

\begin{IEEEbiography}[{\includegraphics[width=1in,height =1.25in,clip,keepaspectratio]{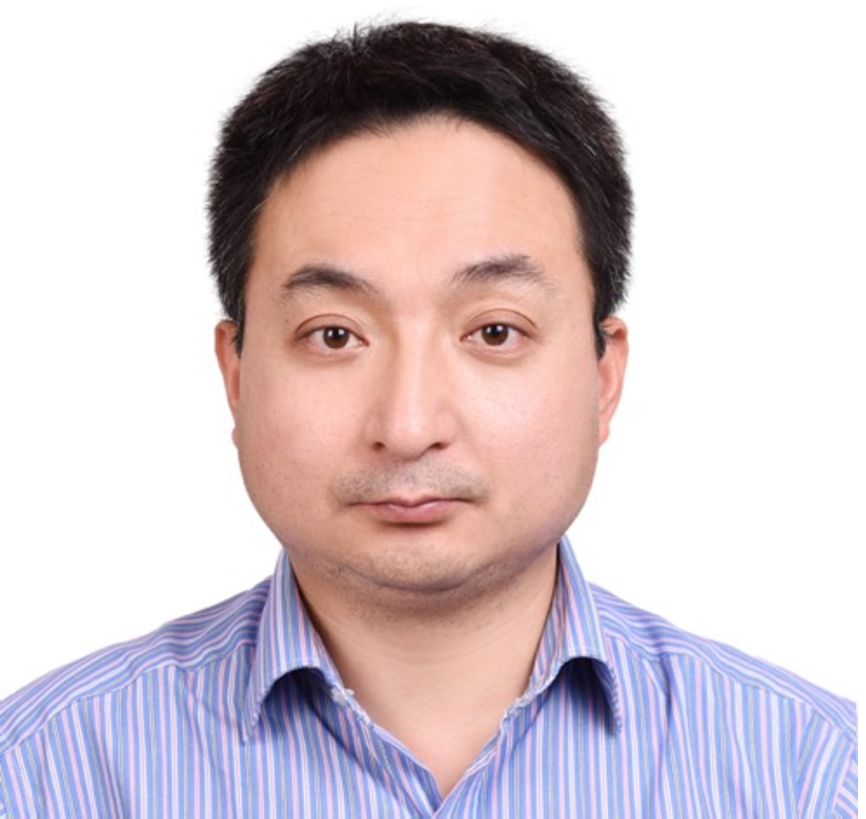}}]{Jianyong Sun} received the B.Sc. and M.Sc. degree in Mathematics from Xi'an Jiaotong University, Xi'an, China, in 1997 and 1999, respectively, and the Ph.D. degree in Computer Science from University of Essex, Colchester, U.K., in 2006. He is a full Professor in the School of Mathematics and Statistics, Xi'an Jiaotong University. His research interests include evolutionary computation and optimization, statistical machine learning, and their applications in bioinformatics, astro-informatics and image processing. \end{IEEEbiography}

\begin{IEEEbiography}[{\includegraphics[width=1in,height=1.25in,clip,keepaspectratio]{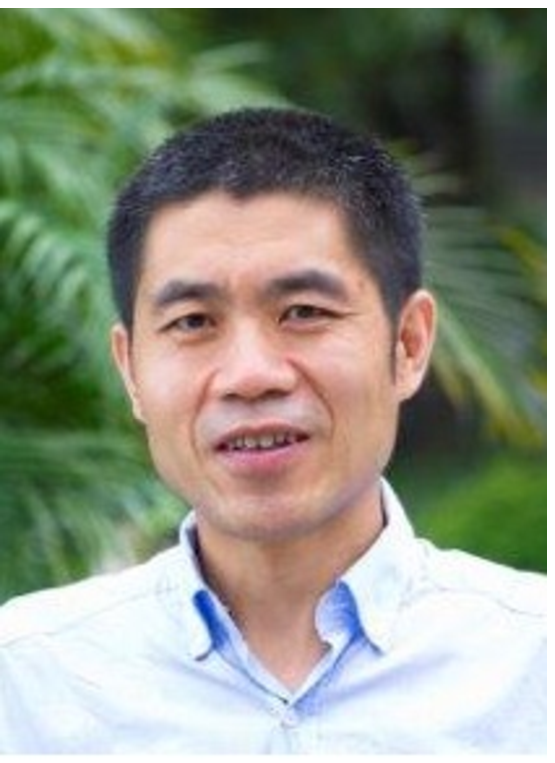}}]{Qingfu Zhang} received the BSc degree in mathematics from Shanxi University, China in 1984, the MSc degree in applied mathematics and the PhD degree in information engineering from Xidian University, China, in 1991 and 1994, respectively.

He is a Chair Professor of Computational Intelligence at the Department of Computer Science, City University of Hong Kong. His main research interests include evolutionary computation, optimization and machine learning.

Dr. Zhang is an Associate Editor of the IEEE Transactions on Evolutionary Computation and the IEEE Transactions on Cybernetics. He is a Web of Science highly cited researcher in Computer Science for four consecutive years from 2016.

\end{IEEEbiography}

\begin{IEEEbiography}[{\includegraphics[width=1in,height =1.25in,clip,keepaspectratio]{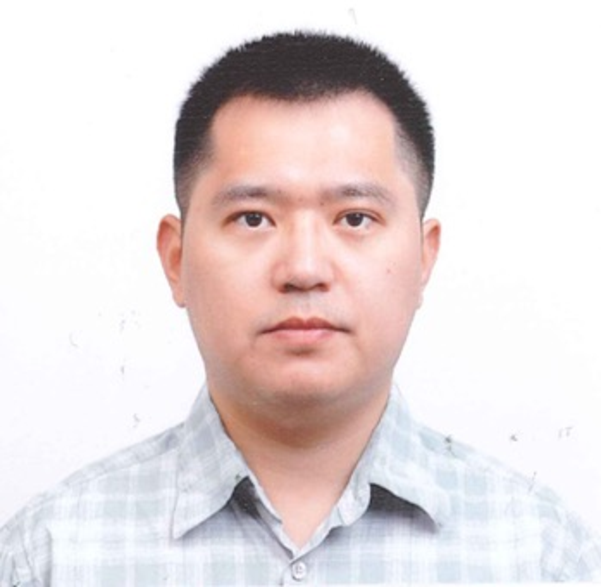}}]{Kai Ye} obtained B.S. and M.S. from Wuhan University and PhD from Leiden University in 2008. After one year postdoc training at European Bioinformatics Institute, he joined Leiden University Medical Center in 2009 as an assistant professor. In 2012, he moved to Washington University in St. Louis. He received the China 1000 Young Talent Program Award in 2016 and is a full Professor in the School of Electronics and Information Engineering at Xi'an Jiaotong University. He has published over 50 peer-reviewed papers. Dr. Ye's research interests include sequential pattern mining, computational methodology on biomedical big data and applications on precision medicine.
\end{IEEEbiography}








\end{document}